\documentclass{article}

\usepackage{microtype}
\usepackage{graphicx}
\usepackage{booktabs} 

\usepackage{hyperref}
\usepackage{times}  
\usepackage{helvet}  
\usepackage{courier}  
\usepackage{url}  
\usepackage{graphicx}  
\usepackage[utf8]{inputenc} 
\usepackage[T1]{fontenc}    
\usepackage{hyperref}       
\usepackage{bm}
\usepackage{amsfonts}       
\usepackage{nicefrac}       
\usepackage{amsmath}
\usepackage{subcaption}
\usepackage{bm}
\usepackage{float}

\usepackage{paralist}  
\usepackage{natbib}
\usepackage{array}
\usepackage{diagbox}
\usepackage{amsthm}
\usepackage[overload]{empheq}
\usepackage{algorithmic}

\usepackage{comment}
\usepackage{cases}

\newcommand{\real}{{\mathbb R}}

\newcommand{\set}[1]{\ensuremath{\mathcal #1}}

\newcommand{\separator}{
  \begin{center}
    \rule{\columnwidth}{0.3mm}
  \end{center}
}

 \def\11{{\textbf{1}}}

\newcommand{\beq}{\begin{eqnarray*}}
\newcommand{\eeq}{\end{eqnarray*}}
\newcommand{\beqn}{\begin{eqnarray}}
\newcommand{\eeqn}{\end{eqnarray}}
\newcommand{\bemn}{\begin{multiline}}
\newcommand{\eemn}{\end{multiline}}

\newtheorem{theorem}{Theorem}

\newtheorem{prop}{Proposition}

\newtheorem{defi}{Definition}

\newcommand{\Qtot}{Q_{\rm{jt}}}
\newcommand{\Vtot}{V_{\rm{jt}}}

\newcommand{\hQtot}{\Hat{Q}_{\rm{jt}}}
\newcommand{\algoname}{\text{QTRAN}}
\newcommand{\algonamebasic}{\text{QTRAN-base}}
\newcommand{\algonameadv}{\text{QTRAN-alt}}

\newcommand{\myparagraph}[1]{\vspace{-0.3cm} \paragraph{#1}}

\usepackage[accepted]{icml2019}

\icmltitlerunning{\algoname: Learning to Factorize with Transformation for Cooperative Multi-Agent Reinforcement Learning}

\allowdisplaybreaks
\begin{document}

\twocolumn[
\icmltitle{QTRAN: Learning to Factorize with Transformation for \\
           Cooperative Multi-Agent Reinforcement learning}



\icmlsetsymbol{equal}{*}

\begin{icmlauthorlist}
\icmlauthor{Kyunghwan Son}{KAIST}
\icmlauthor{Daewoo Kim}{KAIST}
\icmlauthor{Wan Ju Kang}{KAIST}
\icmlauthor{David Hostallero}{KAIST}
\icmlauthor{Yung Yi}{KAIST}
\end{icmlauthorlist}

\icmlaffiliation{KAIST}{School of Electrical Enginerring, KAIST, Daejeon, South Korea}

\icmlcorrespondingauthor{Yung Yi}{yiyung@kaist.edu}
     \icmlcorrespondingauthor{Kyunghwan Son}{kevinson9473@kaist.ac.kr}

\icmlkeywords{Multi-Agent Reinforcement Learning, Deep Reinforcement Learning, Multi-Agent Systems}

\vskip 0.3in
]



\printAffiliationsAndNotice{}  

\begin{abstract}
We explore value-based solutions for multi-agent reinforcement learning (MARL) tasks in the centralized training with decentralized execution (CTDE) regime popularized recently.
However, VDN and QMIX are representative examples that use the idea of factorization of the joint action-value function into individual ones for decentralized execution. VDN and QMIX address only a fraction of factorizable MARL tasks due to their structural constraint in factorization such as additivity and monotonicity. In this paper, we propose a new factorization method for MARL, $\algoname$, which is free from such structural constraints and 
takes on a new approach to transforming the original joint action-value function into an easily factorizable one, with the same optimal actions. $\algoname$ guarantees more general factorization than VDN or QMIX, thus covering a much wider class of MARL tasks than does previous methods. Our experiments  for the tasks of multi-domain Gaussian-squeeze and modified predator-prey demonstrate  $\algoname$'s superior performance with especially larger margins in games whose payoffs penalize non-cooperative behavior more aggressively. 


\end{abstract}
\section{Introduction}

Reinforcement learning aims to instill in agents a good policy that maximizes the cumulative reward in a given environment. Recent progress has witnessed success in various tasks, such as Atari games \citep{mnih2015human}, Go \citep{silver2016mastering, silver2017mastering}, and robot control \citep{lillicrap2015continuous}, just to name a few, with the development of deep learning techniques. Such advances largely consist of deep neural networks, which can represent action-value functions and policy functions in reinforcement learning problems as a high-capacity function approximator. 
However, more complex tasks such as robot swarm control and autonomous driving, often modeled as cooperative multi-agent learning problems, still remain unconquered due to their high scales and operational constraints such as distributed execution.

The use of deep learning techniques carries through to cooperative multi-agent reinforcement learning (MARL). MADDPG \cite{Lowe:MADDPG} learns distributed policy in continuous action spaces, and COMA \cite{DBLP:conf/aaai/FoersterFANW18} utilizes a counterfactual baseline to address the credit assignment problem.
Among value-based methods, value function factorization \cite{koller1999computing, guestrin2002multiagent, DBLP:conf/atal/SunehagLGCZJLSL18, pmlr-v80-rashid18a} methods have been proposed to efficiently handle a joint action-value function whose complexity grows exponentially with the number of agents. 

Two representative examples of value function factorization  include VDN \cite{DBLP:conf/atal/SunehagLGCZJLSL18} and QMIX \cite{pmlr-v80-rashid18a}. VDN factorizes the joint action-value function into a sum of individual action-value functions. QMIX extends this additive value factorization to represent the joint action-value function as a monotonic function --- rather than just as a sum --- of individual action-value functions, thereby covering a richer class of multi-agent reinforcement learning problems than does VDN. However, these value factorization techniques still suffer structural constraints, namely, additive decomposability in VDN and monotonicity in QMIX, often failing to factorize a factorizable task. A task is factorizable if the optimal actions of the joint action-value function are the same as the optimal ones of the individual action-value functions, where additive decomposability and monotonicity are only sufficient --- somewhat excessively restrictive --- for factorizability.
 

\myparagraph{Contribution}
In this paper, we aim at successfully factorizing {\em any} factorizable task, free from additivity/monotonicity concerns. We transform the original joint action-value function into a new, easily factorizable one with the same optimal actions in both functions. This is done by learning a state-value function, which corrects for the severity of the partial observability issue in the agents.




We incorporate the said idea in a novel architecture, called $\algoname$, consisting of the following inter-connected deep neural networks: (i) joint action-value network, (ii) individual action-value networks, and (iii) state-value network. To train this architecture, we define loss functions appropriate for each neural network. We develop two variants of $\algoname$: $\algonamebasic$ and $\algonameadv$, whose distinction is twofold: how to construct the transformed action-value functions for non-optimal actions; and the degree of stability and convergence speed. We assess the performance of $\algoname$ by comparing it against VDN and QMIX in three environments. First, we consider a simple, single-state matrix game that does not satisfy additivity or monotonicity, where $\algoname$ successfully finds the joint optimal action, whereas neither VDN nor QMIX does. We then observe a similarly desirable cooperation-inducing tendency of $\algoname$ in more complex environments: modified predator-prey games and multi-domain Gaussian squeeze tasks. In particular, we show that the performance gap between $\algoname$ and VDN/QMIX increases with environments having more pronounced non-monotonic characteristics.

\myparagraph{Related work}


Extent of centralization varies across the spectrum of cooperative MARL research. While more decentralized methods benefit from scalability, they often suffer non-stationarity problems arising from a trivialized superposition of individually learned behavior. Conversely, more centralized methods alleviate the non-stationarity issue at the cost of complexity that grows exponentially with the number of agents.

Prior work tending more towards the decentralized end of the spectrum include \citet{tan1993multi}, whose independent Q-learning algorithm exhibits the greatest degree of decentralization. \citet{tampuu2017multiagent} combines this algorithm with deep learning techniques presented in DQN \citep{mnih2015human}. These studies, while relatively simpler to implement, are subject to the threats of training instability, as multiple agents attempt to improve their policy in the midst of other agents, whose policies also change over time during training. This simultaneous alteration of policies essentially makes the environment non-stationary. 


The other end of the spectrum involves some centralized entity to resolve the non-stationarity problem. \citet{guestrin2002coordinated} and \citet{kok2006collaborative} are some of the earlier representative works. \citet{guestrin2002coordinated} proposes a graphical model approach in presenting an alternative characterization of a global reward function as a sum of conditionally independent agent-local terms. \citet{kok2006collaborative} exploits the sparsity of the states requiring coordination compared to the whole state space and then tabularize those states to carry out tabular Q-learning methods as in \citet{watkins1989learning}.

The line of research positioned mid-spectrum aims to put together the best of both worlds. More recent studies, such as COMA \citep{DBLP:conf/aaai/FoersterFANW18}, take advantage of CTDE \citep{oliehoek2008optimal}; actors are trained by a joint critic to estimate a counterfactual baseline designed to gauge each agent's contribution to the shared task. \citet{gupta2017cooperative} implements per-agent critics to opt for better scalability at the cost of diluted benefits of centralization. MADDPG \citep{Lowe:MADDPG} extends DDPG \citep{lillicrap2015continuous} to the multi-agent setting by similar means of having a joint critic train the actors. \citet{wei2018multiagent} proposes Multi-Agent Soft Q-learning in continuous action spaces to tackle the relative overgeneralization problem \cite{wei2016lenient} and achieves better coordination.
Other related work includes CommNet \cite{sukhbaatar2016learning}, DIAL \cite{foerster2016learning}, ATOC \cite{NIPS2018_7956}, and SCHEDNET \cite{kim2018learning},
which exploit inter-agent communication in execution.

On a different note, two representative examples of value-based methods have recently been shown to be somewhat effective in analyzing a class of games. Namely, VDN \citep{DBLP:conf/atal/SunehagLGCZJLSL18} and QMIX \citep{pmlr-v80-rashid18a} represent the body of literature most closely related to this paper. While both are value-based methods and follow the CTDE approach, the additivity and monotonicity assumptions naturally limit the class of games that VDN or QMIX can solve.

\section{Background}

\subsection{Model and CTDE}

\paragraph{DEC-POMDP} We take DEC-POMDP \cite{oliehoek2016concise} as the {\em de facto} standard for modelling cooperative multi-agent tasks, as do many previous works: as a tuple $\set{G} = <\set{S},\set{U},P,r,\set{Z},O,N,\gamma>$, where $\bm{s} \in \set{S}$ denotes the true state of the environment. Each agent $i \in \set{N} := \{1,...,N\} $ chooses an action $u_i \in \set{U}$ at each time step, giving rise to a joint action vector, $\bm{u} := [u_i]_{i=1}^N \in \set{U}^N$. Function $P(\bm{s}'|\bm{s},\bm{u}):\set{S}\times \set{U}^N \times \set{S} \mapsto [0,1]$ governs all state transition dynamics. Every agent shares the same joint reward function $r(\bm{s},\bm{u}):\set{S} \times \set{U}^N \mapsto \real $, and $\gamma \in [0,1)$ is the discount factor.
Each agent has individual, partial observation $z \in \set{Z}$, according to some observation function $O(\bm{s},i) : \set{S} \times \set{N} \mapsto \set{Z}.$ Each agent also has an action-observation history $\tau_i \in \set{T} := (\set{Z} \times \set{U})^{*}$, on which it conditions its stochastic policy $\pi_i(u_i|\tau_i) : \set{T} \times \set{U} \mapsto [0,1]$. 

\myparagraph{Training and execution: CTDE}
Arguably the most na\"ive training method for MARL tasks is to learn the individual agents' action-value functions independently, {\em i.e.}, independent Q-learning. This method would be simple and scalable, but it cannot guarantee convergence even in the limit of infinite greedy exploration. As an alternative solution, recent works including VDN \citep{DBLP:conf/atal/SunehagLGCZJLSL18} and QMIX \citep{pmlr-v80-rashid18a} employ centralized training with decentralized execution (CTDE) \citep{oliehoek2008optimal} to train multiple agents. CTDE allows agents to learn and construct individual action-value functions, such that optimization at the individual level leads to optimization of the joint action-value function. This in turn, enables agents at execution time to select an optimal action simply by looking up the individual action-value functions, without having to refer to the joint one. Even with only partial observability and restricted inter-agent communication, information can be made accessible to all agents at training time.

\subsection{{\bf IGM} Condition and Factorizable Task}

Consider a class of sequential decision-making tasks that are amenable to factorization in the centralized training phase.
We first define {\bf IGM} (Individual-Global-Max):
\begin{defi}[IGM]
\label{def:igm}
For a joint action-value function $\Qtot: \set{T}^N \times \set{U}^N \mapsto \real$, where $\bm{\tau} \in \set{T}^N$ is a joint action-observation histories, if there exist 
individual action-value functions $[Q_i: \set{T} \times \set{U} \mapsto \real]_{i=1}^N,$ such that the following holds 
\begin{align}
\label{eq:max_cond}
\arg\max_{\bm{u}} \Qtot(\bm{\tau},\bm{u}) & = 
\begin{pmatrix}
\arg\max_{u_1} Q_1(\tau_1, u_1) \\
\vdots \\
\arg\max_{u_N} Q_N(\tau_n, u_N)
\end{pmatrix},
\end{align}
then, we say that $[Q_i]$ satisfy {\bf IGM} for $\Qtot$ under $\bm{\tau}$. 
In this case, we also say that $\Qtot(\bm{\tau},\bm{u})$ is factorized by $[Q_i(\tau_i,u_i)]$, or that $[Q_i]$ are factors of $\Qtot$.
\end{defi}

Simply put, the optimal joint actions across agents are equivalent to the collection of individual optimal actions of each agent. 
If $\Qtot(\bm{\tau},\bm{u})$ in a given task is factorizable under all $\bm{\tau} \in \set{T}^N,$ we say that the task itself is factorizable.



\subsection{VDN and QMIX}

Given $\Qtot$, one can consider the following two sufficient conditions for {\bf IGM}:  
\begin{align}
\text{\bf (Additivity)} \quad & \Qtot(\bm{\tau},\bm{u}) = \sum_{i=1}^N Q_i(\tau_i, u_i), \label{eq:addcondition} \\ 
\text{\bf (Monotonicity)} \quad &{\partial\Qtot(\bm{\tau},\bm{u}) \over\partial Q_i(\tau_i,u_i)} \geq 0, \quad \forall i \in \set{N}. \label{eq:monocondition}
\end{align}

VDN \cite{DBLP:conf/atal/SunehagLGCZJLSL18} and QMIX \cite{pmlr-v80-rashid18a} are methods that attempt to factorize $\Qtot$ assuming additivity and monotonicity, respectively. 
Thus, joint action-value functions satisfying those conditions would be well-factorized by VDN and QMIX. However, there exist tasks whose joint action-value functions do not meet the said conditions.
We illustrate this limitation of VDN and QMIX using a simple matrix game in the next section.

\section{$\algoname$: Learning to Factorize with Transformation}

In this section, we propose a new method called $\algoname$, aiming at factorizing any factorizable task. The key idea is to transform the original joint action-value function
$\Qtot$ into a new one $\Qtot'$ that shares the optimal joint action with 
$\Qtot.$



\begin{figure*}[!t]
  \centering
 \includegraphics[width=0.93\textwidth]{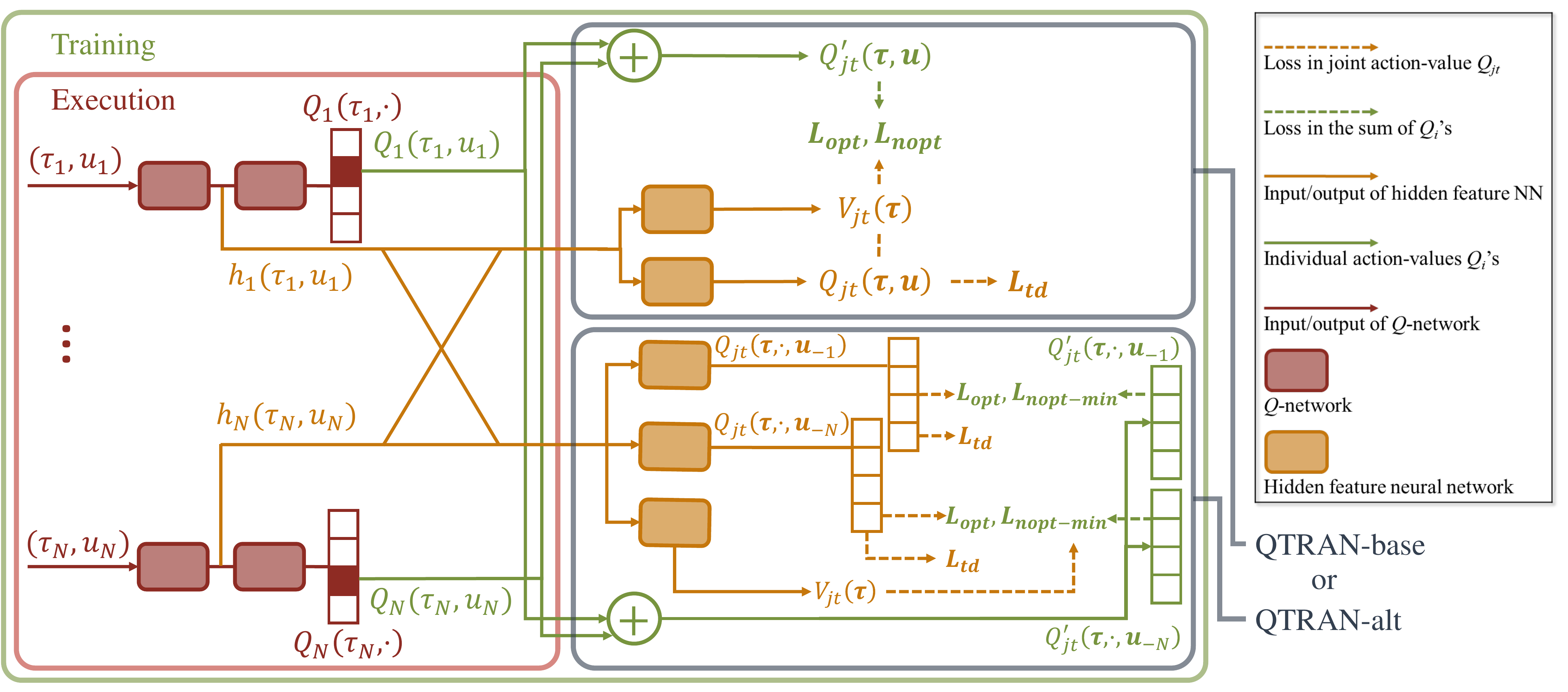}
\caption{$\algonamebasic$ and $\algonameadv$ Architecture}
  \label{fig:qreg}
\end{figure*}

\subsection{Conditions for the Factor Functions $[Q_i]$}
\label{sec:transformation}
For a given joint observation $\bm{\tau}$, consider an arbitrary factorizable $\Qtot(\bm{\tau},\bm{u}).$ 
Then, by Definition~\ref{def:igm} of {\bf IGM}, we can find individual action-value functions $[Q_i(\tau_i,u_i)]$ that factorize $\Qtot(\bm{\tau},\bm{u}).$ Theorem~\ref{thm:ifonlyif} states the sufficient condition for $[Q_i]$ that satisfy {\bf IGM}.
Let $\bar{u}_{i}$ denote the optimal local action $\arg\max_{u_{i}}Q_i(\tau_i,u_i)$ and $\bm{\bar{u}} = [\bar{u}_{i}]_{i=1}^N,$. Also, let $\bm{Q} = [Q_i] \in \real^N,$ {\em i.e.}, a column vector of $Q_i, i = 1, \ldots, N.$ 

\begin{theorem}
\label{thm:ifonlyif}
A factorizable joint action-value function $\Qtot(\bm{\tau},\bm{u})$ is factorized by $[Q_i(\tau_i,u_i)],$ if 
\begin{subnumcases}{\label{eq:regular_cond0}\hspace{-0.7cm}\sum_{i=1}^N Q_i(\tau_i,u_i) \! -\! \Qtot(\bm{\tau},\bm{u}) \! +\!\Vtot(\bm{\tau})\!=\!}
    \!\!\! 0 & $\hspace{-0.5cm}\bm{u} = \bm{\bar{u}},$ \label{eq:regular_cond1}\\
    \!\!\! \geq 0 & $\hspace{-0.5cm}\bm{u} \neq \bm{\bar{u},}$ \label{eq:regular_cond2}
\end{subnumcases}
where 
\begin{align*}
\Vtot(\bm{\tau}) &= \max_{\bm{u}}\Qtot(\bm{\tau},\bm{u}) - \sum_{i=1}^N Q_i(\tau_i,\bar{u}_i).
\end{align*}

\end{theorem}
The proof is provided in the Supplementary. 
We note that conditions in \eqref{eq:regular_cond0} are also {\em necessary} under an affine transformation. That is, there exists an affine transformation $\phi(\bm{Q}) = A\cdot \bm{Q} +B,$ where $A=[a_{ii}] \in \real_+^{N\times N}$ is a symmetric diagonal matrix with $a_{ii} > 0, \forall i$ and $B=[b_i] \in \real^N,$ such that if $\Qtot$ is factorized by $[Q_i]$, then \eqref{eq:regular_cond0} holds by replacing $Q_i$ with $a_{ii}Q_i +b_i.$ This is because for all $i$, $b_i$ cancels out, and $a_{ii}$ just plays the role of re-scaling the value of $\sum_{i=1}^N Q_i$ in multiplicative (with a positive scaling constant) and additive manners, since {\bf IGM} is invariant to $\phi$ of $[Q_i].$


\myparagraph{Factorization via transformation}
We first define a new function $\Qtot'$ as the linear sum of individual factor functions $[Q_i]$:
 \begin{align}
 \Qtot'(\bm{\tau},\bm{u}) & := \sum_{i=1}^N Q_i(\tau_i, u_i). \label{eq:define_qtprime}
 \end{align}
We call $\Qtot'(\bm{\tau},\bm{u})$ the {\em transformed joint-action value function} throughout this paper. Our idea of factorization is as follows: from the additive construction of $\Qtot'$ based on $[Q_i],$ 
$[Q_i]$ satisfy {\bf IGM} for the new joint action-value function $\Qtot'$, implying that $[Q_i]$ are also the factorized individual action-value functions of $\Qtot'$.
From the fact that $\arg\max_{\bm{u}} \Qtot(\bm{\tau},\bm{u})=\arg\max_{\bm{u}} \Qtot'(\bm{\tau},\bm{u}),$ finding $[Q_i]$ satisfying \eqref{eq:regular_cond0} is precisely the factorization of $\Qtot'(\bm{\tau},\bm{u})$. One interpretation of this process is that rather than directly factorizing $\Qtot$, we consider an alternative joint action-value function ({\em i.e.}, $\Qtot'$) that is factorized by additive decomposition. 
The function $\Vtot(\bm{\tau})$ corrects for the discrepancy between the centralized joint action-value function $\Qtot$ and the sum of individual joint action-value functions $[Q_i]$. This discrepancy arises from the partial observability of agents. If bestowed with full observability, $\Vtot$ can be re-defined as zero, and the definitions would still stand. Refer to the Supplementary for more detail.

\subsection{Method}\label{sec:arch}

\paragraph{Overall framework}

In this section, we propose a new deep RL framework with value
function factorization, called $\algoname$, whose architectural sketch is given in Figure~\ref{fig:qreg}.
$\algoname$ consists of three separate estimators: (i) each agent $i$'s individual action-value network for $Q_i$, (ii) a joint action-value network for $\Qtot$ to be factorized into individual action-value functions $Q_i$, and (iii) a state-value network for $\Vtot$, {\em i.e.},
\begin{align*}
  \text{\bf (Individual action-value network)} & \quad f_{\text{q}}: (\tau_i,u_i) \mapsto Q_i, \cr \quad
  \text{\bf (Joint action-value network)} &  \quad f_{\text{r}}: (\bm{\tau}, \bm{u}) \mapsto \Qtot, \cr
  \text{\hspace{0cm}} \text{\bf (State-value network)} & \quad f_{\text{v}}: \bm{\tau} \mapsto
    \Vtot. 
\end{align*}
Three neural networks are trained in a centralized manner, and each agent uses its own factorized individual action-value function $Q_i$ to take action during
decentralized execution. Each network is elaborated next.

\myparagraph{Individual action-value networks} 
For each agent, an action-value network takes its own action-observation history
$\tau_i$ as input, and produces action-values $Q_i(\tau_i,\cdot)$ as output. This
action-value network is used for each agent to determine its own
action by calculating the action-value for a given $\tau_i$. As defined in
\eqref{eq:define_qtprime}, $\Qtot'$ is just the summation of the outputs of
all agents.

\myparagraph{Joint action-value network} The joint action-value network approximates $\Qtot$.
It takes as input the selected action and produces the Q-value of the chosen action as output. For scalability and sample efficiency, we design
this network as follows. First, we use the action vector sampled by all individual action-value networks to update the joint action-value network. Since the joint action space is $\set{U}^N$, finding an optimal action requires high complexity as the number of agents $N$ grows, whereas obtaining an optimal action in each individual network is done by decentralized policies with linear-time individual $\arg\max$ operations. Second, the joint action-value network shares the parameters at the lower layers of individual networks, where the joint action-value network combines hidden features with summation $\sum_i h_{Q,i}(\tau_i, u_i)$ of $h_i(\tau_i, u_i) = [h_{Q,i}(\tau_i, u_i), h_{V,i}(\tau_i)]$ from  individual networks. This parameter sharing is used to enable scalable training with good sample efficiency at the expense of expressive power.
\myparagraph{State-value network} The state-value network is responsible for computing a scalar state-value, similar to $V(s)$ in the dueling network \citep{wang2016dueling}. $\Vtot$ is required to provide the flexibility to match $\Qtot$ and $\Qtot'+\Vtot$ at $\arg\max$. Without state-value network, partial observability would limit the representational complexity of $\Qtot'$. The state-value is independent of the selected action for a given $\bm{\tau}$. Thus, this value network does not contribute to choosing an action, and is instead used to calculate the loss of
\eqref{eq:regular_cond0}. Like the joint action-value network, we use
the combined hidden features $\sum_i h_{V,i}(\tau_i)$ from the individual
 networks as input to the value network for scalability. 


\subsection{Loss Functions}
\label{sec:loss}



There are two major goals in centralized training.
One is that we should train the joint action-value function $\Qtot$ to estimate the true action-value; the other is that the transformed action-value function $\Qtot'$
should ``track'' the joint action-value function in the sense that their optimal actions are equivalent. 
We use the algorithm introduced in DQN \citep{mnih2015human} to update networks, where we maintain a target network and a replay buffer. 
To this end, we devise the following global loss function in $\algoname$, combining three loss functions in a weighted manner: 
\begin{align}
  \label{eq:loss_total}
L(\bm{\tau}, \bm{u}, r, \bm{\tau}'; \bm{\theta})=
 L_{\text{td}} + \lambda_{\text{opt}} L_{\text{opt}} + \lambda_\text{nopt} L_\text{nopt},
\end{align}
where $r$ is the reward for action $\bm{u}$ at observation histories $\bm{\tau}$ with transition to $\bm{\tau}'$. 
$L_{\text{td}}$ is the loss function for estimating the true action-value, by minimizing the TD-error as $\Qtot$ is learned. $L_{\text{opt}}$ and $L_{\text{nopt}}$ are losses for factorizing $\Qtot$ by $[Q_i]$ satisfying condition \eqref{eq:regular_cond0}. The role of $L_{\text{nopt}}$ is to check at each step if the action selected in the samples satisfied \eqref{eq:regular_cond2}, and $L_{\text{opt}}$ confirms that the optimal local action obtained satisfies \eqref{eq:regular_cond1}. One could implement \eqref{eq:regular_cond0} by defining a loss depending on how well the networks satisfy \eqref{eq:regular_cond1} or \eqref{eq:regular_cond2} with actions taken in the samples. However, in this way, verifying whether \eqref{eq:regular_cond1} is indeed satisfied would take too many samples since optimal actions are seldom taken at training. Since we aim to learn $\Qtot'$ and $\Vtot$ to factorize for a given $\Qtot$, we stabilize the learning by fixing $\Qtot$ when learning with $L_{\text{opt}}$ and $L_{\text{nopt}}$. We let $\hQtot$ denote this fixed $\Qtot$. $\lambda_{\text{opt}}$ and $\lambda_{\text{nopt}}$ are the weight constants for two losses. 
The detailed forms of $L_{\text{td}}$, $L_{\text{opt}},$ and $L_{\text{nopt}}$ are given as follows, where we omit their common function arguments $(\bm{\tau}, \bm{u}, r, \bm{\tau}')$ in loss functions for presentational simplicity: 
\begin{align*}
L_{\text{td}}(;\bm{\theta})&= \big(\Qtot(\bm{\tau},\bm{u}) - y^{\text{dqn}}(r, \bm{\tau}'; \bm{\theta}^{-})\big)^{2},\cr 
L_{\text{opt}}(;\bm{\theta})&= \big(\Qtot'(\bm{\tau},\bm{\bar{u}}) - \hQtot(\bm{\tau},\bm{\bar{u}}) + \Vtot(\bm{\tau})\big)^{2}, \cr 
L_{\text{nopt}}(;\bm{\theta})&= \left(\min \big[\Qtot'(\bm{\tau},\bm{u}) - \hQtot(\bm{\tau},\bm{u}) + \Vtot(\bm{\tau}),0\big] \right)^{2}, 
\end{align*} 
where 
$y^{\text{dqn}}(r, \bm{\tau}'; \bm{\theta^{-}}) = r + \gamma
\Qtot(\bm{\tau}', \bm{\bar{u}}';\bm{\theta} ^{-}),$ $ \bar{\bm{u}}' = [\arg\max_{u_{i}}Q_i(\tau'_i,u_i;\bm{\theta}^{-})]_{i=1}^N,$ and $\bm{\theta}^{-}$ are the periodically copied parameters from $\bm{\theta}$, as in DQN \citep{mnih2015human}.

\subsection{Tracking the Joint Action-value Function Differently}


We name the method previously discussed {\bf $\algonamebasic$}, to reflect the basic nature of how it keeps track of the joint action-value function. Here on, we consider a variant of $\algoname$, which utilizes a counterfactual measure. As mentioned earlier, Theorem~\ref{thm:ifonlyif} is used to enforce {\bf IGM} by \eqref{eq:regular_cond1} and determine how the individual action-value functions $[Q_i]$ and the State-value function $\Vtot$ jointly ``track'' $\Qtot$ by \eqref{eq:regular_cond2}, which governs the stability of constructing the correct factorizing $Q_i$'s. We found that condition \eqref{eq:regular_cond2} is often too loose, leading the neural networks to fail their mission of constructing the correct factors of $\Qtot.$ That is, condition \eqref{eq:regular_cond2} imposes undesirable influence on the non-optimal actions, which in turn compromises the stability and/or convergence speed of the training process. This motivates us to study conditions stronger than \eqref{eq:regular_cond2} that would still be sufficient for factorizability, but at the same time would also be necessary under the aforementioned affine transformation $\phi$, as in Theorem~\ref{thm:ifonlyif}.

\begin{theorem}
\label{thm:min}
The statement presented in {\rm Theorem~\ref{thm:ifonlyif}} and the necessary condition of {\rm Theorem~\ref{thm:ifonlyif}} holds by replacing \eqref{eq:regular_cond2} with the following \eqref{eq:min2}: if $\bm{u} \neq \bar{\bm{u}}$, 
\begin{multline}
\label{eq:min2}
\min_{u_i \in \set{U}} \Big [\Qtot'(\bm{\tau},u_i,\bm{u}_{-i}) 
- \Qtot(\bm{\tau},u_i,\bm{u}_{-i}) \cr + \Vtot(\bm{\tau}) \Big] = 0, \quad \forall i =1, \ldots, N,
\end{multline}
where $\bm{u}_{-i} = (u_1, \ldots, u_{i-1}, u_{i+1}, \ldots, u_{N}),$ {\rm i.e.}, the action vector except for $i$'s action.  
\end{theorem}

The proof is presented in the Supplementary. The key idea behind \eqref{eq:min2} lies in what conditions to enforce on {\em non-optimal} actions. It stipulates that $\Qtot'(\bm{\tau},\bm{u}) - \Qtot(\bm{\tau},\bm{u}) + \Vtot(\bm{\tau})$ be set to zero for some actions. Now, it is not possible to zero this value for every action $u$, but it is available for at least one action whilst still abiding by Theorem~\ref{thm:ifonlyif}. 
It is clear that condition~\eqref{eq:min2} is stronger than condition \eqref{eq:regular_cond2}, as desired.
For non-optimal actions $\bm{u} \neq \bm{\bar{u}}$, the conditions of Theorem~\ref{thm:ifonlyif} are satisfied when $\Qtot(\bm{\tau},\bm{u}) - \Vtot(\bm{\tau}) \le \Qtot’(\bm{\tau},\bm{u}) \le \Qtot’(\bm{\tau},\bm{\bar{u}})$ for any given $\bm{\tau}$. Under this condition, however, there can exist a non-optimal action $\bm{u} \neq \bm{\bar{u}}$ whose $\Qtot’(\bm{\tau},\bm{u})$ is comparable to $\Qtot’(\bm{\tau},\bm{\bar{u}})$ but  $\Qtot(\bm{\tau},\bm{u})$ is much smaller than  $\Qtot(\bm{\tau},\bm{\bar{u}})$. 
This may cause instability in the practical learning process.
However, the newly devised condition~\eqref{eq:min2} compels $\Qtot’(\bm{\tau},\bm{u})$ to track $\Qtot(\bm{\tau},\bm{u})$ even for the problematic non-optimal actions mentioned above. This helps in widening the gap between $\Qtot'(\bm{\tau},\bm{u})$ and $\Qtot'(\bm{\tau},\bm{\bar{u}})$, and this gap makes the algorithm more stable. 
Henceforth, we call the deep MARL method outlined by Theorem~\ref{thm:min} {\bf $\algonameadv$}, to distinguish it from the one due to condition~\eqref{eq:regular_cond2}. 

\myparagraph{Counterfactual joint networks}
To reflect our idea of \eqref{eq:min2}, we now propose a counterfactual joint network, which replaces the joint action-value network of $\algonamebasic$, to efficiently calculate
$\Qtot(\bm{\tau},\cdot,\bm{u}_{-i})$ and
$\Qtot'(\bm{\tau},\cdot,\bm{u}_{-i})$ for all $i \in \set{N}$ with only one
forward pass. To this end, in the $\algonameadv$ module, 
each agent has a counterfactual joint network with the output of
$\Qtot(\bm{\tau},\cdot,\bm{u}_{-i})$ for each possible action,
given other agents' actions. As a joint action-value network, we use
$h_{V,i}(\tau_i)$ and the combined hidden features $\sum_{j\ne i} h_{Q,j}(\tau_j,u_j)$ from other agents.
Finally, $\Qtot'(\bm{\tau},\cdot,\bm{u}_{-i})$ is calculated as
$Q_i(\tau_i,\cdot) + \sum_{j\ne i} Q_j(\tau_j,u_j)$ for all agents. This architectural choice is realized by choosing the loss function to be $L_\text{nopt-min}$, replacing $L_{\text{nopt}}$ in $\algonamebasic$ as follows:
\begin{align*}
L_{\text{nopt-min}}(\bm{\tau}, \bm{u}, r, \bm{\tau}'; \bm{\theta})&= {1 \over N}\sum_{i=1}^N (\min_{u_i \in \set{U}}D(\bm{\tau},u_i,\bm{u}_{-i}))^2,
\end{align*} 
where 
$$\!\! D(\bm{\tau},u_i,\bm{u}_{-i})\! = \! 
\Qtot'(\bm{\tau},u_i,\bm{u}_{-i}) -
\hQtot(\bm{\tau},u_i,\bm{u}_{-i})\! +\! 
\Vtot(\bm{\tau}).$$ 
In $\algonameadv$, $L_{\text{td}}, L_{\text{opt}}$ are also used, but they are also computed for all agents.

\subsection{Example: Single-state Matrix Game}\label{sec:matgame}
In this subsection, we present how $\algoname$ performs compared to existing value factorization ideas such as VDN and QMIX, and 
how the two variants $\algonamebasic$ and $\algonameadv$ behave. 
The matrix game and learning results are shown in Table~\ref{table:matrix}\footnote{We present only $\Qtot$ and $\Qtot'$, because in fully observable cases ({\em i.e.}, observation function is bijective for all $i$) Theorem~\ref{thm:ifonlyif} holds for $\Vtot(\tau)=0$. We discuss further in Supplementary.}.
This symmetric matrix game has the optimal joint action $(A,A)$, and captures a very simple cooperative multi-agent task, where we have two users with three actions each. Evaluation with more complicated tasks are provided in the next subsection.
We show the results of VDN, QMIX, and $\algoname$ through a full exploration ({\em i.e.}, $\epsilon=1$ in $\epsilon$-greedy) conducted over 20,000 steps. Full exploration guarantees to explore all available game states. Therefore, we can compare only the expressive power of the methods. Other details are included in the Supplementary. 

\newcolumntype{P}[1]{>{\centering\arraybackslash}p{#1}}
\begin{table}[t]
\scriptsize
    \setlength{\tabcolsep}{-1pt}
    \begin{minipage}{.48\columnwidth}
      \centering
        \begin{tabular}{|P{1cm}||P{1cm}|P{1cm}|P{1cm}|}
        \hline
        \backslashbox{$u_1$}{$u_2$}& A & B & C\\ \hline \hline
        A & \textbf{8}& -12 & -12 \\ \hline
        B & -12 & 0 & 0 \\ \hline
        C & -12 & 0 & 0 \\ \hline
        \end{tabular}
        \subcaption{Payoff of matrix game }
        \label{table:p7orig}
    \end{minipage}
    \begin{minipage}{.48\columnwidth}
      \centering
        \begin{tabular}{|P{1cm}||P{1cm}|P{1cm}|P{1cm}|}
        \hline
        \backslashbox{$Q_1$}{$Q_2$} & \textbf{4.16}(A) & 2.29(B) & 2.29(C)\\ \hline \hline
        \textbf{3.84}(A) & \textbf{8.00}  & 6.13 & 6.12 \\ \hline
        -2.06(B) & 2.10 & 0.23 & 0.23 \\ \hline
        -2.25(C) & 1.92 & 0.04 & 0.04 \\ \hline
        \end{tabular}
        \subcaption{$\algoname$:  $Q_1, Q_2,\Qtot'$}
        \label{table:matrix-qreg}
    \end{minipage}\\
    \begin{minipage}{.48\columnwidth}
      \centering
        \begin{tabular}{|P{1cm}||P{1cm}|P{1cm}|P{1cm}|}
        \hline
        \backslashbox{$u_1$}{$u_2$}& A & B & C\\ \hline \hline
        A & \textbf{8.00}  & -12.02 & -12.02 \\ \hline
        B & -12.00 & 0.00 & 0.00 \\ \hline
        C & -12.00 & 0.00 & -0.01 \\ \hline
        \end{tabular}
        \subcaption{$\algoname$: $\Qtot$ }
        \label{table:matrix-qreg'}
    \end{minipage}
    \begin{minipage}{.48\columnwidth}
      \centering
        \begin{tabular}{|P{1cm}||P{1cm}|P{1cm}|P{1cm}|}
        \hline
        \backslashbox{$u_1$}{$u_2$}& A & B & C\\ \hline \hline
        A & \textbf{0.00}  & 18.14 & 18.14 \\ \hline
        B & 14.11 & 0.23 & 0.23 \\ \hline
        C & 13.93 & 0.05 & 0.05 \\ \hline
        \end{tabular}
        \subcaption{$\algoname$: $\Qtot' - \Qtot$}
        \label{table:matrix-F}
    \end{minipage}\\
        \begin{minipage}{.48\columnwidth}
      \centering
        \begin{tabular}{|P{1cm}||P{1cm}|P{1cm}|P{1cm}|}
            \hline
            \backslashbox{$Q_1$}{$Q_2$} & -3.14(A) & \textbf{-2.29}(B) & -2.41(C)\\ \hline \hline
            -2.29(A) & -5.42  & -4.57 & -4.70 \\ \hline
            -1.22(B) & -4.35 & -3.51 & -3.63 \\ \hline
            \textbf{-0.73}(C) & -3.87 & \textbf{-3.02} & -3.14 \\ \hline
        \end{tabular}
        \subcaption{VDN: $Q_1, Q_2,\Qtot$}
        \label{table:matrix-vdn}
    \end{minipage}
    \begin{minipage}{.48\columnwidth}
      \centering
         \begin{tabular}{|P{1cm}||P{1cm}|P{1cm}|P{1cm}|}
        \hline
        \backslashbox{$Q_1$}{$Q_2$} & -0.92(A) & 0.00(B) & \textbf{0.01}(C)\\ \hline \hline
        -1.02(A) & -8.08  & -8.08 & -8.08 \\ \hline
        \textbf{0.11}(B) & -8.08 & 0.01 & \textbf{0.03} \\ \hline
        0.10(C) & -8.08 & 0.01 & 0.02 \\ \hline
        \end{tabular}
        \subcaption{QMIX: $Q_1, Q_2,\Qtot$}
        \label{table:matrix-qmix}
    \end{minipage}
        \caption{Payoff matrix of the one-step game and reconstructed $\Qtot$ results on the game. Boldface means optimal/greedy actions from the state-action value}
\vspace{-0.2cm}
\label{table:matrix}
\end{table}

\myparagraph{Comparison with VDN and QMIX}
Tables \ref{table:matrix-qreg}-\ref{table:matrix-qmix} show the learning results of 
$\algoname$, VDN, and QMIX. Table~\ref{table:matrix-qreg} shows that 
$\algoname$ enables each agent to jointly take the optimal action only by using its own locally optimal action, meaning successful factorization. Note that Tables~\ref{table:matrix-qreg'} and \ref{table:matrix-F} demonstrate the difference between $\Qtot$ and $\Qtot'$, stemming from our transformation, where their optimal actions are the nonetheless same. 
Table~\ref{table:matrix-F} shows that $\algoname$ also satisfies \eqref{eq:regular_cond0}, thereby validating our design principle as described in Theorem~\ref{thm:ifonlyif}.
However, in VDN, agents 1's and 2's individual optimal actions are $C$ and $B$, respectively. VDN fails to factorize because the structural constraint of additivity $\Qtot(\bm{u}) = \sum_{i=1,2} Q_i(u_i)$ is enforced, leading to deviations from {\bf IGM}, whose sufficient condition is additivity, {\em i.e.}, $\Qtot(\bm{u}) = \sum_{i=1,2} Q_i(u_i)$ for all $\bm{u}=(u_1,u_2).$ QMIX also fails in factorization in a similar manner due to the structural constraint of monotonicity.


\begin{figure}[!t]
 
\hspace*{\fill}
\begin{subfigure}[t]{.25\columnwidth}
  \includegraphics[width=\linewidth]{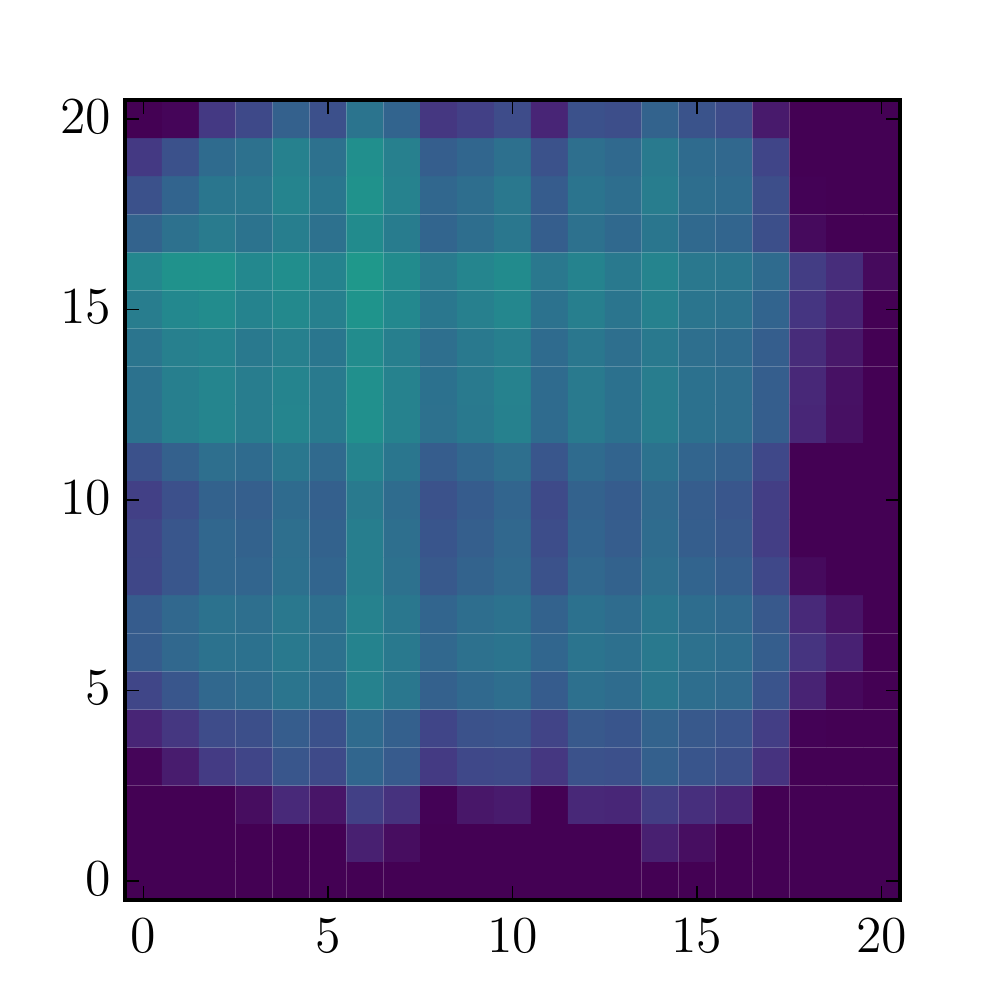}
  \caption{\scriptsize $\Qtot$: 1000 step}
  \label{fig:PQgraph1}
\end{subfigure}\hspace*{\fill}
\begin{subfigure}[t]{.25\columnwidth}
  \includegraphics[width=\linewidth]{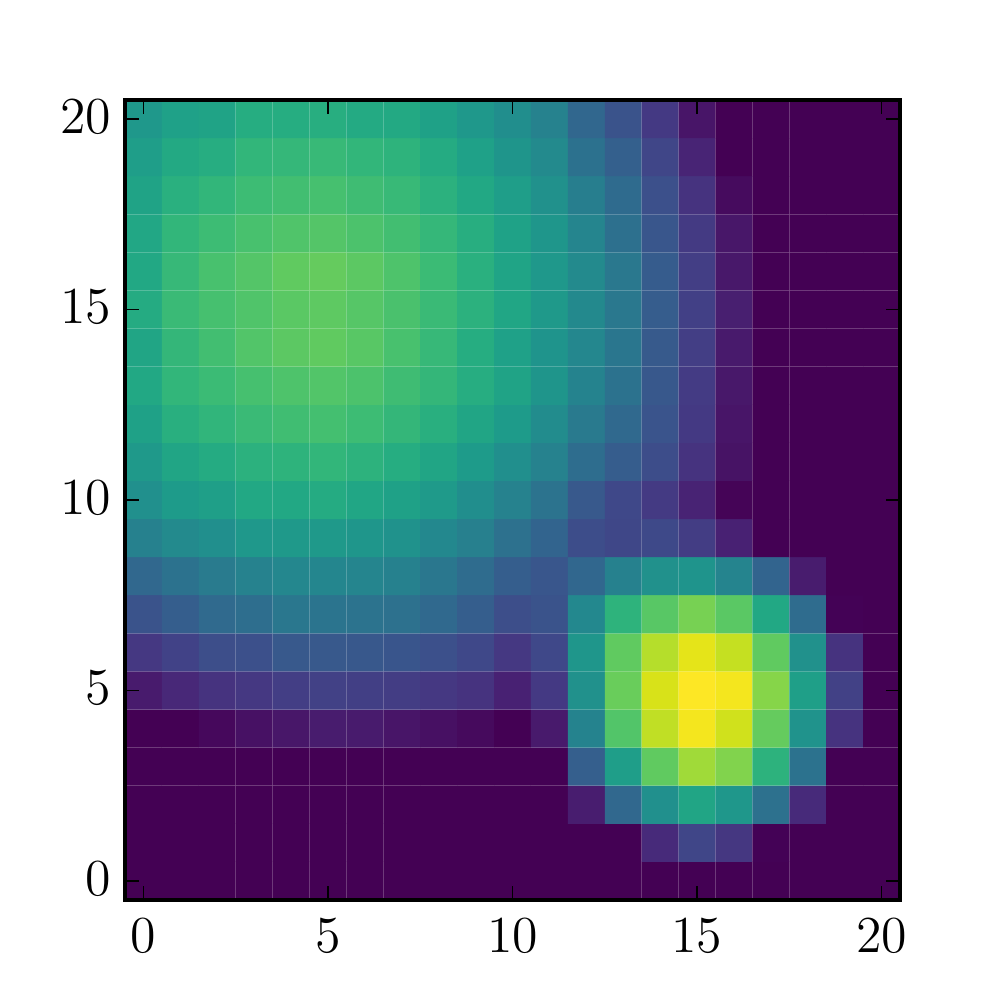}
  \caption{\scriptsize $\Qtot$: 2000 step}
  \label{fig:PQgraph2}
\end{subfigure}\hspace*{\fill}
\begin{subfigure}[t]{.25\columnwidth}
  \includegraphics[width=\linewidth]{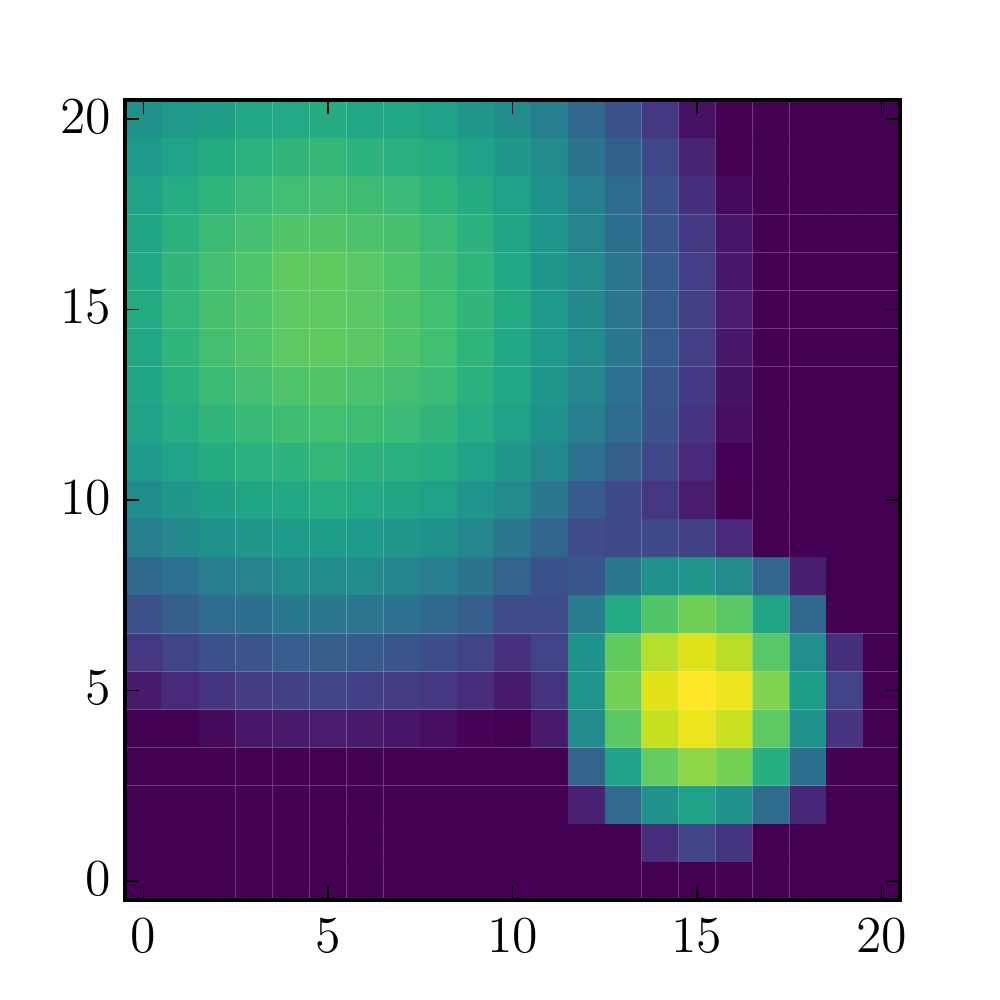}
  \caption{\scriptsize $\Qtot$: 3000 step}
  \label{fig:PQgraph3}
\end{subfigure}\hspace*{\fill}
\begin{subfigure}[t]{.25\columnwidth}
  \includegraphics[width=\linewidth]{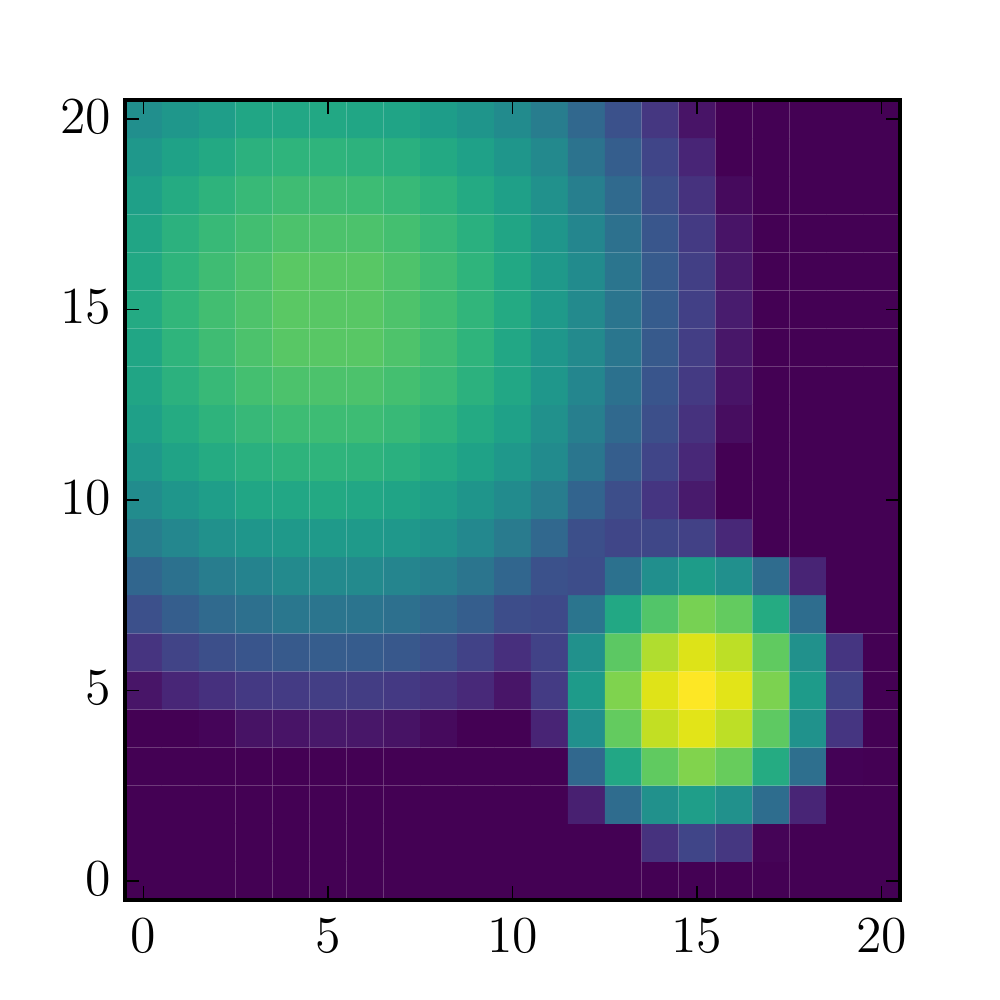}
  \caption{\scriptsize $\Qtot$: 4000 step}
  \label{fig:PQgraph4}
\end{subfigure}\hspace*{\fill}\\
\hspace*{\fill}
\begin{subfigure}[t]{.25\columnwidth}
  \includegraphics[width=\linewidth]{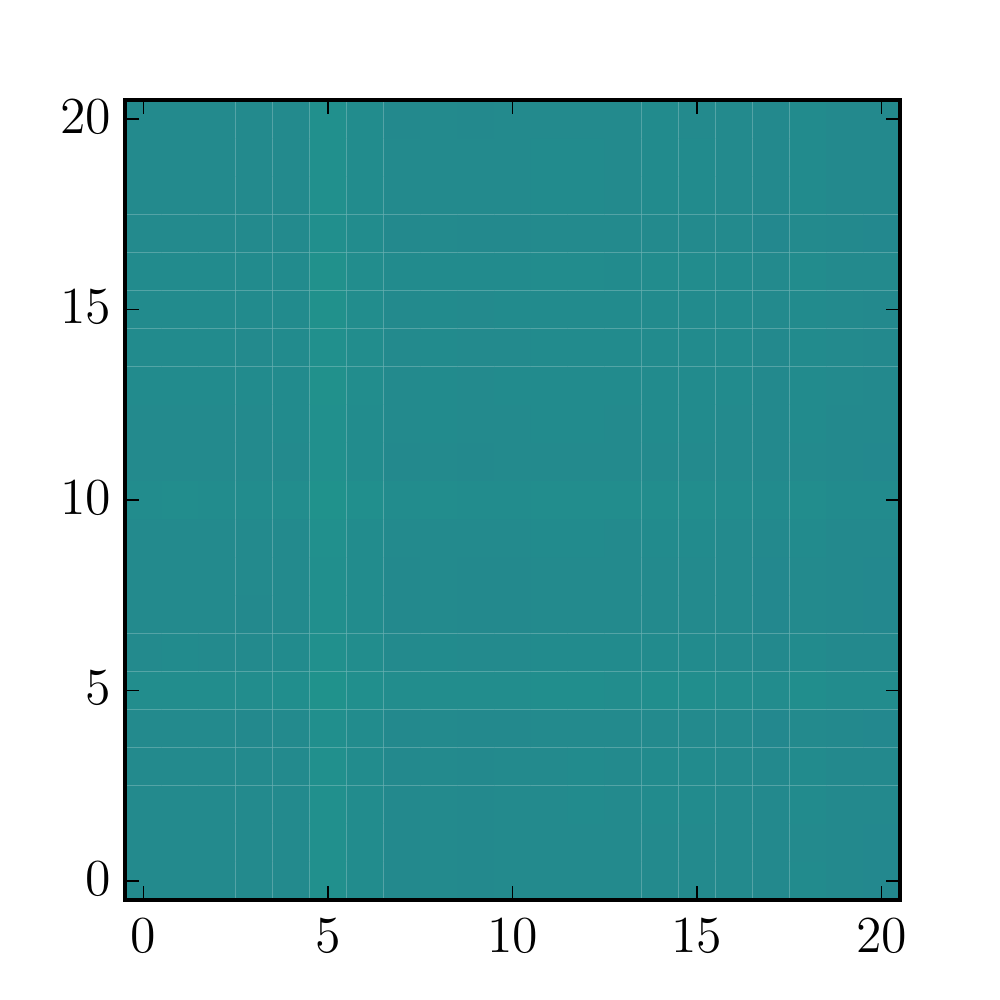}
  \caption{\scriptsize {\bf base}: 1000 step}
  \label{fig:PCgraph1}
\end{subfigure}\hspace*{\fill}
\begin{subfigure}[t]{.25\columnwidth}
  \includegraphics[width=\linewidth]{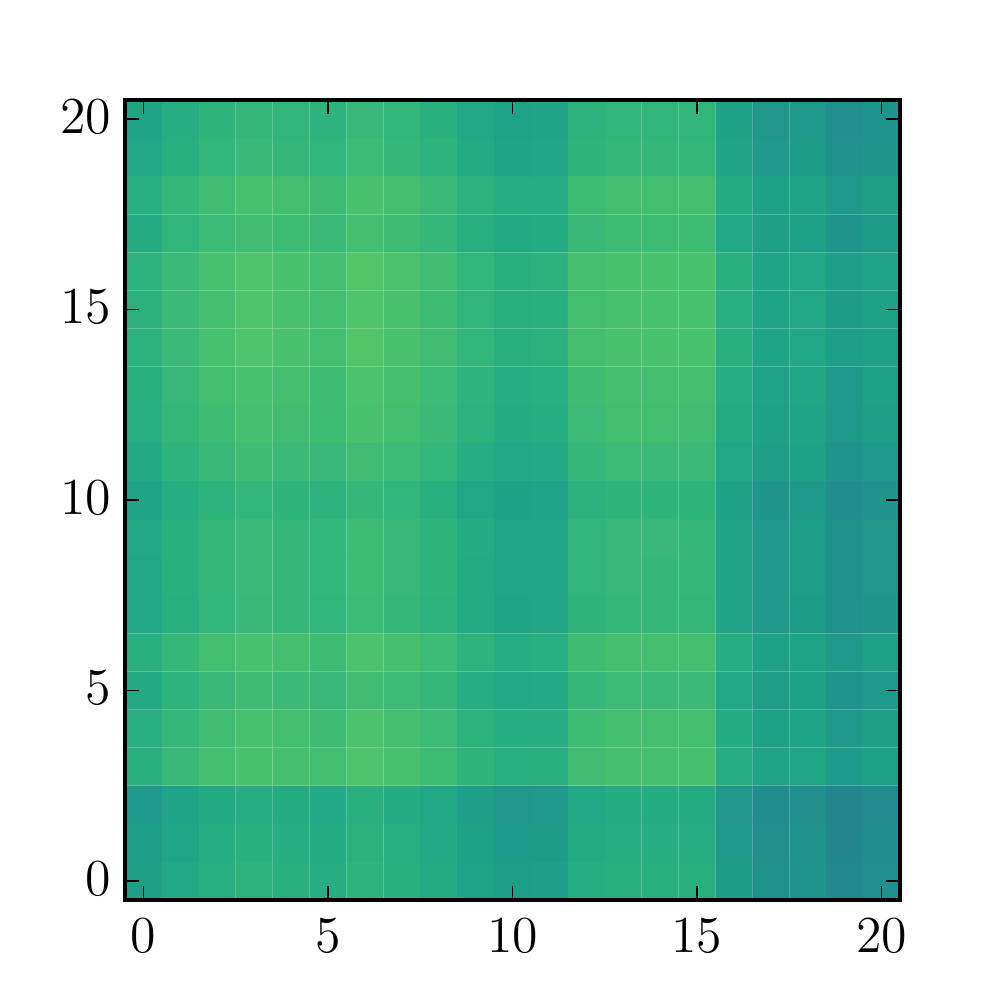}
  \caption{\scriptsize {\bf base}: 2000 step}
  \label{fig:PCgraph2}
\end{subfigure}\hspace*{\fill}
\begin{subfigure}[t]{.25\columnwidth}
  \includegraphics[width=\linewidth]{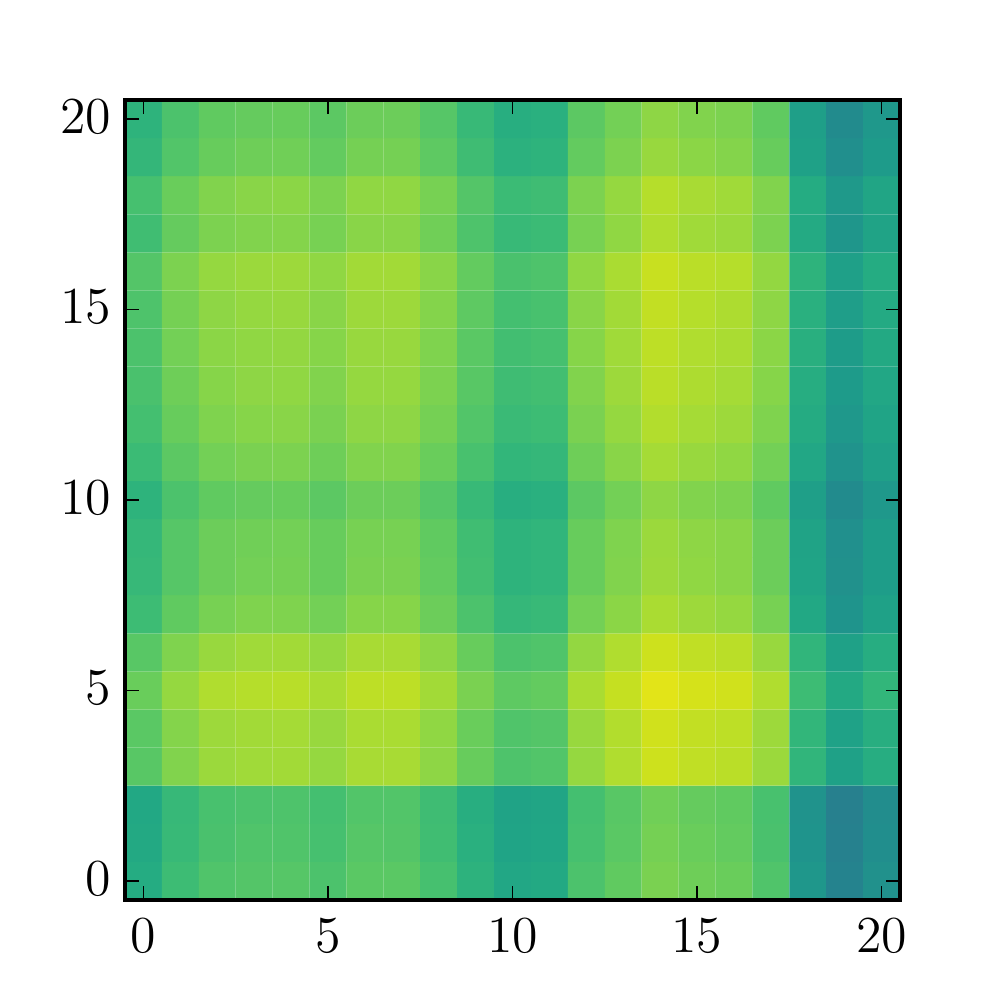}
  \caption{\scriptsize {\bf base}: 3000 step}
  \label{fig:PCgraph3}
\end{subfigure}\hspace*{\fill}
\begin{subfigure}[t]{.25\columnwidth}
  \includegraphics[width=\linewidth]{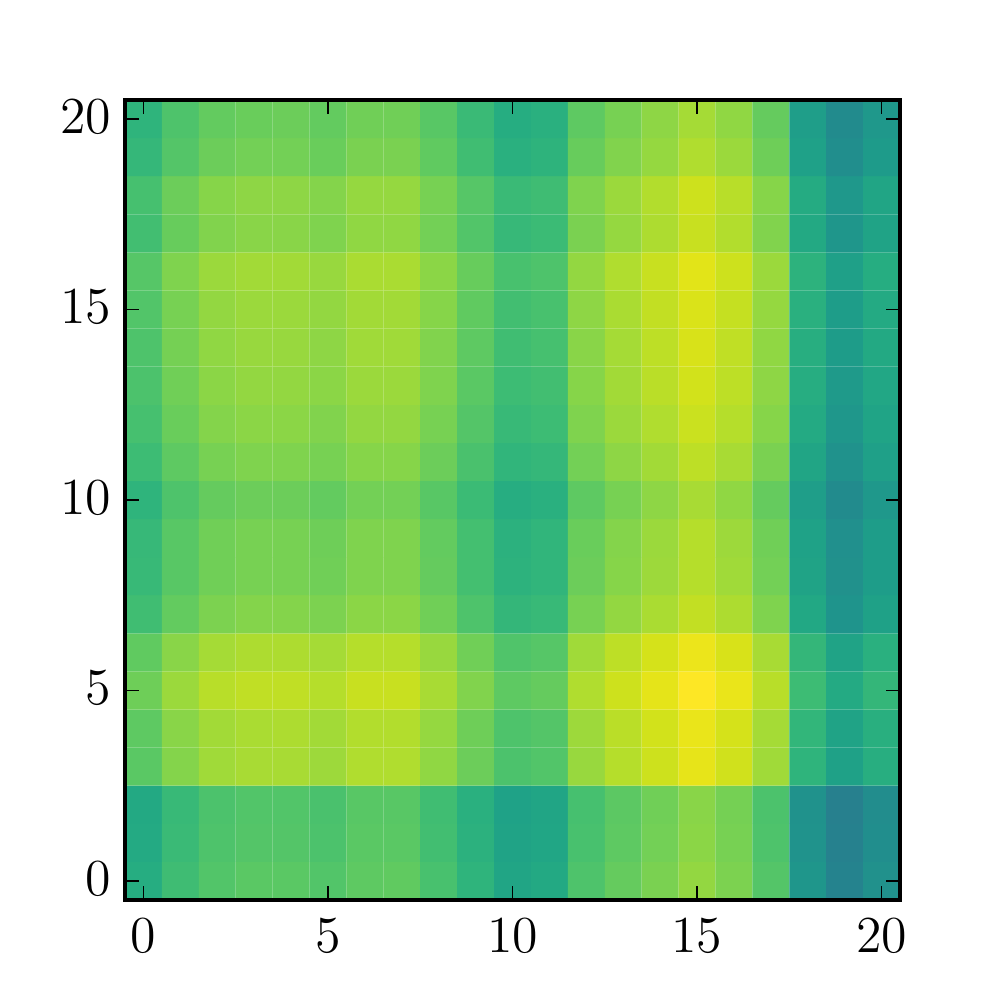}
  \caption{\scriptsize {\bf base}: 4000 step}
  \label{fig:PCgraph4}
\end{subfigure}\hspace*{\fill}\\
\begin{subfigure}[t]{.25\columnwidth}
  \includegraphics[width=\linewidth]{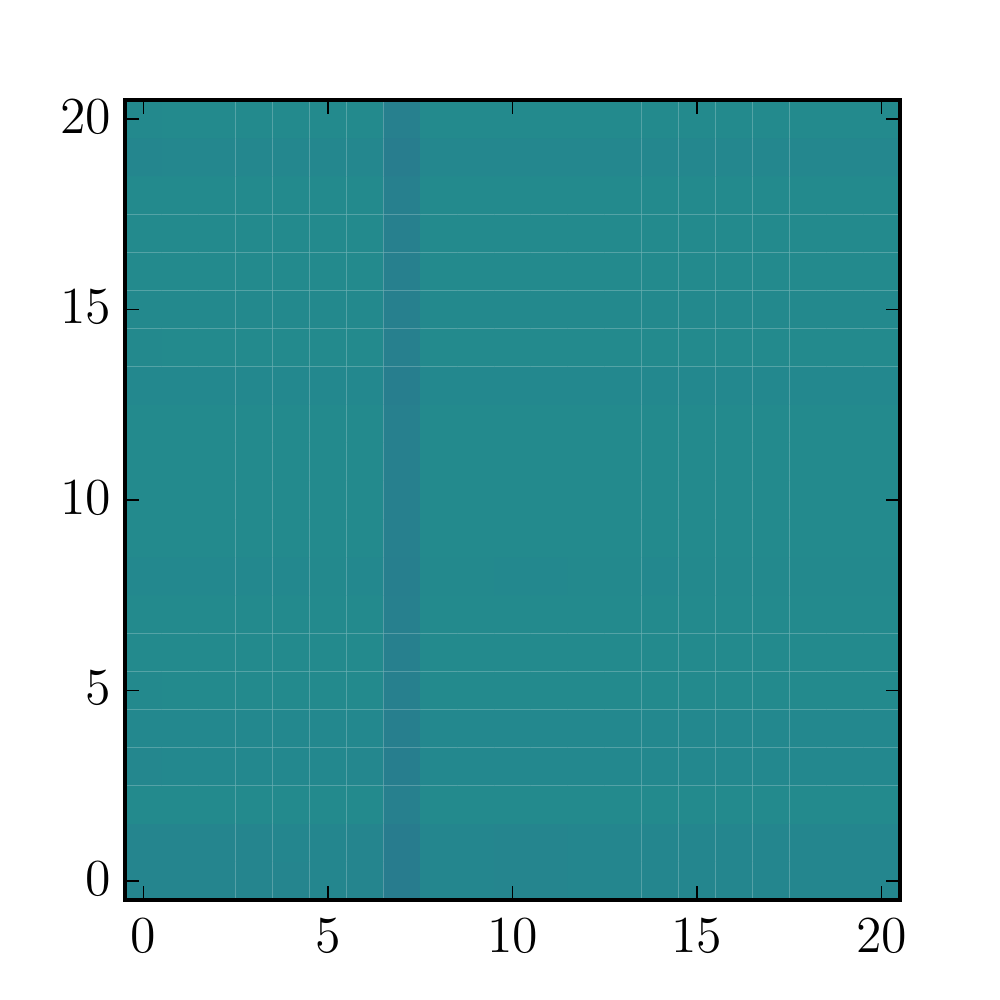}
  \caption{\scriptsize {\bf alt}: 1000 step}
  \label{fig:PDgraph1}
\end{subfigure}\hspace*{\fill}
\begin{subfigure}[t]{.25\columnwidth}
  \includegraphics[width=\linewidth]{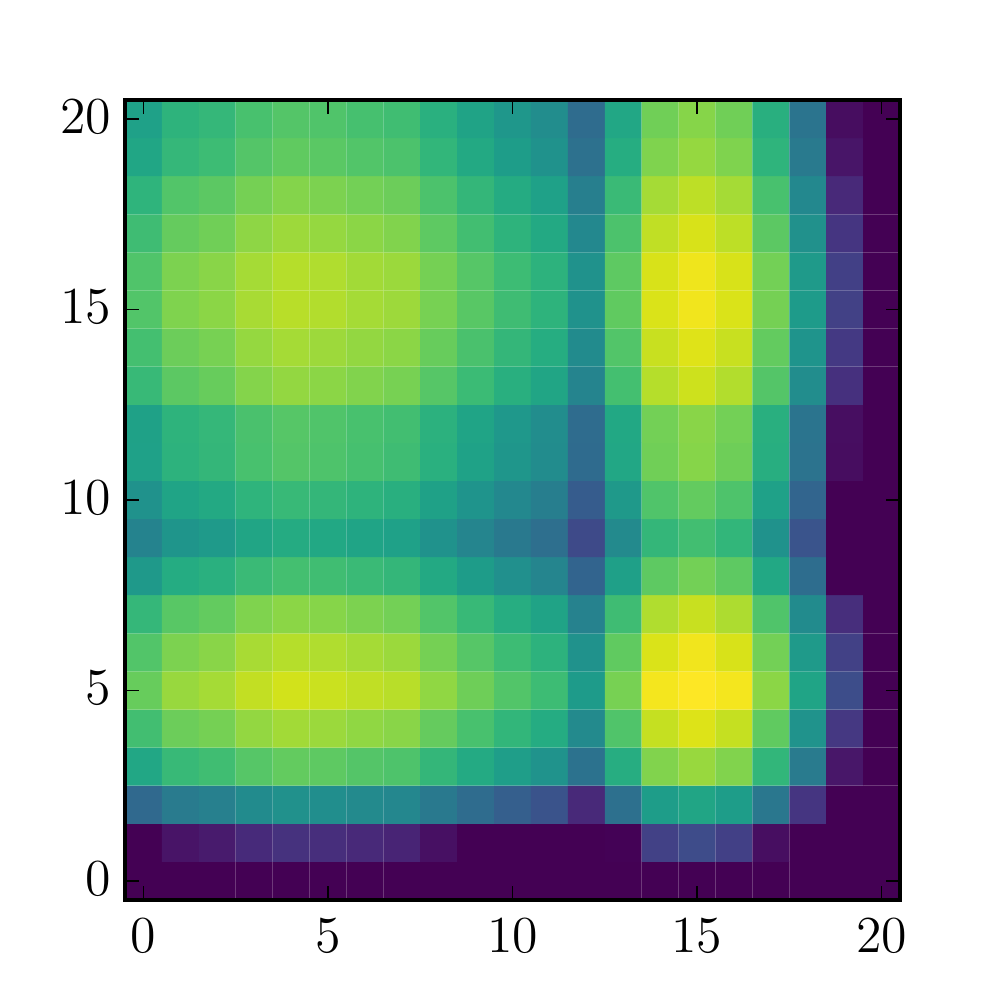}
  \caption{\scriptsize {\bf alt}: 2000 step}
  \label{fig:PDAgraph2}
\end{subfigure}\hspace*{\fill}
\begin{subfigure}[t]{.25\columnwidth}
  \includegraphics[width=\linewidth]{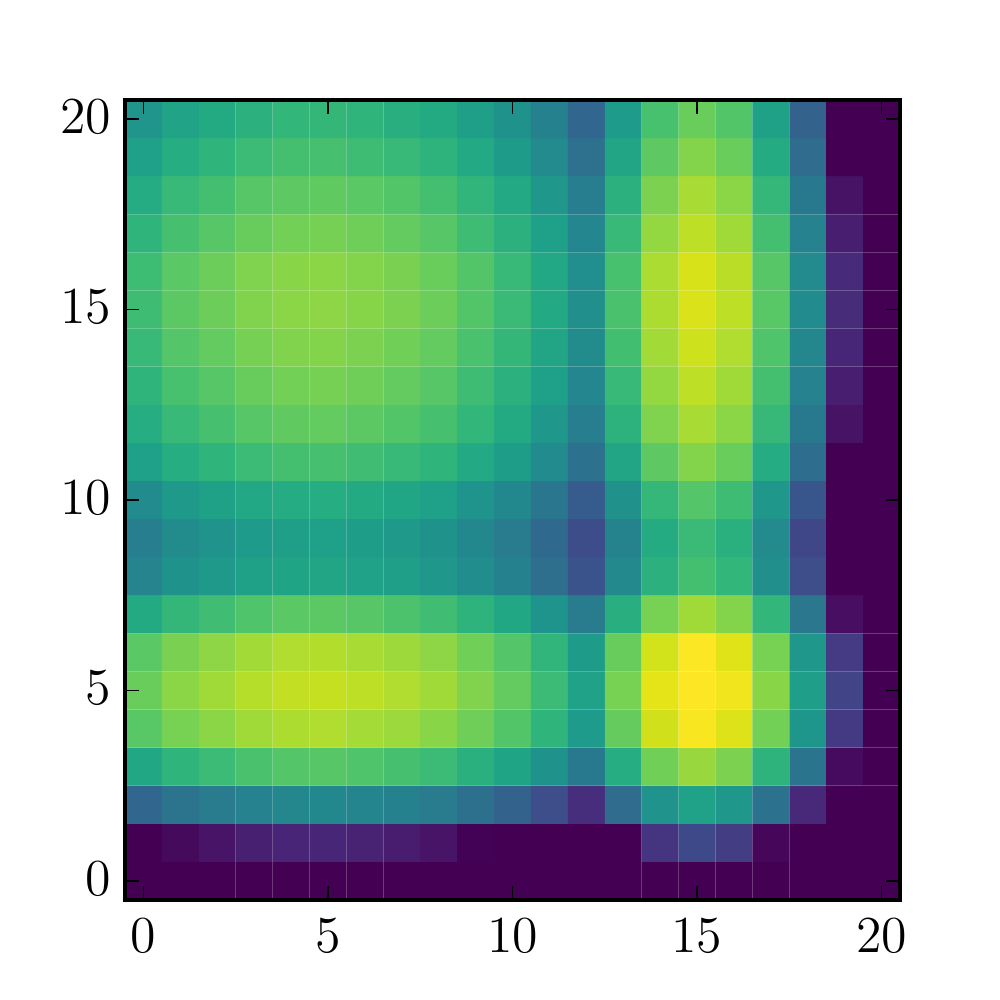}
  \caption{\scriptsize {\bf alt}: 3000 step}
  \label{fig:PDgraph3}
\end{subfigure}\hspace*{\fill}
\begin{subfigure}[t]{.25\columnwidth}
  \includegraphics[width=\linewidth]{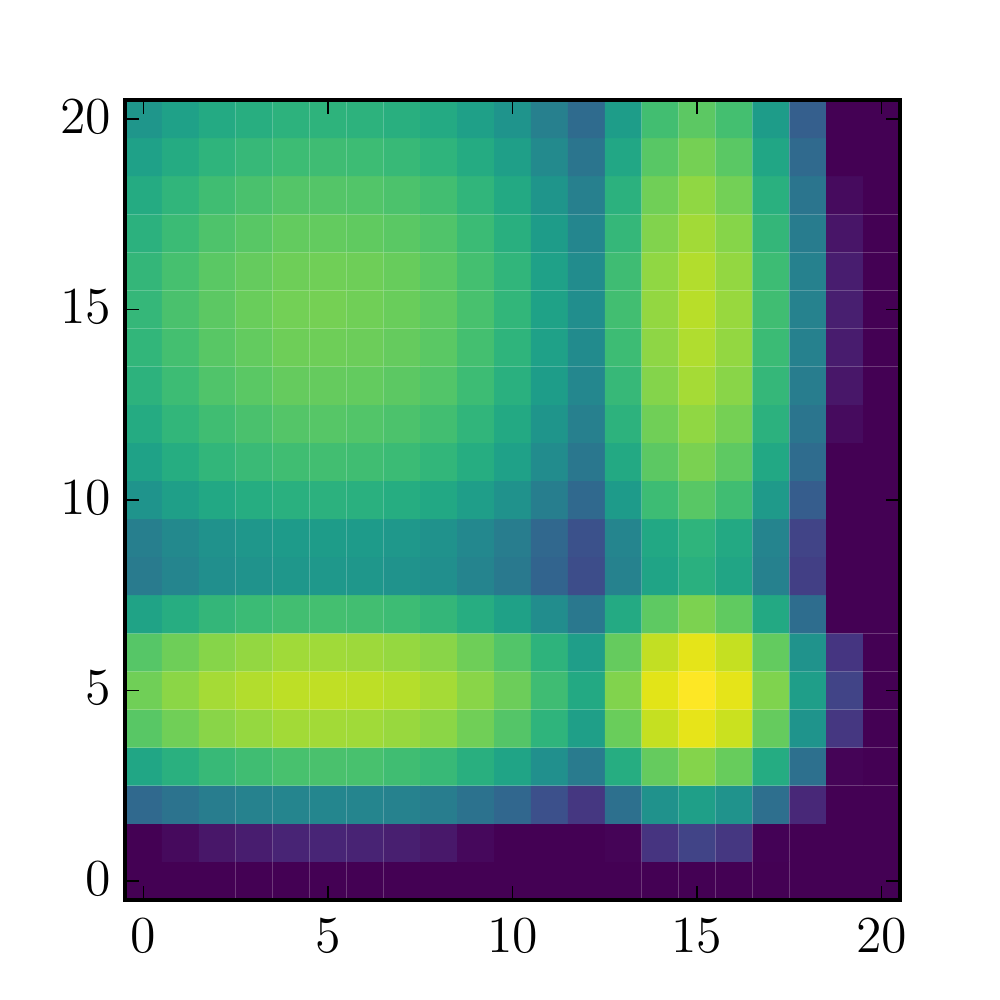}
  \caption{\scriptsize {\bf alt}: 4000 step}
  \label{fig:PDgraph4}
\end{subfigure}\hspace*{\fill}
\caption{{\bf base}: $\algonamebasic$, {\bf alt}: $\algonameadv$. $x$-axis and $y$-axis: agents 1 and 2's actions, respectively. Colored values represent the values of $\Qtot$ ((a)-(d)) and $\Qtot'$ ((e)-(l)) for selected actions.}
\label{fig:PCresult1}
\vspace{-0.4cm}
\end{figure}

\myparagraph{Impact of $\algonameadv$}
In order to see the impact of $\algonameadv$, we train the agents in a matrix game where two agents each have 21 actions. Figure~\ref{fig:PCresult1} illustrates the joint action-value function and its transformations of both $\algonamebasic$ and $\algonameadv$.
The result shows that both algorithms successfully learn the 
optimal action by correctly estimating the $\Qtot$ for $\bm{u}=\bm{\bar{u}}$ for any given state. 
$\Qtot'$ values for non-optimal actions are different from $\Qtot$, but it has a different tendency in each algorithm as follows. 
As shown in Figures~\ref{fig:PCgraph1}-\ref{fig:PCgraph4}, all $\Qtot'$ values in $\algonamebasic$ have only a small difference from the maximum value of $\Qtot$, whereas Figures~\ref{fig:PDgraph1}-\ref{fig:PDgraph4} show that $\algonameadv$ has the ability to more accurately distinguish optimal actions from non-optimal actions. Thus, in $\algonameadv$, the agent can smartly explore and have better sample efficiency to train the networks. This feature of $\algonameadv$ also prevents learning unsatisfactory policies in complex environments.
Full details on the experiment are included in the Supplementary.

\begin{figure*}[!t]
\hspace*{\fill}
\begin{subfigure}[t]{.33\textwidth}
  \includegraphics[width=\linewidth]{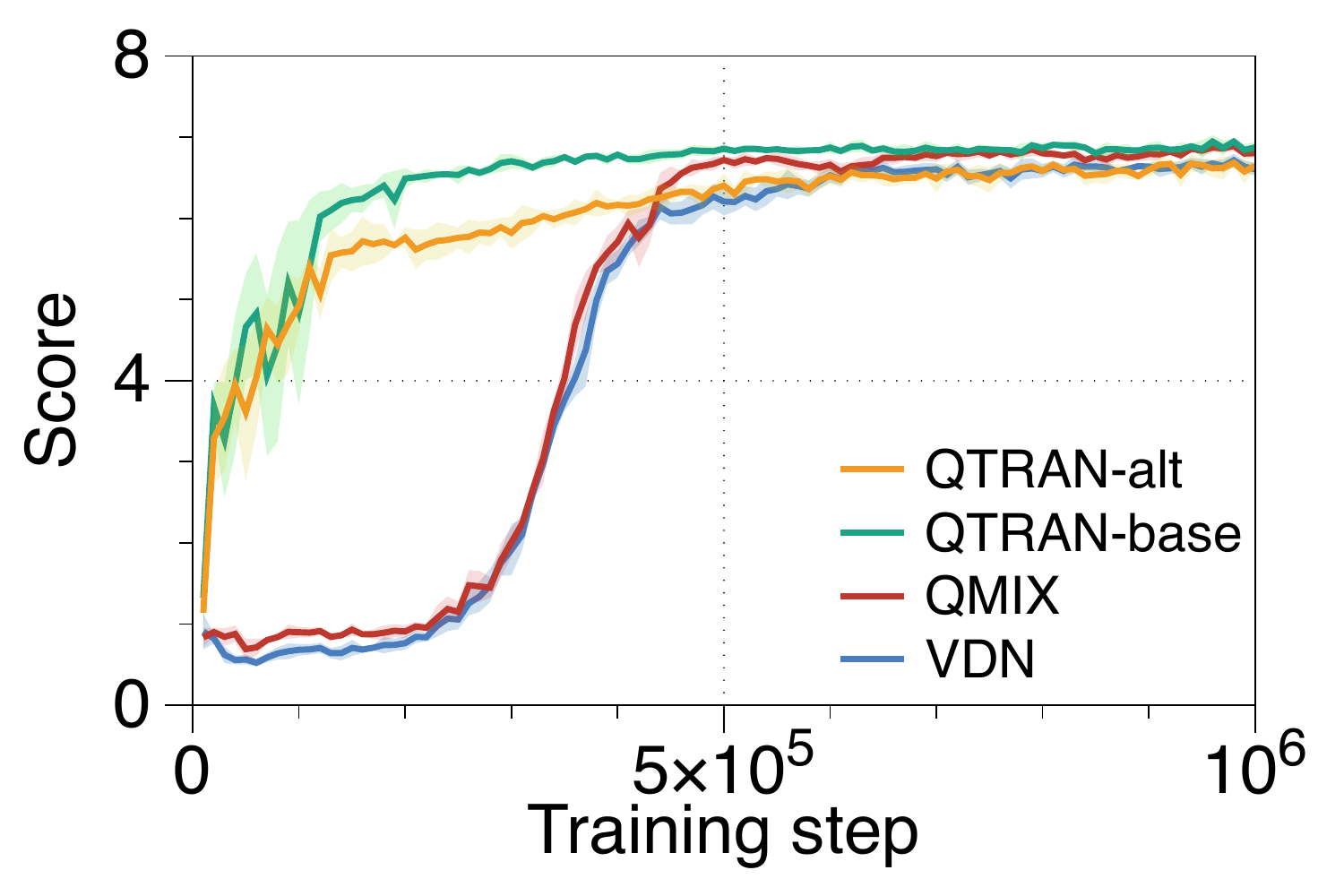}
  \caption{Gaussian Squeeze}
  \label{fig:GSgraph1}
\end{subfigure}\hspace*{\fill}
\begin{subfigure}[t]{.33\textwidth}
  \includegraphics[width=\linewidth]{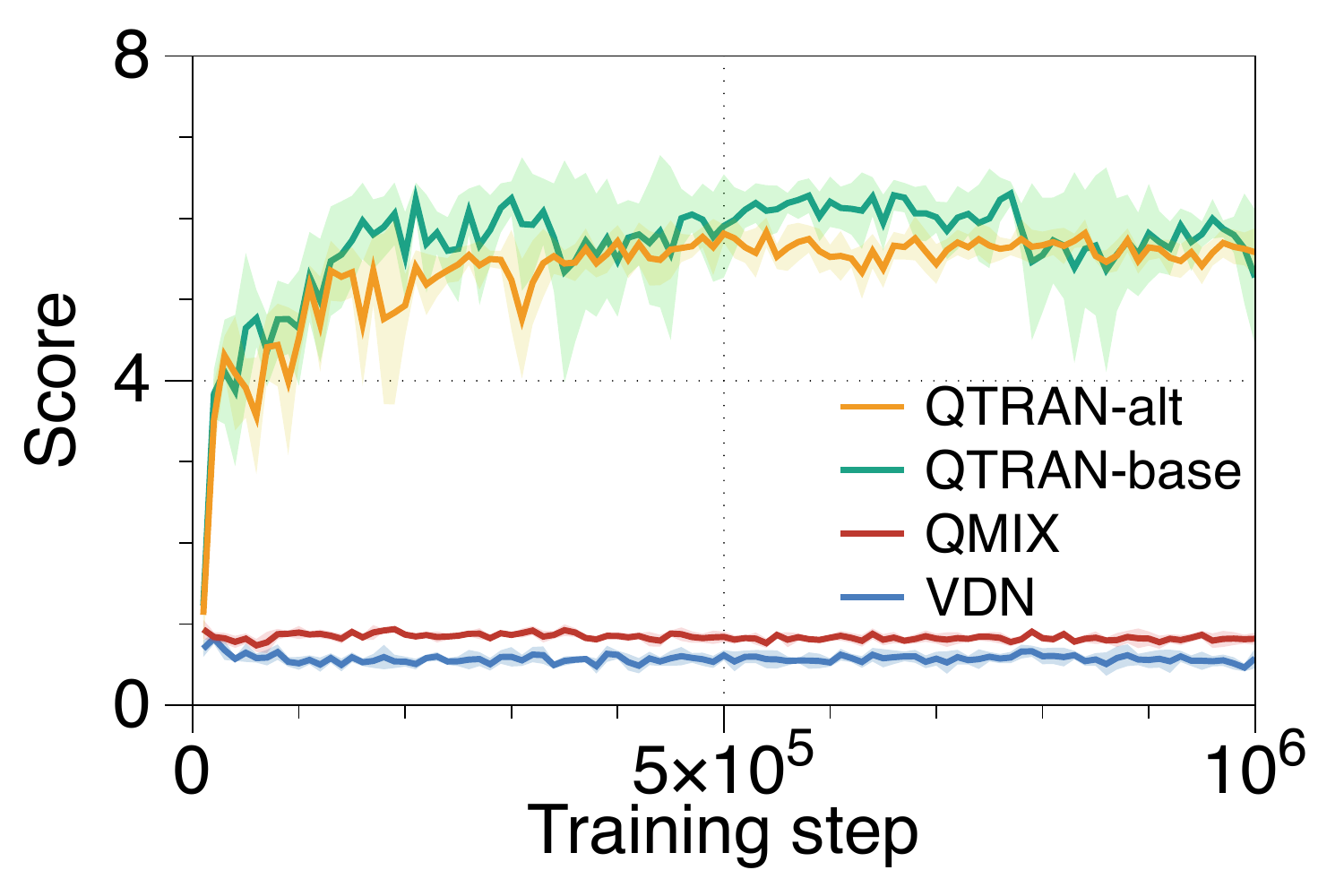}
  \caption{Gaussian Squeeze without epsilon decay}
  \label{fig:GSgraph2}
\end{subfigure}\hspace*{\fill}
\begin{subfigure}[t]{.33\textwidth}
  \includegraphics[width=\linewidth]{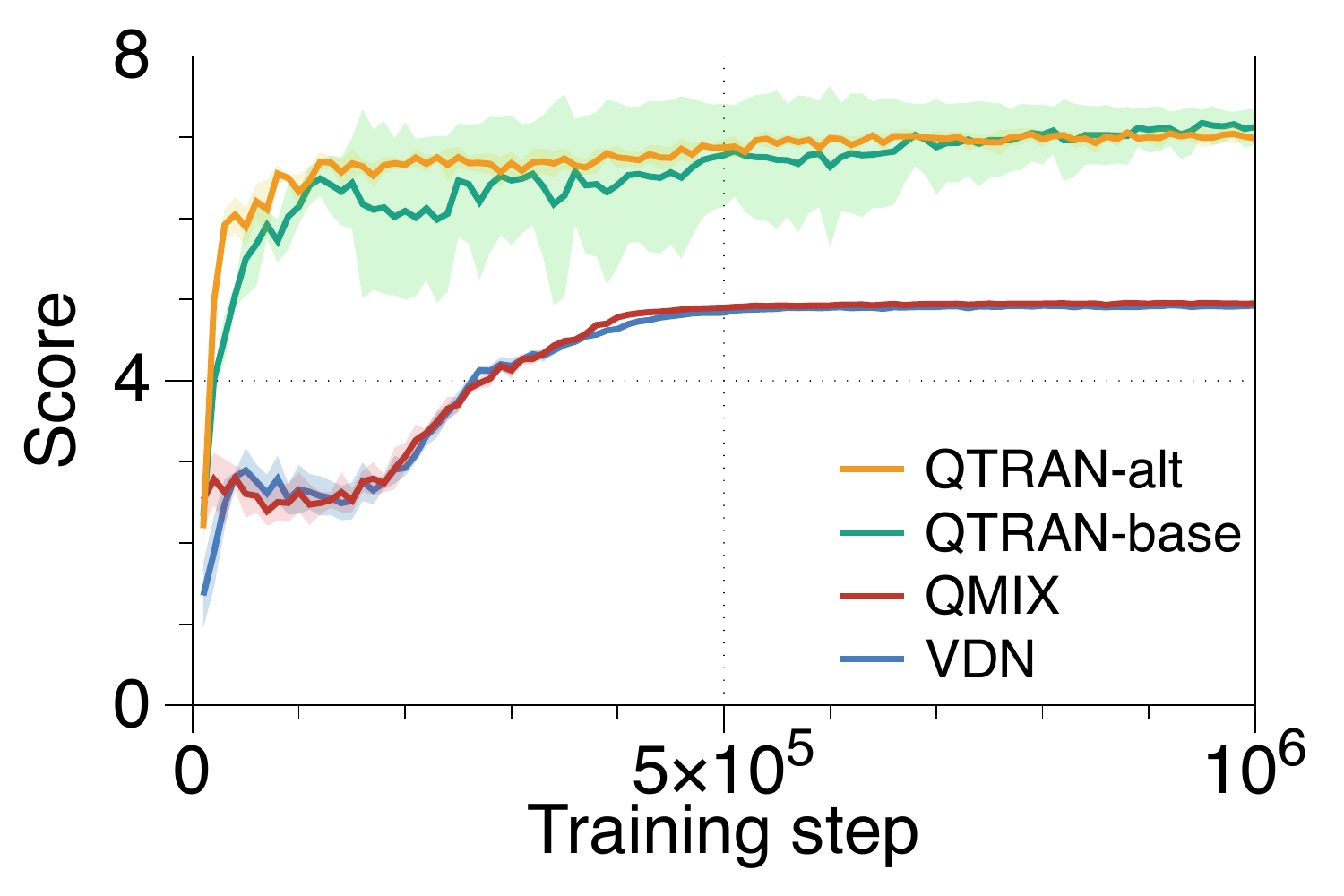}
  \caption{Multi-domain Gaussian Squeeze $K = 2$}
  \label{fig:GSgraph3}
\end{subfigure}\hspace*{\fill}
\caption{Average reward on the GS and GMS tasks with 95\% confidence intervals for VDN, QMIX, and $\algoname$}
\label{fig:GSresult1}
\vspace{-0.3cm}
\end{figure*}

\section{Experiment}


\subsection{Environments}
To demonstrate the performance of $\algoname$, we consider two environments: 
(i) Multi-domain Gaussian Squeeze and (ii) modified predator-prey. 
Details on our implementation of $\algoname$, the source code of our implementation in TensorFlow, and other experimental scripts are available in the Supplementary and a public repository: \url{https://github.com/Sonkyunghwan/QTRAN}.




\myparagraph{Multi-domain Gaussian Squeeze (MGS)}
Gaussian Squeeze (GS) \cite{holmesparker2014exploiting} is a simple non-monotonic multi-agent resource allocation problem, used in other papers, {\em e.g.}, \citet{pmlr-v80-yang18d} and \citet{chen2018factorized}. In GS, multiple homogeneous agents need to work together for efficient resource allocation whilst avoiding congestion. We use GS in conjunction with Multi-domain Gaussian Squeeze (MGS) as follows:
we have ten agents; each agent $i$ takes action $u_i$, which controls the resource usage level, ranging over $\{0,1,...,9\}$. Each agent has its own amount of unit-level resource $s_i \in [0,0.2]$, given by the environment {\em a priori}. Then, a joint action $\bm{u}$ determines the overall resource usage $x(\bm{u}) = \sum_i s_i \times u_i$. We assume that there exist $K$ domains, where the above resource allocation takes place. Then, the goal is to maximize the joint reward defined as $G(\bm{u}) = \sum_{k=1}^K x e^{-(x-\mu_k)^2 / {\sigma_k}^2}$, where $\mu_k$ and $\sigma_k$ are the parameters of each domain. 
Depending on the number of domains, GS has only one local maximum, whereas MGS has multiple local maxima. In our MGS setting, compared to GS, the optimal policy is similar to that in GS, and through this policy, the reward similar to that in GS can be obtained. Additionally, in MGS, a new sub-optimal ``pitfall'' policy that is easier to achieve but is only half as rewarding as the optimal policy.
The case when $K>1$ is usefully utilized to test the algorithms that are required to avoid sub-optimal points in the joint space of actions. The full details on the environment setup and hyperparameters are described in the Supplementary.




\myparagraph{Modified predator-prey (MPP)}
We adopt a more complicated environment by modifying the well-known predator-prey \citep{stone2000multiagent} in the grid world, used in many other MARL research. State and action spaces are constructed similarly to those of the classic predator-prey game. ``Catching'' a prey is equivalent to having the prey within an agent's observation horizon.
We extend it to the scenario that positive reward is given only if multiple predators catch a prey simultaneously, requiring a higher degree of cooperation. The predators get a team reward of 1, if two or more catch a prey at the same time, but they are given negative reward $-P$, when only one predator catches the prey. 


Note that the value of penalty $P$ also determines the degree of monotonicity, {\em i.e.}, the higher $P$ is, the less monotonic the task is. The prey that has been caught regenerated at random positions whenever caught by more than one predator. In our evaluation, we tested up to $N=4$ predators and up to two prey, and the game proceeds over fixed 100 steps. We experimented with six different settings with varying $P$ values and numbers of agents, where $N=2, 4$ and $P = 0.5, 1.0, 1.5$. For the $N=4$ case, we placed two prey; otherwise, just one.
The detailed settings are available in the Supplementary. 

\begin{figure*}[t!]
\begin{subfigure}[t]{.33\textwidth}
  \includegraphics[width=\linewidth]{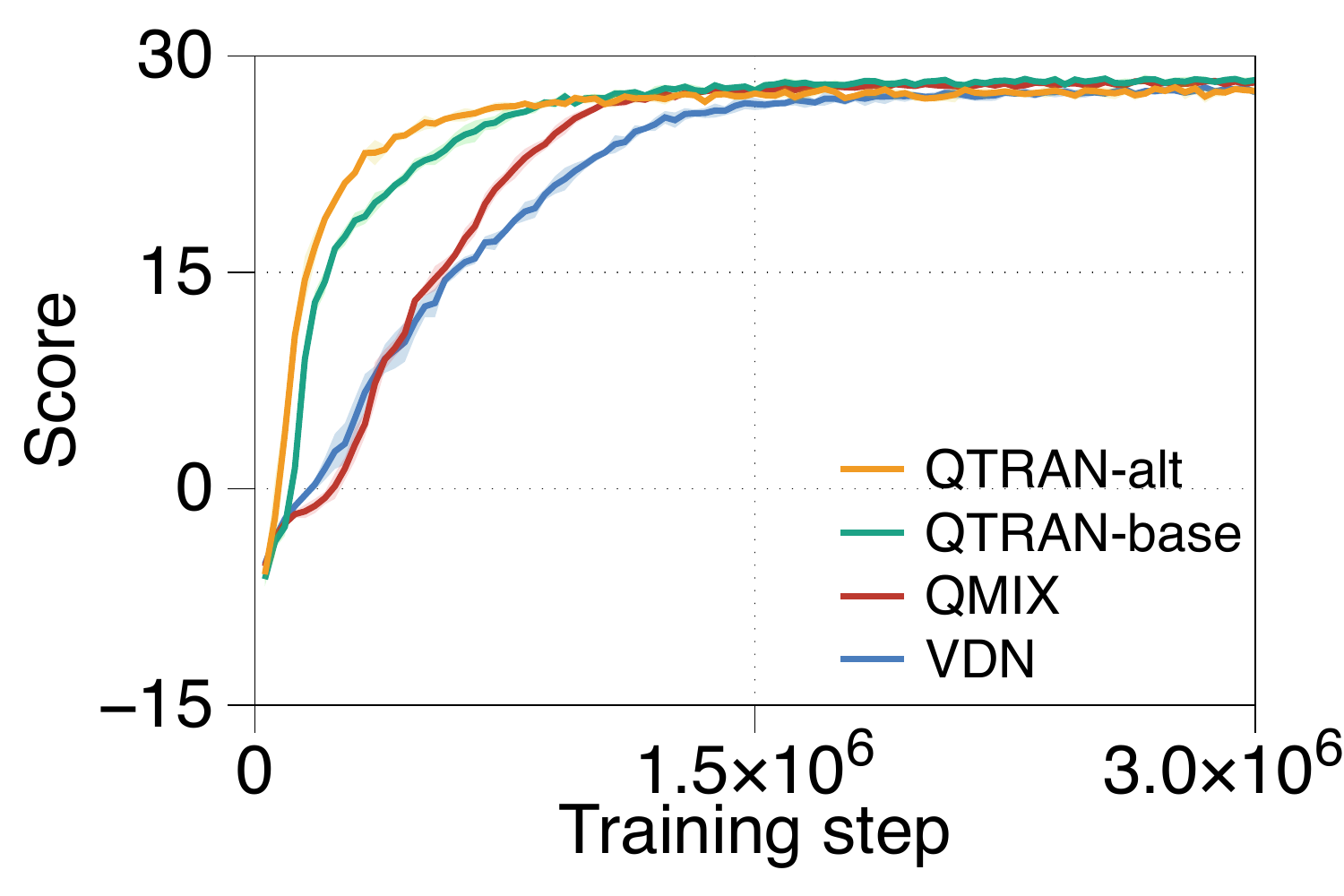}
  \caption{$N=2, P = 0.5$}
  \label{fig:PPgraph1}
\end{subfigure}\hspace*{\fill}
\begin{subfigure}[t]{.33\textwidth}
  \includegraphics[width=\linewidth]{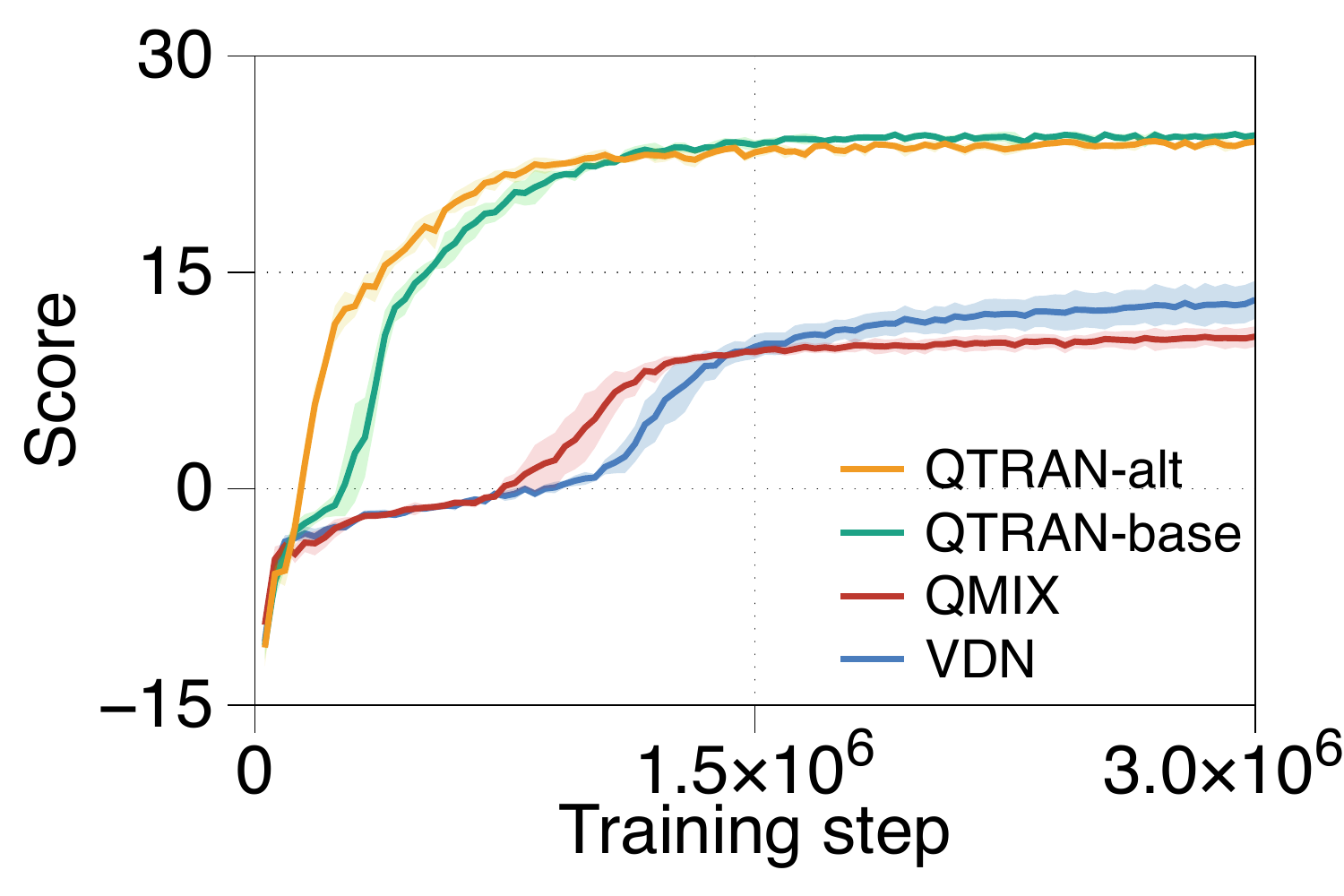}
  \caption{$N=2, P = 1.0$}
  \label{fig:PPgraph2}
\end{subfigure}\hspace*{\fill}
\begin{subfigure}[t]{.33\textwidth}
  \includegraphics[width=\linewidth]{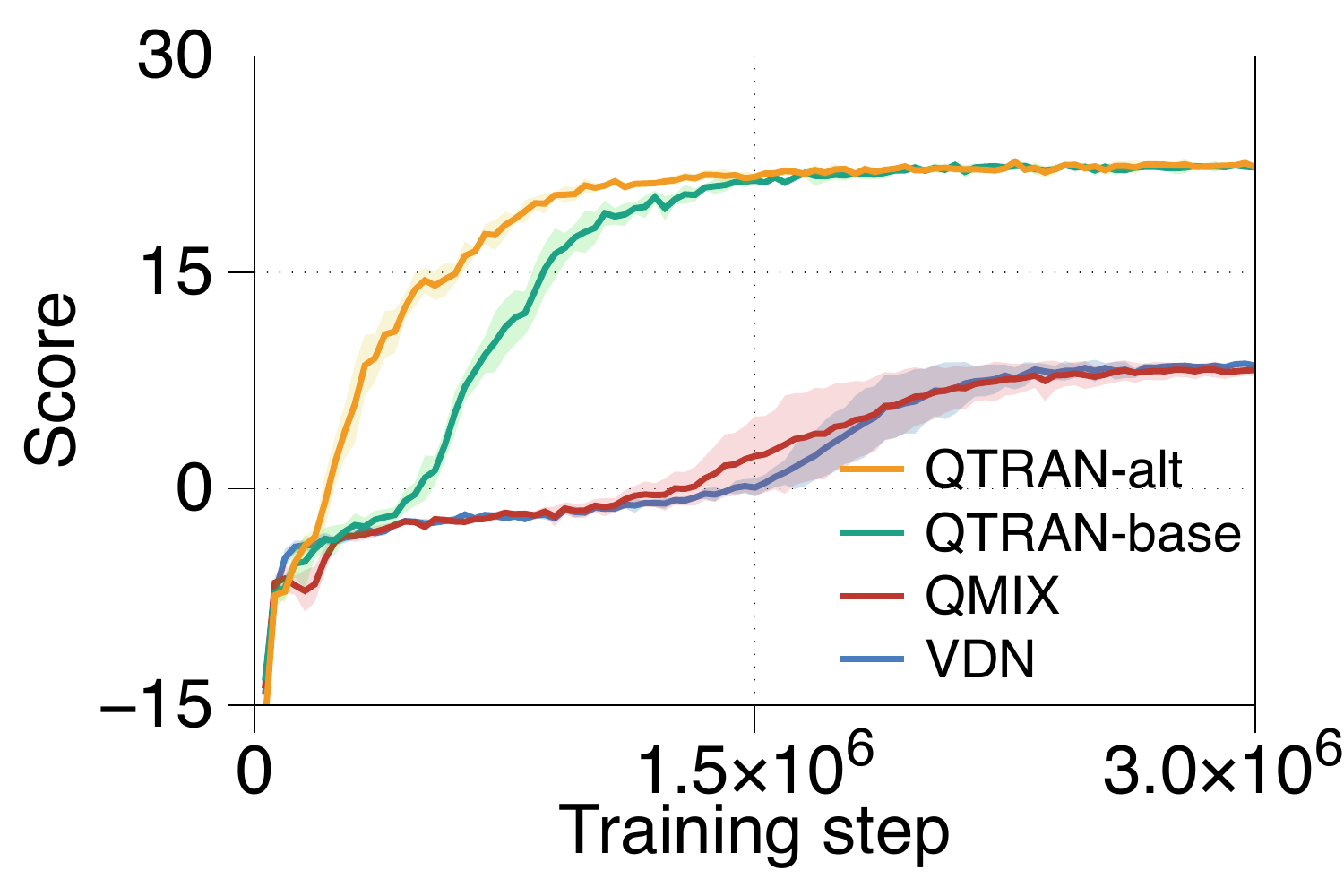}
  \caption{$N=2, P = 1.5$}
  \label{fig:PPgraph3}
\end{subfigure}\hspace*{\fill}\\
\begin{subfigure}[t]{.33\textwidth}
  \includegraphics[width=\linewidth]{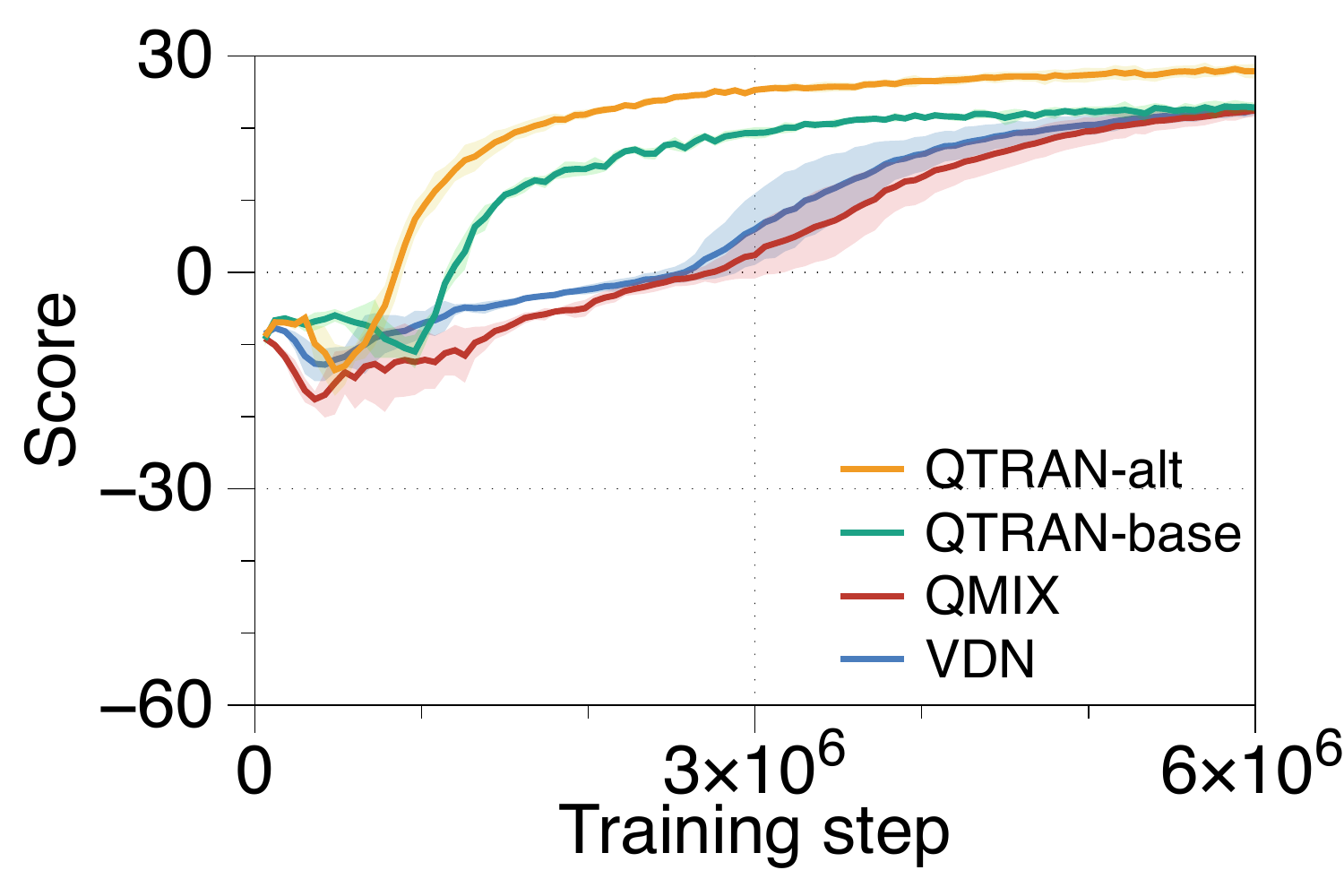}
  \caption{$N=4, P = 0.5$}
  \label{fig:PPgraph4}
\end{subfigure}\hspace*{\fill}
\begin{subfigure}[t]{.33\textwidth}
  \includegraphics[width=\linewidth]{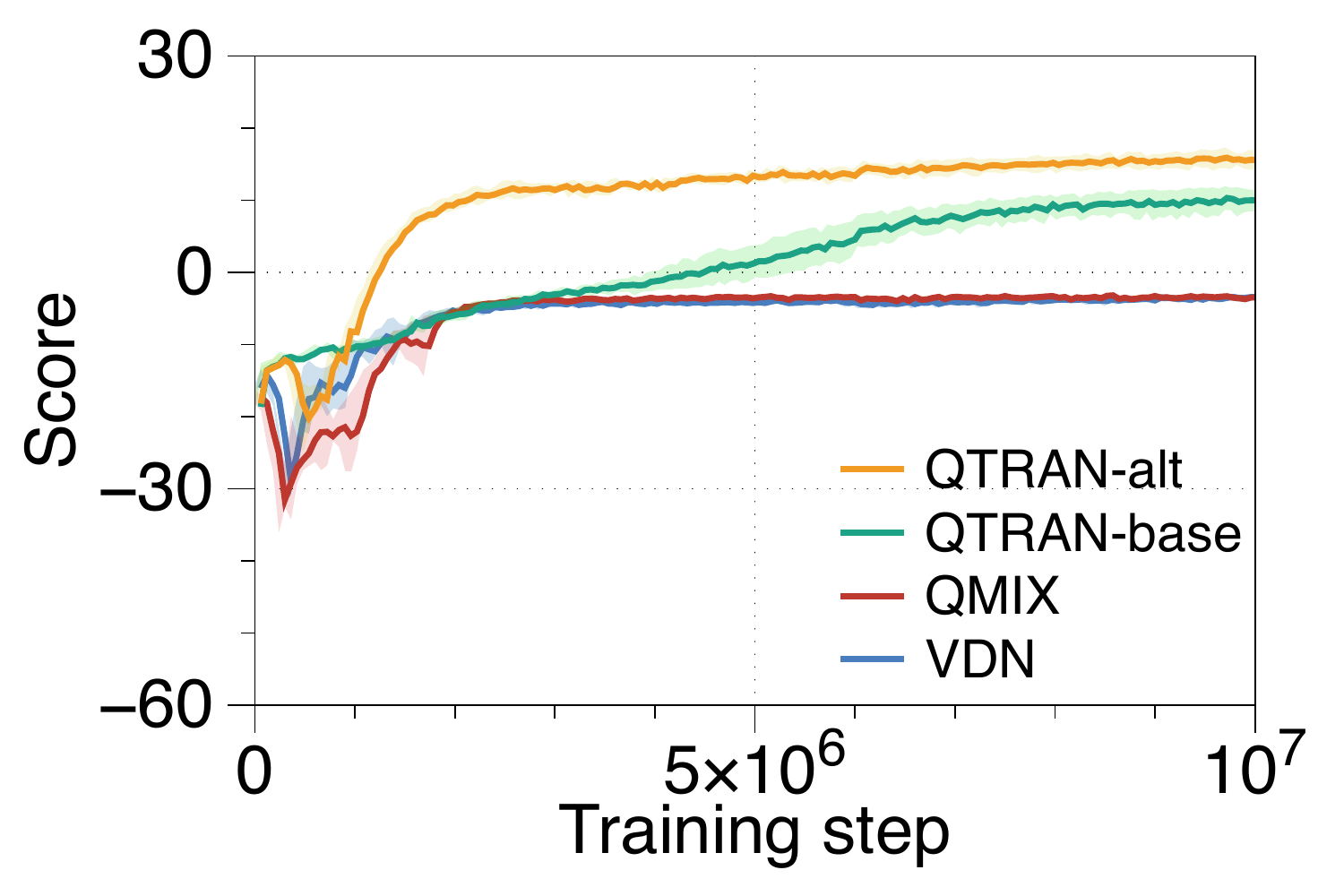}
  \caption{$N=4, P = 1.0$}
  \label{fig:PPgraph5}
\end{subfigure}\hspace*{\fill}
\begin{subfigure}[t]{.33\textwidth}
  \includegraphics[width=\linewidth]{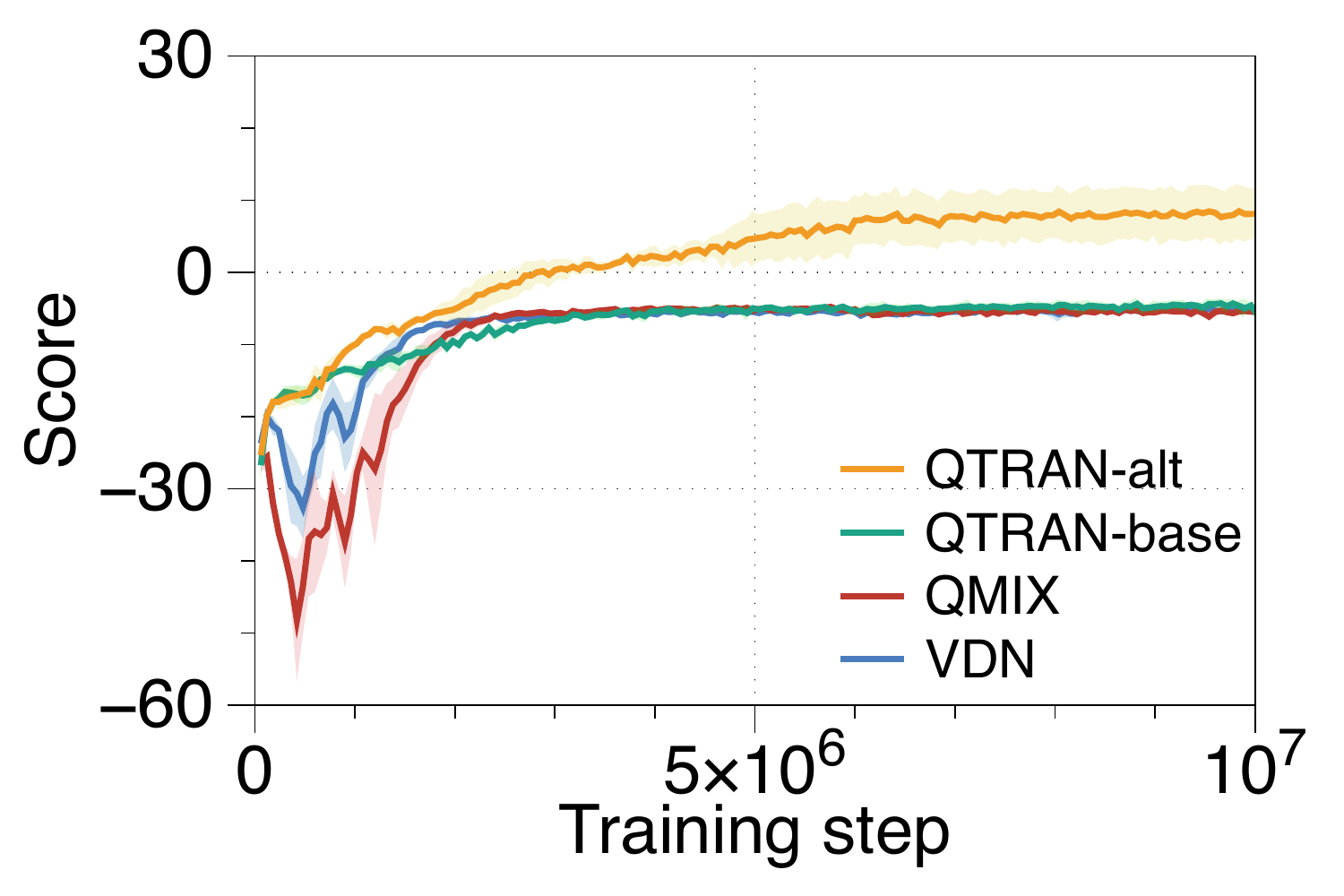}
  \caption{$N=4, P = 1.5$}
  \label{fig:PPgraph6}
\end{subfigure}\hspace*{\fill}
\caption{Average reward per episode on the MPP tasks with 95\% confidence intervals for VDN, QMIX, and $\algoname$}
\label{fig:PPresult1}
\vspace{-0.3cm}
\end{figure*}

\subsection{Results}

\paragraph{Multi-domain Gaussian Squeeze}
Figure~\ref{fig:GSgraph1} shows the result of GS, where it is not surprising to observe that all algorithms converge to the optimal point.   
However, $\algoname$ noticeably differs in its convergence speed; it is capable of handling the non-monotonic nature of the environment more accurately and of finding an optimal policy from well-expressed action-value functions. VDN and QMIX have some structural constraints, which hinder the accurate learning of the action-value functions for non-monotonic structures. 
These algorithms converge to the locally optimal point --- which is the globally optimal point in GS, where $K=1$ --- near the biased samples by a wrong policy with epsilon-decay exploration. 
To support our claim, we experiment with the full exploration without epsilon-decay for the same environment, as shown in Figure~\ref{fig:GSgraph2}. We observe that  $\algoname$ learns more or less the same policy as in Figure~\ref{fig:GSgraph1}, whereas VDN and QMIX significantly deteriorate.
Figure~\ref{fig:GSgraph3} shows the result for a more challenging scenario of MGS, with each of $[s_i]$ being the same, where 
agents can always achieve higher performance than GS. 
VDN and QMIX are shown to learn only a sub-optimal policy in MGS, whose rewards are even smaller than those in GS. 
$\algonamebasic$ and $\algonameadv$ achieve significantly higher rewards, where $\algonameadv$ is more stable, as expected. This is because $\algonameadv$’s alternative loss increases the gap between $\Qtot’$ for non-optimal and optimal actions, and prevents it from being updated to an undesirable policy.



\myparagraph{Modified predator-prey}
Figure~\ref{fig:PPresult1} shows the performance of the three algorithms for six settings with different $N$ and $P$ values, where all results demonstrate the superiority of $\algoname$ to VDN and QMIX. 
In Figures~\ref{fig:PPgraph1} and \ref{fig:PPgraph4} with low penalties $P$, all three algorithms learn the policy that catches the prey well and thus obtain high rewards. However, again, the speed in finding the policy in VDN and QMIX is slower than that in $\algoname$. 
As the penalty $P$ grows and exacerbates the non-monotonic characteristic of the environment, we observe a larger performance gap between $\algoname$ and the two other algorithms. As shown in Figures~\ref{fig:PPgraph2} and \ref{fig:PPgraph3}, even for a large penalty, $\algoname$ agents still cooperate well to catch the prey. However, in VDN and QMIX, agents learn to work together with somewhat limited coordination, so that they do not actively try to catch the prey and instead attempt only to minimize the penalty-receiving risk. In Figures~\ref{fig:PPgraph5} and \ref{fig:PPgraph6}, when the number of agents $N$ grows, VDN and QMIX do not cooperate at all and learn only a sub-optimal policy: running away from the prey to minimize the risks of being penalized.

For the tested values of $N$ and $P$, we note that the performance gap between $\algonamebasic$ and $\algonameadv$ is influenced more strongly by $N$.
When $N=2$, the gap in the convergence speed between $\algonamebasic$ and $\algonameadv$ increases as the penalty grows. Nevertheless, both algorithms ultimately turn out to train the agents to cooperate well for every penalty value tested.
On the other hand, with $N=4$, the gap of convergence speed between $\algonamebasic$ and $\algonameadv$ becomes larger. With more agents, as shown in Figures~\ref{fig:PPgraph4} and \ref{fig:PPgraph5}, $\algonameadv$ achieves higher scores than $\algonamebasic$ does, and $\algonamebasic$ requires longer times to reach positive scores. Figure~\ref{fig:PPgraph6} shows that $\algonamebasic$ does not achieve a positive reward until the end, implying that $\algonamebasic$ fails to converge, whereas $\algonameadv$'s score increases, with training steps up to $10^7$.

\section{Conclusion}
This paper presented $\algoname$, a learning method to factorize the joint action-value functions of a wide variety of MARL tasks. $\algoname$ takes advantage of centralized training and fully decentralized execution of the learned policies by appropriately transforming and factorizing the joint action-value function into individual action-value functions. Our theoretical analysis demonstrates that $\algoname$ handles a richer class of tasks than its predecessors, and our simulation results indicate that $\algoname$ outperforms VDN and QMIX by a substantial margin, especially so when the game exhibits more severe non-monotonic characteristics. 
\section*{Acknowledgements}
This work was supported by Institute for Information \& communications
Technology Planning \& Evaluation(IITP) grant funded by the Korea government(MSIT) (No.2018-0-00170,Virtual Presence in Moving Objects through 5G (PriMo-5G))

This research was supported by Basic Science Research Program through the National Research Foundation of Korea(NRF) funded by the Ministry of  Science and ICT(No. 2016R1A2A2A05921755)
\bibliography{ref.bib}

\begin{thebibliography}{30}
\providecommand{\natexlab}[1]{#1}
\providecommand{\url}[1]{\texttt{#1}}
\expandafter\ifx\csname urlstyle\endcsname\relax
  \providecommand{\doi}[1]{doi: #1}\else
  \providecommand{\doi}{doi: \begingroup \urlstyle{rm}\Url}\fi

\bibitem[Chen et~al.(2018)Chen, Zhou, Wen, Yang, Su, Zhang, Zhang, Wang, and
  Liu]{chen2018factorized}
Chen, Y., Zhou, M., Wen, Y., Yang, Y., Su, Y., Zhang, W., Zhang, D., Wang, J.,
  and Liu, H.
\newblock Factorized q-learning for large-scale multi-agent systems.
\newblock \emph{arXiv preprint arXiv:1809.03738}, 2018.

\bibitem[Foerster et~al.(2016)Foerster, Assael, de~Freitas, and
  Whiteson]{foerster2016learning}
Foerster, J., Assael, I.~A., de~Freitas, N., and Whiteson, S.
\newblock Learning to communicate with deep multi-agent reinforcement learning.
\newblock In \emph{Proceedings of the Advances in Neural Information Processing
  Systems}, pp.\  2137--2145, 2016.

\bibitem[Foerster et~al.(2018)Foerster, Farquhar, Afouras, Nardelli, and
  Whiteson]{DBLP:conf/aaai/FoersterFANW18}
Foerster, J.~N., Farquhar, G., Afouras, T., Nardelli, N., and Whiteson, S.
\newblock Counterfactual multi-agent policy gradients.
\newblock In \emph{Proceedings of {AAAI}}, 2018.

\bibitem[Guestrin et~al.(2002{\natexlab{a}})Guestrin, Koller, and
  Parr]{guestrin2002multiagent}
Guestrin, C., Koller, D., and Parr, R.
\newblock Multiagent planning with factored mdps.
\newblock In \emph{Proceedings of the Advances in Neural Information Processing
  Systems}, pp.\  1523--1530, 2002{\natexlab{a}}.

\bibitem[Guestrin et~al.(2002{\natexlab{b}})Guestrin, Lagoudakis, and
  Parr]{guestrin2002coordinated}
Guestrin, C., Lagoudakis, M., and Parr, R.
\newblock Coordinated reinforcement learning.
\newblock In \emph{Proceedings of ICML}, pp.\  227--234, 2002{\natexlab{b}}.

\bibitem[Gupta et~al.(2017)Gupta, Egorov, and
  Kochenderfer]{gupta2017cooperative}
Gupta, J.~K., Egorov, M., and Kochenderfer, M.
\newblock Cooperative multi-agent control using deep reinforcement learning.
\newblock In \emph{Proceedings of International Conference on Autonomous Agents
  and Multiagent Systems}, 2017.

\bibitem[HolmesParker et~al.(2014)HolmesParker, Taylor, Zhan, and
  Tumer]{holmesparker2014exploiting}
HolmesParker, C., Taylor, M., Zhan, Y., and Tumer, K.
\newblock Exploiting structure and agent-centric rewards to promote
  coordination in large multiagent systems.
\newblock In \emph{Proceedings of Adaptive and Learning Agents Workshop}, 2014.

\bibitem[Jiang \& Lu(2018)Jiang and Lu]{NIPS2018_7956}
Jiang, J. and Lu, Z.
\newblock Learning attentional communication for multi-agent cooperation.
\newblock In \emph{Proceedings of NIPS}, pp.\  7265--7275, 2018.

\bibitem[Kim et~al.(2019)Kim, Moon, Hostallero, Kang, Lee, Son, and
  Yi]{kim2018learning}
Kim, D., Moon, S., Hostallero, D., Kang, W.~J., Lee, T., Son, K., and Yi, Y.
\newblock Learning to schedule communication in multi-agent reinforcement
  learning.
\newblock In \emph{Proceedings of ICLR}, 2019.

\bibitem[Kok \& Vlassis(2006)Kok and Vlassis]{kok2006collaborative}
Kok, J.~R. and Vlassis, N.
\newblock Collaborative multiagent reinforcement learning by payoff
  propagation.
\newblock \emph{Journal of Machine Learning Research}, 7\penalty0
  (Sep):\penalty0 1789--1828, 2006.

\bibitem[Koller \& Parr(1999)Koller and Parr]{koller1999computing}
Koller, D. and Parr, R.
\newblock Computing factored value functions for policies in structured mdps.
\newblock In \emph{Proceedings of IJCAI}, pp.\  1332--1339, 1999.

\bibitem[Lillicrap et~al.(2015)Lillicrap, Hunt, Pritzel, Heess, Erez, Tassa,
  Silver, and Wierstra]{lillicrap2015continuous}
Lillicrap, T.~P., Hunt, J.~J., Pritzel, A., Heess, N., Erez, T., Tassa, Y.,
  Silver, D., and Wierstra, D.
\newblock Continuous control with deep reinforcement learning.
\newblock \emph{arxiv preprint arXiv:1509.02971}, 2015.

\bibitem[Lowe et~al.(2017)Lowe, WU, Tamar, Harb, Pieter~Abbeel, and
  Mordatch]{Lowe:MADDPG}
Lowe, R., WU, Y., Tamar, A., Harb, J., Pieter~Abbeel, O., and Mordatch, I.
\newblock Multi-agent actor-critic for mixed cooperative-competitive
  environments.
\newblock In \emph{Proceedings of NIPS}, pp.\  6379--6390, 2017.

\bibitem[Mnih et~al.(2015)Mnih, Kavukcuoglu, Silver, Rusu, Veness, Bellemare,
  Graves, Riedmiller, Fidjeland, Ostrovski, et~al.]{mnih2015human}
Mnih, V., Kavukcuoglu, K., Silver, D., Rusu, A.~A., Veness, J., Bellemare,
  M.~G., Graves, A., Riedmiller, M., Fidjeland, A.~K., Ostrovski, G., et~al.
\newblock Human-level control through deep reinforcement learning.
\newblock \emph{Nature}, 518\penalty0 (7540):\penalty0 529, 2015.

\bibitem[Oliehoek et~al.(2008)Oliehoek, Spaan, and
  Vlassis]{oliehoek2008optimal}
Oliehoek, F.~A., Spaan, M.~T., and Vlassis, N.
\newblock Optimal and approximate q-value functions for decentralized pomdps.
\newblock \emph{Journal of Artificial Intelligence Research}, 32:\penalty0
  289--353, 2008.

\bibitem[Oliehoek et~al.(2016)Oliehoek, Amato, et~al.]{oliehoek2016concise}
Oliehoek, F.~A., Amato, C., et~al.
\newblock \emph{A concise introduction to decentralized POMDPs}, volume~1.
\newblock Springer, 2016.

\bibitem[Omidshafiei et~al.(2017)Omidshafiei, Pazis, Amato, How, and
  Vian]{Omid:Hysteretic}
Omidshafiei, S., Pazis, J., Amato, C., How, J.~P., and Vian, J.
\newblock Deep decentralized multi-task multi-agent reinforcement learning
  under partial observability.
\newblock \emph{arxiv preprint arXiv:1703.06182}, 2017.

\bibitem[Rashid et~al.(2018)Rashid, Samvelyan, Schroeder, Farquhar, Foerster,
  and Whiteson]{pmlr-v80-rashid18a}
Rashid, T., Samvelyan, M., Schroeder, C., Farquhar, G., Foerster, J., and
  Whiteson, S.
\newblock {QMIX}: Monotonic value function factorisation for deep multi-agent
  reinforcement learning.
\newblock In \emph{Proceedings of ICML}, 2018.

\bibitem[Silver et~al.(2016)Silver, Huang, Maddison, Guez, Sifre, Van
  Den~Driessche, Schrittwieser, Antonoglou, Panneershelvam, Lanctot,
  et~al.]{silver2016mastering}
Silver, D., Huang, A., Maddison, C.~J., Guez, A., Sifre, L., Van Den~Driessche,
  G., Schrittwieser, J., Antonoglou, I., Panneershelvam, V., Lanctot, M.,
  et~al.
\newblock Mastering the game of go with deep neural networks and tree search.
\newblock \emph{Nature}, 529\penalty0 (7587):\penalty0 484, 2016.

\bibitem[Silver et~al.(2017)Silver, Schrittwieser, Simonyan, Antonoglou, Huang,
  Guez, Hubert, Baker, Lai, Bolton, et~al.]{silver2017mastering}
Silver, D., Schrittwieser, J., Simonyan, K., Antonoglou, I., Huang, A., Guez,
  A., Hubert, T., Baker, L., Lai, M., Bolton, A., et~al.
\newblock Mastering the game of go without human knowledge.
\newblock \emph{Nature}, 550\penalty0 (7676):\penalty0 354, 2017.

\bibitem[Stone \& Veloso(2000)Stone and Veloso]{stone2000multiagent}
Stone, P. and Veloso, M.
\newblock Multiagent systems: A survey from a machine learning perspective.
\newblock \emph{Autonomous Robots}, 8\penalty0 (3):\penalty0 345--383, 2000.

\bibitem[Sukhbaatar et~al.(2016)Sukhbaatar, Fergus,
  et~al.]{sukhbaatar2016learning}
Sukhbaatar, S., Fergus, R., et~al.
\newblock Learning multiagent communication with backpropagation.
\newblock In \emph{Proceedings of the Advances in Neural Information Processing
  Systems}, pp.\  2244--2252, 2016.

\bibitem[Sunehag et~al.(2018)Sunehag, Lever, Gruslys, Czarnecki, Zambaldi,
  Jaderberg, Lanctot, Sonnerat, Leibo, Tuyls, and
  Graepel]{DBLP:conf/atal/SunehagLGCZJLSL18}
Sunehag, P., Lever, G., Gruslys, A., Czarnecki, W.~M., Zambaldi, V.~F.,
  Jaderberg, M., Lanctot, M., Sonnerat, N., Leibo, J.~Z., Tuyls, K., and
  Graepel, T.
\newblock Value-decomposition networks for cooperative multi-agent learning
  based on team reward.
\newblock In \emph{Proceedings of {AAMAS}}, pp.\  2085--2087, 2018.

\bibitem[Tampuu et~al.(2017)Tampuu, Matiisen, Kodelja, Kuzovkin, Korjus, Aru,
  Aru, and Vicente]{tampuu2017multiagent}
Tampuu, A., Matiisen, T., Kodelja, D., Kuzovkin, I., Korjus, K., Aru, J., Aru,
  J., and Vicente, R.
\newblock Multiagent cooperation and competition with deep reinforcement
  learning.
\newblock \emph{PloS one}, 12\penalty0 (4):\penalty0 e0172395, 2017.

\bibitem[Tan(1993)]{tan1993multi}
Tan, M.
\newblock Multi-agent reinforcement learning: Independent vs. cooperative
  agents.
\newblock In \emph{Proceedings of ICML}, pp.\  330--337, 1993.

\bibitem[Wang et~al.(2016)Wang, Schaul, Hessel, Hasselt, Lanctot, and
  Freitas]{wang2016dueling}
Wang, Z., Schaul, T., Hessel, M., Hasselt, H., Lanctot, M., and Freitas, N.
\newblock Dueling network architectures for deep reinforcement learning.
\newblock In \emph{Proceedings of the 33rd International Conference on Machine
  Learning}, pp.\  1995--2003, 2016.

\bibitem[Watkins(1989)]{watkins1989learning}
Watkins, C. J. C.~H.
\newblock \emph{Learning from delayed rewards}.
\newblock PhD thesis, King's College, Cambridge, 1989.

\bibitem[Wei \& Luke(2016)Wei and Luke]{wei2016lenient}
Wei, E. and Luke, S.
\newblock Lenient learning in independent-learner stochastic cooperative games.
\newblock \emph{The Journal of Machine Learning Research}, 17\penalty0
  (1):\penalty0 2914--2955, 2016.

\bibitem[Wei et~al.(2018)Wei, Wicke, Freelan, and Luke]{wei2018multiagent}
Wei, E., Wicke, D., Freelan, D., and Luke, S.
\newblock Multiagent soft q-learning.
\newblock \emph{arXiv preprint arXiv:1804.09817}, 2018.

\bibitem[Yang et~al.(2018)Yang, Luo, Li, Zhou, Zhang, and
  Wang]{pmlr-v80-yang18d}
Yang, Y., Luo, R., Li, M., Zhou, M., Zhang, W., and Wang, J.
\newblock Mean field multi-agent reinforcement learning.
\newblock In \emph{Proceedings of ICML}, pp.\  5567--5576, 2018.

\end{thebibliography}
\appendix
\onecolumn

\vspace{1cm}
\begin{center}
\huge Supplementary 
\end{center}

\section{$\algoname$ Training Algorithm}

The training algorithms for $\algonamebasic$ and $\algonameadv$ are provided in
Algorithm~\ref{algo:qtran}.

\begin{algorithm}[H]
 \caption{$\algonamebasic$ and $\algonameadv$}
 \label{algo:qtran}
 \begin{algorithmic}[1]
 \STATE Initialize replay memory $D$
 \STATE Initialize $[Q_i]$, $\Qtot$, and $\Vtot$ with random parameters $\theta$
 \STATE Initialize target parameters $\theta^- = \theta$
 \FOR{episode = 1 to $M$}
 \STATE Observe initial state $\bm{s}^0$ and observation $\bm{o}^0 = [O(\bm{s}^0,i)]_{i=1}^N$ for each agent $i$
    \FOR{$t = 1$ to $T$}
        \STATE With probability $\epsilon$ select a random action $u^t_i$
        \STATE Otherwise $u^t_i = \arg\max_{u^t_{i}}Q_i(\tau^t_i,u^t_i)$ for each agent $i$
        \STATE Take action $\bm{u}^t$, and retrieve next observation and reward $(\bm{o}^{t+1},r^t)$
        \STATE Store transition $(\bm{\tau}^t, \bm{u}^t, r^t, \bm{\tau}^{t+1})$ in $D$
        \STATE Sample a random minibatch of transitions $(\bm{\tau}, \bm{u}, r, \bm{\tau'})$ from $D$
        \STATE Set $y^{\text{dqn}}(r, \bm{\tau}'; \bm{\theta^{-}}) = r + \gamma
\Qtot(\bm{\tau}', \bm{\bar{u}}';\bm{\theta} ^{-}),$ $ \bar{\bm{u}}' = [\arg\max_{u_{i}}Q_i(\tau'_i,u_i;\bm{\theta}^{-})]_{i=1}^N,$
        \STATE If $\algonamebasic,$ update $\theta$ by minimizing the loss:
        \begin{align*}
L(\bm{\tau}, \bm{u}, r, \bm{\tau}'; \bm{\theta}) &=
 L_{\text{td}} + \lambda_{\text{opt}} L_{\text{opt}} + \lambda_\text{nopt} L_\text{nopt},\cr
L_{\text{td}}(\bm{\tau}, \bm{u}, r, \bm{\tau}';\bm{\theta}) &= \big(\Qtot(\bm{\tau},\bm{u}) - y^{\text{dqn}}(r, \bm{\tau}'; \bm{\theta}^{-})\big)^{2},\cr 
L_{\text{opt}}(\bm{\tau}, \bm{u}, r, \bm{\tau}';\bm{\theta})&= \big(\Qtot'(\bm{\tau},\bm{\bar{u}}) - \hQtot(\bm{\tau},\bm{\bar{u}}) + \Vtot(\bm{\tau})\big)^{2}, \cr 
L_{\text{nopt}}(\bm{\tau}, \bm{u}, r, \bm{\tau}';\bm{\theta})&= \left(\min \big[\Qtot'(\bm{\tau},\bm{u}) - \hQtot(\bm{\tau},\bm{u}) + \Vtot(\bm{\tau}),0\big] \right)^{2}. 
        \end{align*}
        \STATE If $\algonameadv,$ update $\theta$ by minimizing the loss:
        \begin{align*}
L(\bm{\tau}, \bm{u}, r, \bm{\tau}'; \bm{\theta})&=
 L_{\text{td}} + \lambda_{\text{opt}} L_{\text{opt}} + \lambda_\text{nopt-min} L_\text{nopt-min},\cr
L_{\text{td}}(\bm{\tau}, \bm{u}, r, \bm{\tau}';\bm{\theta})&= \big(\Qtot(\bm{\tau},\bm{u}) - y^{\text{dqn}}(r, \bm{\tau}'; \bm{\theta}^{-})\big)^{2},\cr 
L_{\text{opt}}(\bm{\tau}, \bm{u}, r, \bm{\tau}';\bm{\theta})&= \big(\Qtot'(\bm{\tau},\bm{\bar{u}}) - \hQtot(\bm{\tau},\bm{\bar{u}}) + \Vtot(\bm{\tau})\big)^{2}, \cr 
L_{\text{nopt-min}}(\bm{\tau}, \bm{u}, r, \bm{\tau}'; \bm{\theta})&= {1 \over N}\sum_{i=1}^N\left(\min_{u_i \in \set{U}} \left(\Qtot'(\bm{\tau},u_i,\bm{u}_{-i}) -
\hQtot(\bm{\tau},u_i,\bm{u}_{-i}) + 
\Vtot(\bm{\tau})\right)\right)^{2}.
        \end{align*}
        \STATE Update target network parameters $\theta^- = \theta$ with period $I$
 \ENDFOR
 \ENDFOR
 \end{algorithmic}
 \end{algorithm}

\section{Proofs}
In this section, we provide the proofs of theorems and propositions
coming from theorems. 

\subsection{Proof of Theorem~\ref{thm:ifonlyif}}

{\bf Theorem 1.}
{\em A factorizable joint action-value function $\Qtot(\bm{\tau},\bm{u})$ is factorized by $[Q_i(\tau_i,u_i)],$ if }

\begin{subequations}
  \begin{align*}[left ={\sum_{i=1}^{N} Q_i(\tau_i,u_i) -
    \Qtot(\bm{\tau},\bm{u}) + \Vtot(\bm{\tau}) = \empheqlbrace}]
    & 0         &\hspace{-2.5cm}\bm{u} = \bm{\bar{u}}, \hspace{2.5cm} \tag{\ref{eq:regular_cond1}}\\
    & \geq0  &\hspace{-2.5cm}\bm{u} \neq \bm{\bar{u}}, \hspace{2.5cm}\tag{\ref{eq:regular_cond2}}
   \end{align*}
  \end{subequations}

{\em where }
\begin{align*}
\Vtot(\bm{\tau}) &= \max_{\bm{u}}\Qtot(\bm{\tau},\bm{u}) - \sum_{i=1}^N Q_i(\tau_i,\bar{u}_i).
\end{align*}

\begin{proof}

  Theorem~\ref{thm:ifonlyif} shows that if condition
  \eqref{eq:regular_cond0} holds, then $Q_i$ satisfies {\bf IGM}. Thus,
  for some given $Q_i$ that satisfies \eqref{eq:regular_cond0}, we will
  show that
  $\arg\max_{\bm{u}} \Qtot(\bm{\tau},\bm{u}) = \bm{\bar{u}}$. Recall
  that $\bar{u}_{i} = \arg\max_{u_{i}}Q_i(\tau_i,u_i)$ and
  $\bm{\bar{u}} = [\bar{u}_{i}]_{i=1}^N,$.
\begin{align*}
\Qtot(\bm{\tau},\bm{\bar{u}}) &= \sum_{i=1}^N Q_i(\tau_i,\bar{u}_i) + \Vtot(\bm{\tau}) \quad(\text{From} \; \eqref{eq:regular_cond1}) \\
&\geq \sum_{i=1}^N Q_i(\tau_i,u_i) + \Vtot(\bm{\tau}) \\
&\geq \Qtot(\bm{\tau},\bm{u}) \quad(\text{From} \; \eqref{eq:regular_cond2}).
\end{align*}
It means that the set of optimal local actions $\bar{\bm u}$ maximizes
$\Qtot$, showing that $Q_i$ satisfies {\bf IGM}.
This completes the proof.
\end{proof}


\subsection{Necessity in Theorem~\ref{thm:ifonlyif} Under Affine-transformation}

As mentioned in Section~\ref{sec:transformation}, the conditions
\eqref{eq:regular_cond0} in Theroem~\ref{thm:ifonlyif} are necessary under an affine transformation. The necessary condition shows that for some given factorizable $\Qtot$, there exists $Q_i$ that
satisfies \eqref{eq:regular_cond0}, which guides us to design
the $\algoname$ neural network. Note that the affine transformation
$\phi$ is $\phi(\bm{Q}) = A\cdot \bm{Q} +B$, where
$A=[a_{ii}] \in \real_+^{N\times N}$ is a symmetric diagonal matrix
with $a_{ii} > 0, \forall i$ and $B=[b_i] \in \real^N$. To abuse notation, let
$\phi(Q_i(\tau_i,u_i)) = a_{ii}Q_i(\tau_i,u_i) +b_i$.

\begin{prop}
  If $\Qtot(\bm{\tau},\bm{u})$ is factorized by $[Q_i(\tau_i,u_i)]$,
  then there exists an affine transformation $\phi(\bm{Q})$ such that
  $\Qtot(\bm{\tau},\bm{u})$ is factorized by
  $[\phi(Q_i(\tau_i,u_i))]$ and the condition \eqref{eq:regular_cond0} holds by
  replacing $[Q_i(\tau_i,u_i)]$ with $[\phi(Q_i(\tau_i,u_i))].$
\end{prop}


\begin{proof}
To prove, we will show
  that, for the factors $[Q_i]$ of $\Qtot$, there exists an affine
  transformation of $Q_i$ that also satisfies conditions
  \eqref{eq:regular_cond0}.

  By definition, if $\Qtot(\bm{\tau},\bm{u})$ is factorized by
  $[Q_i(\tau_i,u_i)]$, then the followings hold: (i)
  $\Qtot(\bm{\tau},\bar{\bm{u}}) -
  \max_{\bm{u}}\Qtot(\bm{\tau},\bm{u}) = 0$, (ii)
  $\Qtot(\bm{\tau},\bm{u}) - \Qtot(\bm{\tau},\bar{\bm{u}}) < 0$, and (iii) $\sum_{i=1}^N (Q_i(\tau_i,u_i) - Q_i(\tau_i,\bar{u}_i)) < 0$ if
  $\bm{u} \neq \bm{\bar{u}}$. Now, we consider an affine
  transformation, in which $a_{ii} = \alpha$ and $b_i=0 ~ \forall i$,
  where $\alpha > 0$, and $\phi(Q_i) = \alpha Q_i$ with this
  transformation. Since this is a linearly scaled transformation, it
  satisfies the {\bf IGM} condition, and thus \eqref{eq:regular_cond1}
  holds. We also prove that $\phi(Q_i)$ satisfies condition
  \eqref{eq:regular_cond1} by showing that there exists a constant $\alpha$ small enough such that
  \begin{align*}
    \sum_{i=1}^N \alpha Q_i(\tau_i,u_i) - \Qtot(\bm{\tau},\bm{u}) + \Vtot(\bm{\tau},\bm{u})
    = \sum_{i=1}^N \alpha (Q_i(\tau_i,u_i) -  Q_i(\tau_i,\bar{u}_i)) - (\Qtot(\bm{\tau},\bm{u}) - \Qtot(\bm{\tau},\bar{\bm{u}}))
    \geq 0,
  \end{align*}
  where $\Vtot(\bm{\tau})$ is redefined for linearly scaled
  $\alpha Q_i$ as
  $\max_{\bm{u}}\Qtot(\bm{\tau},\bm{u}) - \sum_{i=1}^N \alpha
  Q_i(\tau_i,\bar{u}_i).$ This completes the proof. 
\end{proof}


\subsection{Special Case: Theorem~\ref{thm:ifonlyif} in Fully Observable Environments}

If the task is a fully observable case (observation function is bijective for all $i$), the state-value network is not required
and all $\Vtot(\tau)$ values can be set to zero. We show that
Theorem~\ref{thm:ifonlyif} holds equally for the case where
$\Vtot(\bm{\tau}) = 0$ for a fully observable case. This fully observable case is applied to our example of the simple matrix game. The similar necessity under an affine-transformation holds in this case. 

{\bf Theorem 1a.} (Fully observable case)
{\em A factorizable joint action-value function $\Qtot(\bm{\tau},\bm{u})$ is factorized by $[Q_i(\tau_i,u_i)],$ if }
\begin{subnumcases}{\label{eq:regular_cond3}\hspace{-0.7cm}\sum_{i=1}^N Q_i(\tau_i,u_i) \! -\! \Qtot(\bm{\tau},\bm{u}) \!=\!}
    \!\!\! 0 & $\hspace{-0.5cm}\bm{u} = \bm{\bar{u}},$ \label{eq:regular_cond4}\\
    \!\!\! \geq 0 & $\hspace{-0.5cm}\bm{u} \neq \bm{\bar{u},}$ \label{eq:regular_cond5}
\end{subnumcases}

\begin{proof}
  We will show
  $\arg\max_{\bm{u}} \Qtot(\bm{\tau},\bm{u}) = \bm{\bar{u}}$ ({\em i.e.},
  {\bf IGM}), if the \eqref{eq:regular_cond3} holds.
\begin{align*}
  \Qtot(\bm{\tau},\bm{\bar{u}}) = \sum_{i=1}^N Q_i(\tau_i,\bar{u}_i) \geq \sum_{i=1}^N Q_i(\tau_i,u_i) 
  \geq \Qtot(\bm{\tau},\bm{u}).
\end{align*}
The first equality comes from \eqref{eq:regular_cond4}, and the last
inequality comes from \eqref{eq:regular_cond5}. We note that
$\arg\max_{\bm{u}} \Qtot(\bm{\tau},\bm{u}) = [\arg\max_{u_i}
Q_i(\tau_i, u_i)]$, so this completes the proof.
\end{proof}

{\bf Proposition 1a.}
  {\em If $\Qtot(\bm{\tau},\bm{u})$ is factorized by $[Q_i(\tau_i,u_i)]$,
  then there exists an affine transformation $\phi(\bm{Q})$, such that
  $\Qtot(\bm{\tau},\bm{u})$ is factorized by
  $[\phi(Q_i(\tau_i,u_i))]$ and condition \eqref{eq:regular_cond3} holds by
  replacing $[Q_i(\tau_i,u_i)]$ with $[\phi(Q_i(\tau_i,u_i))].$}

\begin{proof}
  From Theorem 1a, if $\Qtot(\bm{\tau},\bm{u})$ is
  factorizable, then there exist $[Q_i]$ satisfying both {\bf IGM} and
  \eqref{eq:regular_cond3}. Now, we define an additive transformation
  $\phi'_i(Q_i(\tau_i,u_i)) = Q_i(\tau_i,u_i) + {{1}\over{N}}
  \max_{\bm{u}}\Qtot(\bm{\tau},\bm{u}) - Q_i(\tau_i, \bar{u}_i)$ for a
  given $\tau_i$, which is uniquely defined for fully observable
  cases. [$\phi'_i(Q_i(\tau_i,u_i))$] also satisfy {\bf IGM}, and the  left-hand side of \eqref{eq:regular_cond3} can be rewritten as:

\begin{align*}
    \sum_{i=1}^N Q_i(\tau_i,u_i) - \Qtot(\bm{\tau},\bm{u}) - \sum_{i=1}^N Q_i(\tau_i,\bar{u}_i) + \max_{\bm{u}}\Qtot(\bm{\tau},\bm{u}) = \sum_{i=1}^N \phi'_i (Q_i(\tau_i,u_i)) - \Qtot(\bm{\tau},\bm{u})
\end{align*}

So there exist individual action-value functions $[\phi'_i(Q_i'(\tau_i,u_i))]$ that satisfy both {\bf IGM} and \eqref{eq:regular_cond3}, where $\Vtot(\bm{\tau})$ is redefined as 0. This completes the proof of the necessity.

\end{proof}

\subsection{Proof of Theorem~\ref{thm:min}}

{\bf Theorem~\ref{thm:min}.}  The statement presented in {\rm Theorem~\ref{thm:ifonlyif}} and the necessary condition of {\rm Theorem~\ref{thm:ifonlyif}} holds by replacing
\eqref{eq:regular_cond2} with the following \eqref{eq:min2}: if
$\bm{u} \neq \bar{\bm{u}}$,
\begin{align}
  \min_{u_i \in \set{U}} \Big [\Qtot'(\bm{\tau},u_i,\bm{u}_{-i}) 
  - \Qtot(\bm{\tau},u_i,\bm{u}_{-i}) + \Vtot(\bm{\tau}) \Big] = 0,
  \quad \forall i =1, \ldots, N, \tag{\ref{eq:min2}}
\end{align}
where $\bm{u}_{-i} = (u_1, \ldots, u_{i-1}, u_{i+1}, \ldots, u_{N})$, {\em i.e.}, the action vector except for $i$'s action.

\begin{proof}
  ($\Rightarrow$)
  Recall that condition \eqref{eq:min2} is stronger than \eqref{eq:regular_cond2}, which is itself sufficient for Theorem~\ref{thm:ifonlyif}. Therefore, by transitivity, condition \eqref{eq:min2} is sufficient for Theorem~\ref{thm:min}. Following paragraphs focus on the other direction, {\em i.e.}, how condition \eqref{eq:min2} is {\em necessary} for Theorem~\ref{thm:min}.

  ($\Leftarrow$)
  We prove that, if there exist individual action-value functions
  satisfying condition \eqref{eq:regular_cond0}, then there exists an
  individual action-value function $Q_i'$ that satisfies
  \eqref{eq:min2}. In order to show the existence of such $Q_i'$, we
  propose a way to construct $Q_i'$.

  We first consider the case with $N=2$ and then generalize the
  result for any $N$. The condition \eqref{eq:min2} for $N=2$ is
  denoted as:
  \begin{align*}
    \min_{u_i \in \set{U}} \Big [Q_1(\tau_1,u_1) + Q_2(\tau_2,u_2)
    - \Qtot(\bm{\tau},u_1,u_2) + \Vtot(\bm{\tau}) \Big] = 0.
  \end{align*}
  Since this way of constructing  $Q_i'$ is symmetric for all $i$, we present its construction only for $u_1$ without loss of generality.
  For $Q_1$ and $Q_2$ satisfying
  \eqref{eq:regular_cond0}, if
  $\beta := \min_{u_1 \in \set{U}} \Big [Q_1(\tau_1,u_1) + Q_2(\tau_2,u_2) -
  \Qtot(\bm{\tau},u_1,u_2) + \Vtot(\bm{\tau}) \Big] > 0$ for
  given $\bm{\tau}$ and $u_2$, then $u_2 \neq \bar{u}_2$. This is because
  $Q_1(\tau_1,\bar{u}_1) + Q_2(\tau_2,\bar{u}_2) -
  \Qtot(\bm{\tau},\bar{u}_1,\bar{u}_2) + \Vtot(\bm{\tau}) = 0$ by
  condition \eqref{eq:regular_cond1}. Now, we replace $Q_2(\tau_2,u_2)$
  with $Q'_2(\tau_2,u_2) = Q_2(\tau_2,u_2)- \beta$. Since
  $Q_2(\tau_2,\bar{u}_2) > Q_2(\tau_2,u_2) > Q_2(\tau_2,u_2) - \beta$,
  it does not change the optimal action and other conditions. Then,
  \eqref{eq:min2} is satisfied for given $\bm{\tau}$ and $u_2$.
  By repeating this replacement process, we can construct
  $Q_i'$ that satisfies condition $\eqref{eq:min2}$.

  More generally, when $N \neq 2$, if
  $\min_{u_i \in \set{U}} \Big [\Qtot'(\bm{\tau},u_i,\bm{u}_{-i}) -
  \Qtot(\bm{\tau},u_i,\bm{u}_{-i}) + \Vtot(\bm{\tau}) \Big] = \beta > 0$ for given
  $\bm{\tau}$ and $\bm{u}_{-i}$, then there exists some $j \neq i$ that satisfies
  $u_j \neq \bar{u}_j$. Therefore, by repeating the same process as
  when $N=2$ through $j$, we can construct $Q_i'$ for all $i$, and this
  confirms that individual action-value functions satisfying condition \eqref{eq:min2} exist. This completes the proof.

\end{proof}

\subsubsection{Process of constructing $Q_i'$ in the matrix game using Theorem~\ref{thm:min}}
We now present the process of how we have demonstrated through the example in Section~\ref{sec:matgame}. In the original matrix shown in Tables~\ref{table:matrix-F1} and \ref{table:matrix-F4}, the second row does not satisfy the condition \eqref{eq:min2}, and $\beta = 0.23$ for $u_1=B$. Then, we replace $Q_1(B)$ as shown in Table~\ref{table:matrix-F5}. Table~\ref{table:matrix-F2} shows that its third row does not satisfy the condition \eqref{eq:min2}. Finally, we replace $Q_1(C)$ as shown in Table~\ref{table:matrix-F6}. Then, the resulting Table~\ref{table:matrix-F3} satisfies the condition \eqref{eq:min2}.

\begin{table}[h]
\scriptsize
    \setlength{\tabcolsep}{-1pt}
    \begin{minipage}{.33\columnwidth}
      \centering
        \begin{tabular}{|P{1.2cm}||P{1.2cm}|P{1.2cm}|P{1.2cm}|}
        \hline
        \backslashbox{$u_1$}{$u_2$}& A & B & C\\ \hline \hline
        A & 0.00  & 18.14 & 18.14 \\ \hline
        B & \textbf{14.11} & \textbf{0.23} & \textbf{0.23} \\ \hline
        C & 13.93 & 0.05 & 0.05 \\ \hline
        \end{tabular}
        \subcaption{$\Qtot' - \Qtot$}
        \label{table:matrix-F1}
    \end{minipage}
    \begin{minipage}{.33\columnwidth}
      \centering
        \begin{tabular}{|P{1.2cm}||P{1.2cm}|P{1.2cm}|P{1.2cm}|}
        \hline
        \backslashbox{$u_1$}{$u_2$}& A & B & C\\ \hline \hline
        A & 0.00  & 18.14 & 18.14 \\ \hline
        B & 13.88 & 0.00 & 0.00 \\ \hline
        C & \textbf{13.93} & \textbf{0.05} & \textbf{0.05} \\ \hline
        \end{tabular}
        \subcaption{$\Qtot' - \Qtot$ after one replacement}
        \label{table:matrix-F2}
    \end{minipage}
    \begin{minipage}{.33\columnwidth}
      \centering
        \begin{tabular}{|P{1.2cm}||P{1.2cm}|P{1.2cm}|P{1.2cm}|}
        \hline
        \backslashbox{$u_1$}{$u_2$}& A & B & C\\ \hline \hline
        A & 0.00  & 18.14 & 18.14 \\ \hline
        B & 13.88 & 0.00 & 0.00 \\ \hline
        C & 13.88 & 0.00 & 0.00 \\ \hline
        \end{tabular}
        \subcaption{$\Qtot' - \Qtot$ after two replacements}
        \label{table:matrix-F3}
    \end{minipage}
    \begin{minipage}{.33\columnwidth}
      \centering
        \begin{tabular}{|P{1.2cm}||P{1.2cm}|P{1.2cm}|P{1.2cm}|}
        \hline
        \backslashbox{$Q_1$}{$Q_2$} & 4.16(A) & 2.29(B) & 2.29(C)\\ \hline \hline
        3.84(A) & 8.00  & 6.13 & 6.12 \\ \hline
        -2.06(B) & 2.10 & 0.23 & 0.23 \\ \hline
        -2.25(C) & 1.92 & 0.04 & 0.04 \\ \hline
        \end{tabular}
        \subcaption{$Q_1, Q_2,\Qtot'$}
        \label{table:matrix-F4}
    \end{minipage}
    \begin{minipage}{.33\columnwidth}
      \centering
        \begin{tabular}{|P{1.2cm}||P{1.2cm}|P{1.2cm}|P{1.2cm}|}
        \hline
        \backslashbox{$Q_1$}{$Q_2$} & 4.16(A) & 2.29(B) & 2.29(C)\\ \hline \hline
        3.84(A) & 8.00  & 6.13 & 6.12 \\ \hline
        -2.29(B) & 1.87 & 0.00 & 0.00 \\ \hline
        -2.25(C) & 1.92 & 0.04 & 0.04 \\ \hline
        \end{tabular}
        \subcaption{$Q_1, Q_2,\Qtot'$ after one replacement}
        \label{table:matrix-F5}
    \end{minipage}
    \begin{minipage}{.33\columnwidth}
      \centering
        \begin{tabular}{|P{1.2cm}||P{1.2cm}|P{1.2cm}|P{1.2cm}|}
        \hline
        \backslashbox{$Q_1$}{$Q_2$} & 4.16(A) & 2.29(B) & 2.29(C)\\ \hline \hline
        3.84(A) & 8.00  & 6.13 & 6.12 \\ \hline
        -2.29(B) & 1.87 & 0.00 & 0.00 \\ \hline
        -2.30(C) & 1.87 & -0.01 & -0.01 \\ \hline
        \end{tabular}
        \subcaption{$Q_1, Q_2,\Qtot'$ after two replacements}
        \label{table:matrix-F6}
    \end{minipage}
        \caption{The process of replacing $[Q_i]$ satisfying the condition \eqref{eq:regular_cond5} with $[Q_i]$ satisfying the condition \eqref{eq:min2}}
\vspace{-0.2cm}
\label{table:matrix2}
\end{table}


\section{Details of environments and implementation}

\subsection{Environment}

\paragraph{Matrix game}

In order to see the impact of $\algonameadv$, we train the agents in a single state matrix game where two agents each have 21 actions. Each agent $i$ takes action $u_i$, ranging over $\in \{0,...,20\}.$ The reward value $R$ for a joint action is given as follows:

\begin{align*}
f_1(u_1, u_2) &= 5 - \left(\frac{15 - u_1}{3}\right)^2 - \left(\frac{5 - u_2}{3}\right)^2 \\
f_2(u_1, u_2) &= 10 - \left(\frac{5 - u_1}{1}\right)^2 - \left(\frac{15 - u_2}{1}\right)^2 \\
R(u_1, u_2) &= \max(f_1(u_1,u_2), f_2(u_1,u_2)) 
\end{align*}

\begin{figure*}[!h]
  \centering
 \includegraphics[width=0.5\textwidth]{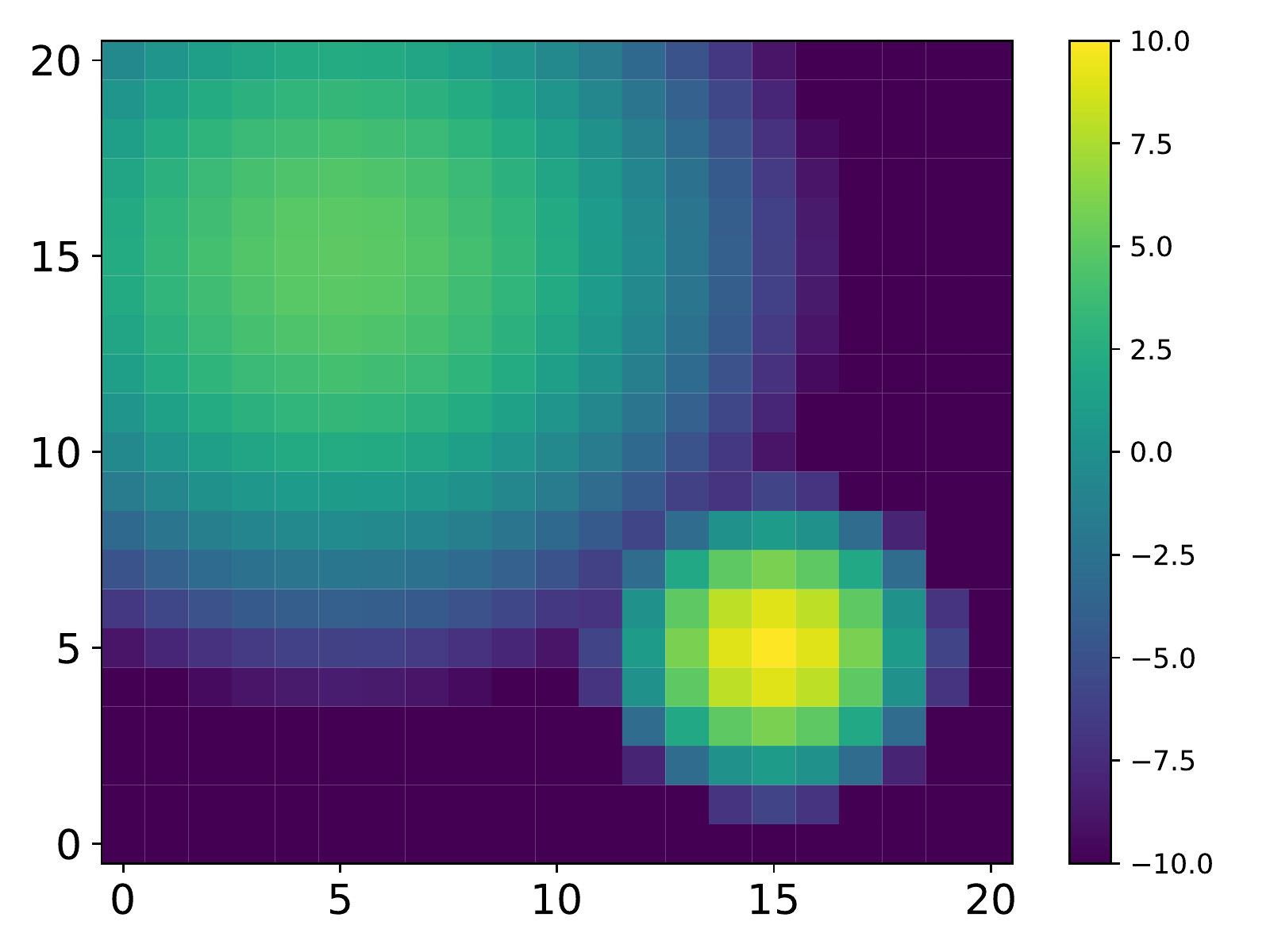}
\caption{$x$-axis and $y$-axis: agents 1 and 2's actions, respectively. Colored values represent the reward for selected actions.}
  \label{fig:Origin}
\end{figure*}

Figure~\ref{fig:Origin} shows reward for selected action. Colored values represent the values. In the simple matrix game, the reward function has a global maximum point at $(u_1, u_2) = (5, 15)$ and a local maximum point at $(u_1, u_2) = (15, 5).$

\myparagraph{Multi-domain Gaussian Squeeze}
We adopt and modify Gaussian Squeeze, where agent numbers ($N=10)$ and action spaces ($u_i \in \{0,1,...,9\})$ are much larger than a simple matrix game. In MGS, each of the ten agents has its own amount of unit-level resource $s_i \in [0,0.2]$ given by the environment {\em a priori}. This task covers a fully observable case where all agents can see the entire state. We assume that there exist $K$ domains, where the above resource allocation takes place. The joint action $\bm{u}$ determines the overall resource usage $x(\bm{u}) = \sum_i s_i \times u_i$. Reward is given as a function of resource usage as follows: $G(\bm{u}) = \sum_{k=1}^K x e^{-(x-\mu_k)^2 / {\sigma_k}^2}$. We test with two different settings: (i) $K=1, \mu_1 = 8, \sigma_1 = 1$ as shown in Figure~\ref{fig:GMSenv1}, and (ii) $K=2, \mu_1 = 8, \sigma_1 = 0.5, \mu_2 = 5, \sigma_2 = 1$ as shown in Figure~\ref{fig:GMSenv2}. In the second setting, there are two local maxima for resource $x$. The maximum on the left is relatively easy to find through exploration -- as manifested in the greater variance of the Gaussian distribution, but the maximum reward --- as represented by the lower peak --- is relatively low. On the other hand, the maximum on the right is more difficult to find through exploration, but it offers higher reward.

\begin{figure*}[!h]
\hspace*{\fill}
\begin{subfigure}[t]{.33\textwidth}
  \includegraphics[width=\linewidth]{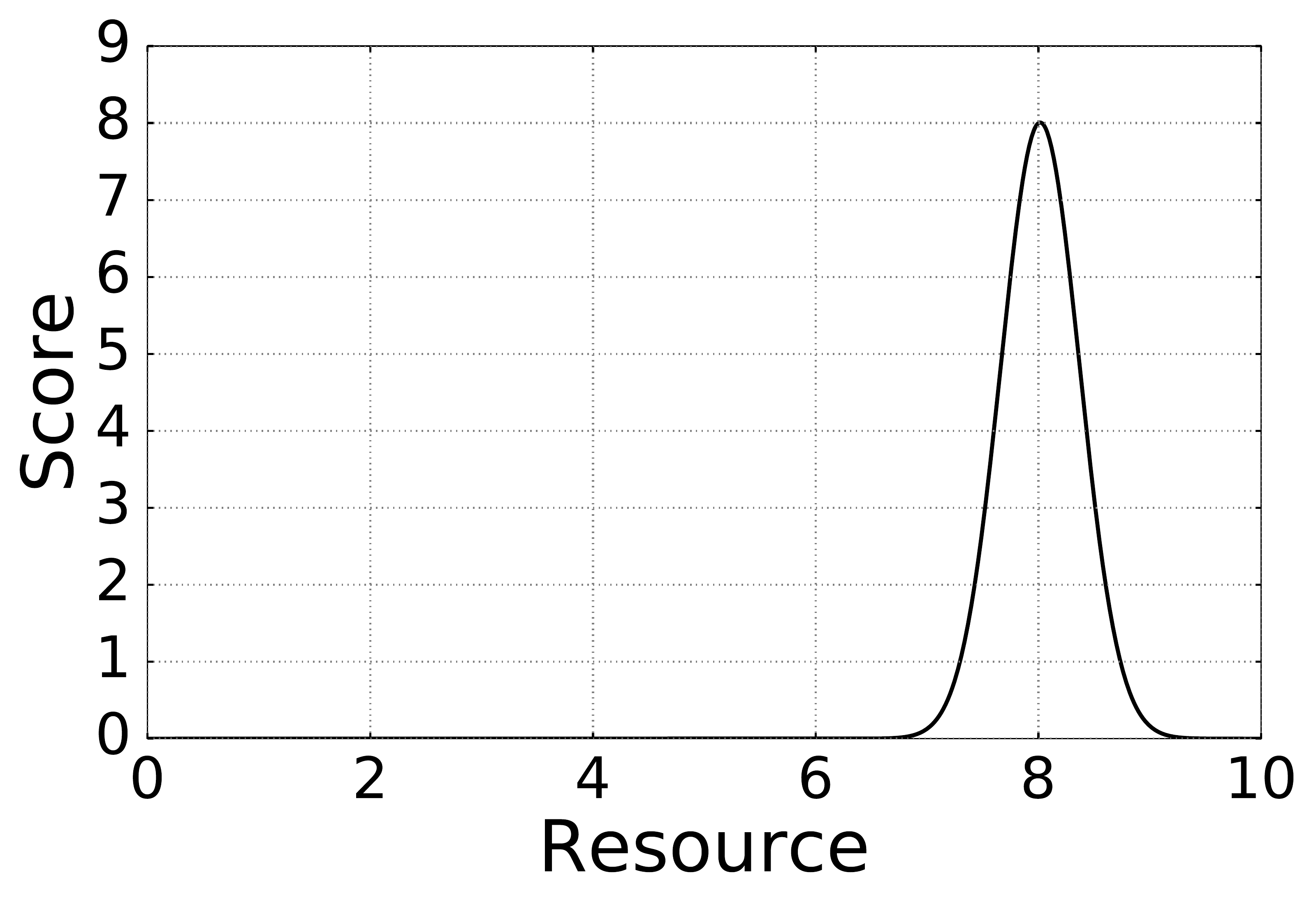}
  \caption{Gaussian Squeeze}
  \label{fig:GMSenv1}
\end{subfigure}\hspace*{\fill}
\begin{subfigure}[t]{.33\textwidth}
  \includegraphics[width=\linewidth]{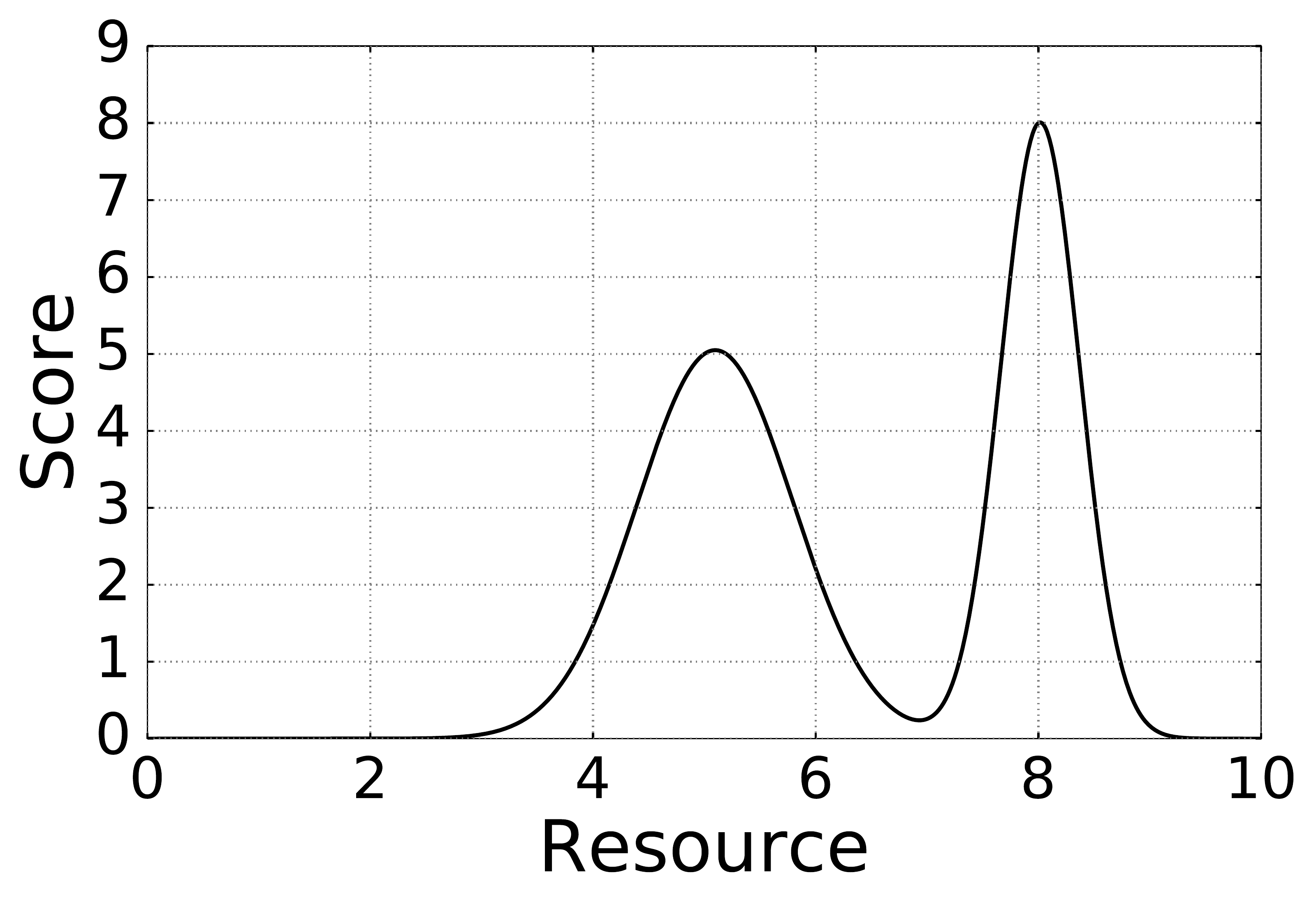}
  \caption{Multi-domain Gaussian Squeeze}
  \label{fig:GMSenv2}
\end{subfigure}\hspace*{\fill}
\caption{Gaussian Squeeze and Multi-domain Gaussian Squeeze}
\label{fig:GMSenv}
\end{figure*}

\paragraph{Modified predator-prey}

Predator-prey involves a grid world, in which multiple predators attempt to capture randomly moving prey. Agents have a $5 \times 5$ view and select one of five actions $\in \{$Left, Right, Up, Down, Stop$\}$ at each time step. Prey move according to selecting a uniformly random action at each time step. We define the ``catching'' of a prey as when the prey is within the cardinal direction of at least one predator. Each agent's observation includes its own coordinates, agent ID, and the coordinates of the prey relative to itself, if observed. The agents can separate roles even if the parameters of the neural networks are shared by agent ID. We test with two different grid worlds: (i) a $5 \times 5$ grid world with two predators and one prey, and (ii) a $7 \times 7$ grid world with four predators and two prey. We modify the general predator-prey, such that a positive reward is given only if multiple predators catch a prey simultaneously, requiring a higher degree of cooperation. The predators get a team reward of $1$ if two or more catch a prey at the same time, but they are given negative reward $-P$, if only one predator catches the prey as shown in Figure~\ref{fig:PPenv}. We experimented with three varying $P$ vales, where $P = 0.5, 1.0, 1.5$. The terminating condition of this task is when a prey is caught with more than one predator. The prey that has been caught is regenerated at random positions whenever the task terminates, and the game proceeds over fixed 100 steps.

\begin{figure*}[!h]
\hspace*{2.3cm}
\begin{subfigure}[t]{.85\textwidth}
  \includegraphics[width=\linewidth]{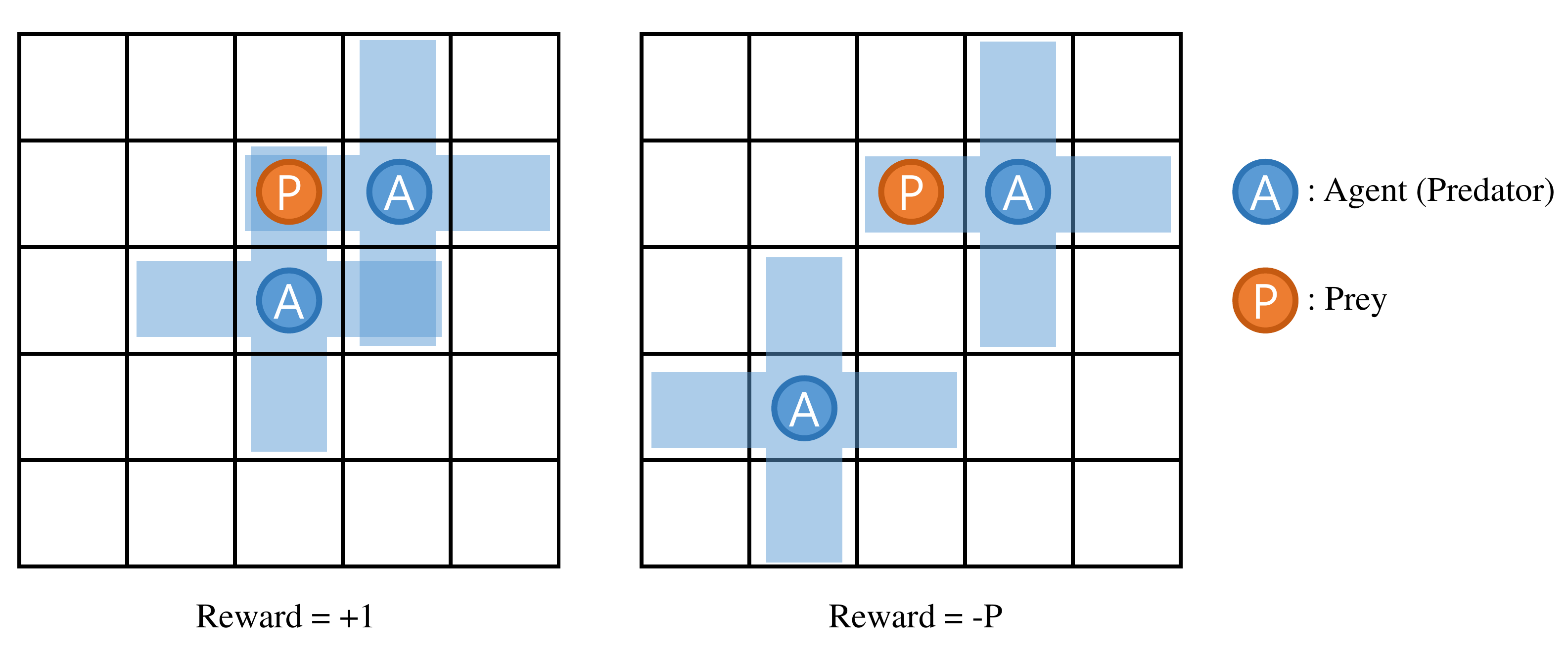}
  \label{fig:PPenv1}
\end{subfigure}
\vspace{-0.5cm}
\caption{Predator-prey environment}
\label{fig:PPenv}
\end{figure*}

\subsection{Experiment details}

\paragraph{Matrix game}
Table~\ref{table:hyperparameter1} shows the values of the hyperparameters for the training in the matrix game environment. Individual action-value networks, which are common in all VDN, QMIX, and $\algoname$, each consist of two hidden layers. In addition to the individual Q-networks, QMIX incorporates a monotone network with one hidden layer, and the weights and biases of this network are produced by separate hypernetworks. $\algoname$ has an additional joint action-value network with two hidden layers. All hidden layer widths are 32, and all activation functions are ReLU. All neural networks are trained using the Adam optimizer. We use $\epsilon$-greedy action selection with $\epsilon = 1$ independently for exploration.

\begin{table}[h!]
\centering
\begin{tabular}{@{}lllll@{}}
\toprule
\textbf{Hyperparameter}       & \textbf{Value} & \multicolumn{3}{l}{\textbf{Description}}                                \\ \midrule
training step                 & 20000         & \multicolumn{3}{l}{Maximum time steps until the end of training}        \\                      
learning rate       & 0.0005        & \multicolumn{3}{l}{Learning rate used by Adam optimizer}  \\
replay buffer size            & 20000           & \multicolumn{3}{l}{Maximum number of samples to store in memory} \\
minibatch size            & 32           & \multicolumn{3}{l}{Number of samples to use for each update} \\
$\lambda_{opt},$ $\lambda_{nopt}$             & 1,1           & \multicolumn{3}{l}{Weight constants for loss functions $L_{opt}$, $L_{nopt}$ and $L_{nopt-min}$} \\
\bottomrule
\end{tabular}
\vspace{-0.2cm}
\caption{\small Hyperparameters for matrix game training}
\label{table:hyperparameter1}
\end{table}
\vspace{-0.4cm}

\myparagraph{Multi-domain Gaussian Squeeze}
Table~\ref{table:hyperparameter2} shows the values of the hyperparameters for the training in the MGS environment. Each individual action-value network consists of three hidden layers. In addition to the individual Q-networks, QMIX incorporates a monotone network with one hidden layer, and the weights and biases of this network are produced by separate hypernetworks. $\algoname$ has an additional joint action-value network with two hidden layers. All hidden layer widths are 64, and all activation functions are ReLU. All neural networks are trained using the Adam optimizer. We use $\epsilon$-greedy action selection with decreasing $\epsilon$ from 1 to 0.1 for exploration.

\begin{table}[h!]
\centering

\begin{tabular}{@{}lllll@{}}
\toprule
\textbf{Hyperparameter}       & \textbf{Value} & \multicolumn{3}{l}{\textbf{Description}}                                \\ \midrule
training step                 & 1000000         & \multicolumn{3}{l}{Maximum time steps until the end of training}        \\                       
learning rate       & 0.0005        & \multicolumn{3}{l}{Learning rate used by Adam optimizer}  \\
replay buffer size            & 200000           & \multicolumn{3}{l}{Maximum number of samples to store in memory} \\
final exploration step            & 500000           & \multicolumn{3}{l}{Number of steps over which $\epsilon$ is annealed to the final value} \\
minibatch size            & 32           & \multicolumn{3}{l}{Number of samples to use for each update} \\
$\lambda_{opt},$ $\lambda_{nopt}$             & 1,1           & \multicolumn{3}{l}{Weight constants for loss functions $L_{opt}$, $L_{nopt}$ and $L_{nopt-min}$} \\
\bottomrule
\end{tabular}
\vspace{-0.2cm}
\caption{\small Hyperparameters for Multi-domain Gaussian Squeeze}
\label{table:hyperparameter2}
\end{table}

\vspace{-0.2cm}

\myparagraph{Modified predator-prey}
Table~\ref{table:hyperparameter3} shows the values of the hyperparameters for the training in the modified predator-prey environment. Each individual action-value network consists of three hidden layers. In addition to the individual Q-networks, QMIX incorporates a monotone network with one hidden layer, and the weights and biases of this network are produced by separate hypernetworks. $\algoname$ has additional joint action-value network with two hidden layers. All hidden layer widths are 64, and all activation functions are ReLU. All neural networks are trained using the Adam optimizer. We use $\epsilon$-greedy action selection with decreasing $\epsilon$ from 1 to 0.1 for exploration. Since history is not very important in this task, our experiments use feed-forward policies, but our method is also applicable with recurrent policies. 

\begin{table}[h!]
\centering
\begin{tabular}{@{}lllll@{}}
\toprule
\textbf{Hyperparameter}       & \textbf{Value} & \multicolumn{3}{l}{\textbf{Description}}                                \\ \midrule
training step                 & 10000000         & \multicolumn{3}{l}{Maximum time steps until the end of training}        \\
discount factor               & 0.99            & \multicolumn{3}{l}{Importance of future rewards}                        \\
learning rate       & 0.0005        & \multicolumn{3}{l}{Learning rate used by Adam optimizer}  \\
target update period            & 10000           & \multicolumn{3}{l}{Target network update period to track learned network} \\
replay buffer size            & 1000000           & \multicolumn{3}{l}{Maximum number of samples to store in memory} \\
final exploration step            & 3000000           & \multicolumn{3}{l}{Number of steps over which $\epsilon$ is annealed to the final value} \\
minibatch size            & 32           & \multicolumn{3}{l}{Number of samples to use for each update} \\
$\lambda_{opt},$ $\lambda_{nopt}$             & 1,1           & \multicolumn{3}{l}{Weight constants for loss functions $L_{opt}$, $L_{nopt}$ and $L_{nopt-min}$} \\
\bottomrule
\end{tabular}
\vspace{-0.2cm}
\caption{\small Hyperparameters for predator-prey training}
\label{table:hyperparameter3}
\end{table}

\vspace{-0.3cm}




\section{Additional results for matrix games}

Table~\ref{table:matrix-random} shows the comparison between the final performance levels of VDN, QMIX, and $\algonamebasic$ for 310 $3 \times 3$ random matrix games, where each value of the payoff matrix is given as a random parameter from 0 to 1. Experimental settings are the same as in the previous matrix game. Since matrix games always satisfy {\bf IGM} conditions, $\algonamebasic$ always trains the optimal action for all 310 cases. On the other hand, VDN and QMIX were shown to be unable to learn an optimal policy for more than half of non-monotonic matrix games.

We briefly analyze the nature of the structural constraints assumed by VDN and QMIX, namely, additivity and monotonicity of the joint action-value functions. There have been only 19 cases in which the results of VDN and QMIX differ from each other. Interestingly, for a total of five cases, the performance of VDN is better than QMIX. QMIX was shown to outperform VDN in more cases (14) than the converse case (5). This supports the idea that the additivity assumption imposed by VDN on the joint action-value functions is indeed stronger than the monotonicity assumption imposed by QMIX.

\begin{table}[h!]
\centering
\begin{tabular}{llll}
\toprule
\textbf{$\algoname$=VDN=QMIX} & \textbf{$\algoname$\textgreater{}VDN=QMIX} & \textbf{VDN\textgreater{}QMIX} & \textbf{QMIX\textgreater{}VDN} \\ \hline
\multicolumn{1}{c}{114} & \multicolumn{1}{c}{177}   & \multicolumn{1}{c}{5} & \multicolumn{1}{c}{14} \\
\bottomrule
\end{tabular}
\caption{\small Final performance comparison with 310 random matrices}
\label{table:matrix-random}
\end{table}

Figures~\ref{fig:PCresult4}-\ref{fig:PCresult5} show the joint action-value function of VDN and QMIX, and Figures~\ref{fig:PCresult2}-\ref{fig:PCresult3} show the transformed joint action-value function of $\algonamebasic$ and $\algonameadv$ for a matrix game where two agents each have 20 actions. In the result, VDN and QMIX can not recover joint action-value, and these algorithms learn sub-optimal policy $u_1, u_2 = (15, 5).$ In the other hand, the result shows that $\algonamebasic$ and $\algonameadv$ successfully learn the optimal action, but $\algonameadv$ has the ability to more accurately distinguish action from non-optimal action as shown in Figure~\ref{fig:PCresult3}.

\begin{figure}[!h]
\hspace*{\fill}
\begin{subfigure}[t]{.12\columnwidth}
  \includegraphics[width=\linewidth]{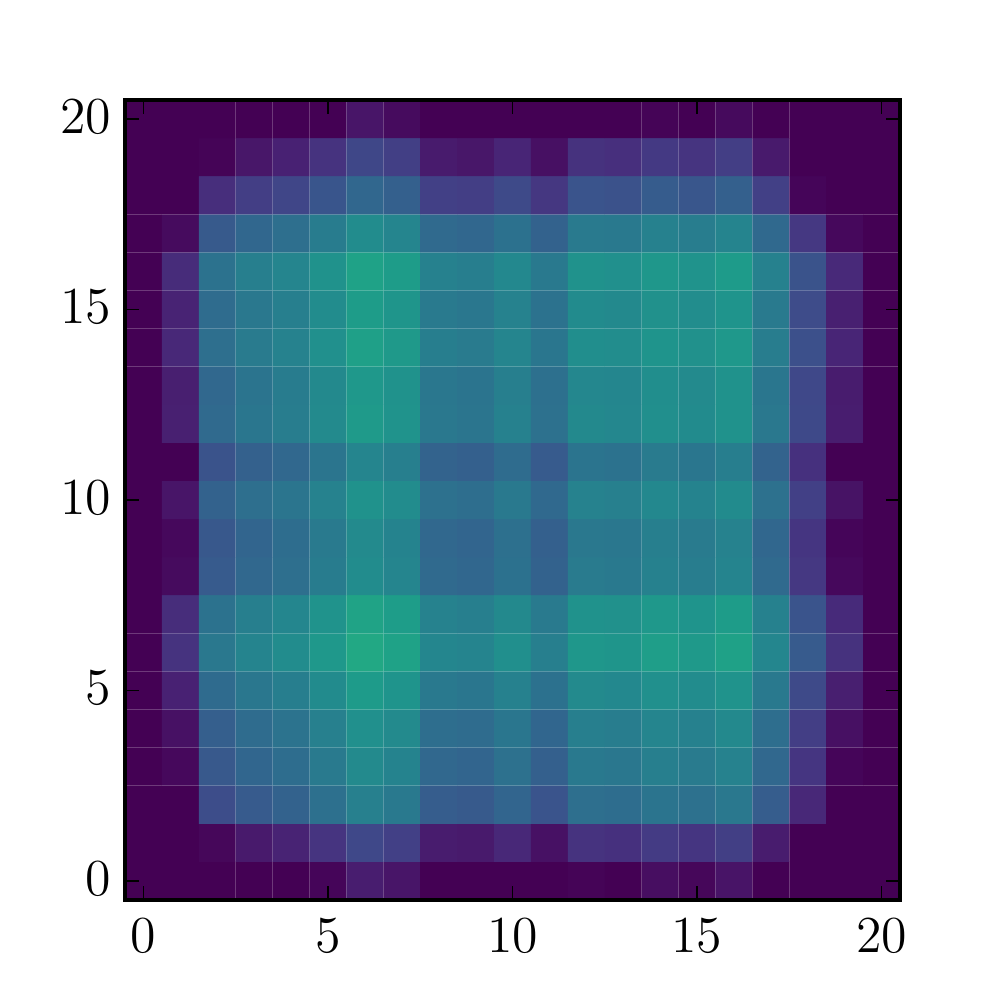}
  \caption{\scriptsize 1000 step}
\end{subfigure}\hspace*{\fill}
\begin{subfigure}[t]{.12\columnwidth}
  \includegraphics[width=\linewidth]{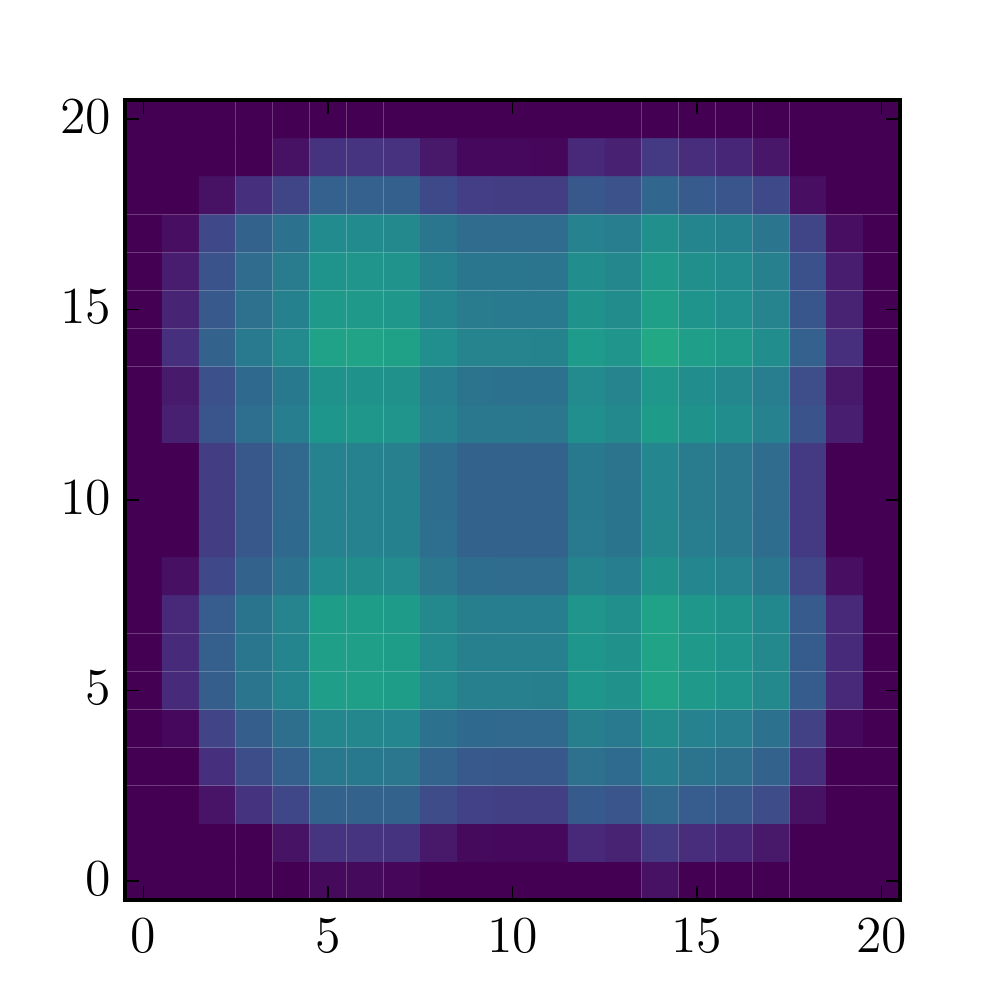}
  \caption{\scriptsize 2000 step}
\end{subfigure}\hspace*{\fill}
\begin{subfigure}[t]{.12\columnwidth}
  \includegraphics[width=\linewidth]{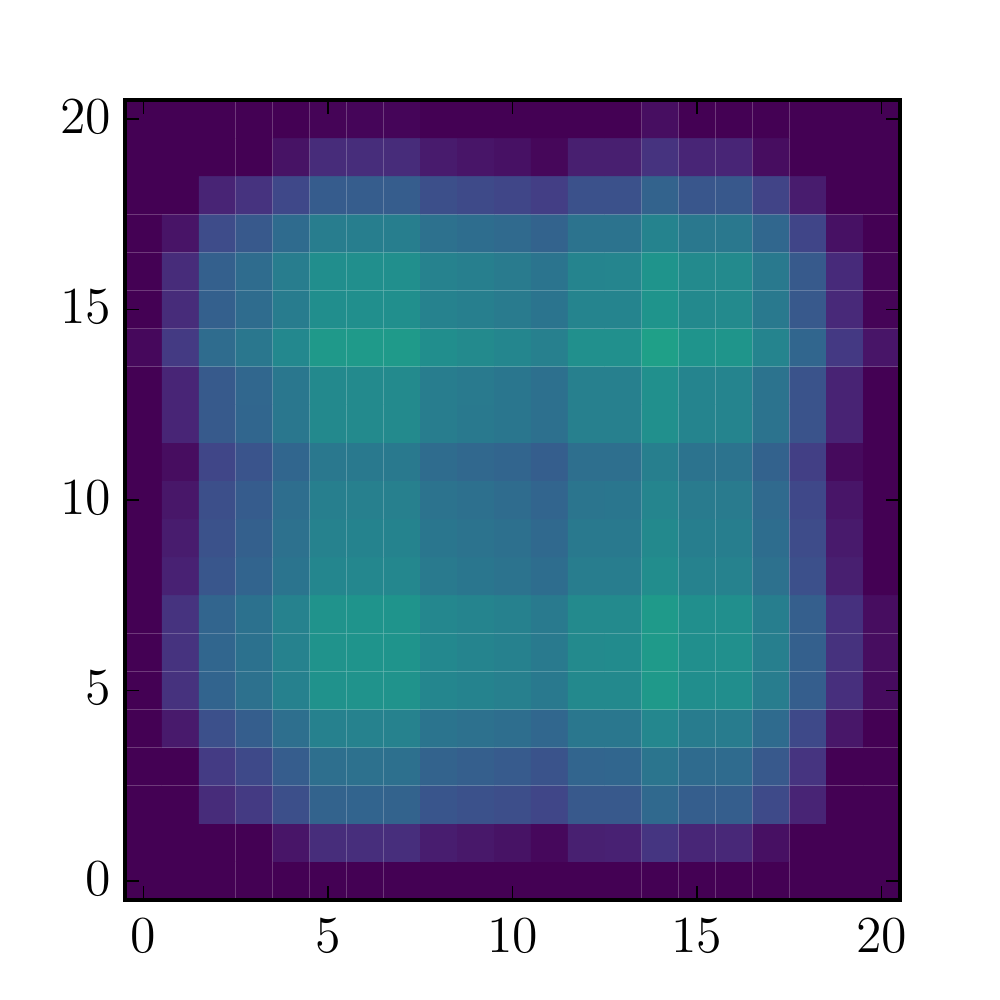}
  \caption{\scriptsize 3000 step}
\end{subfigure}\hspace*{\fill}
\begin{subfigure}[t]{.12\columnwidth}
  \includegraphics[width=\linewidth]{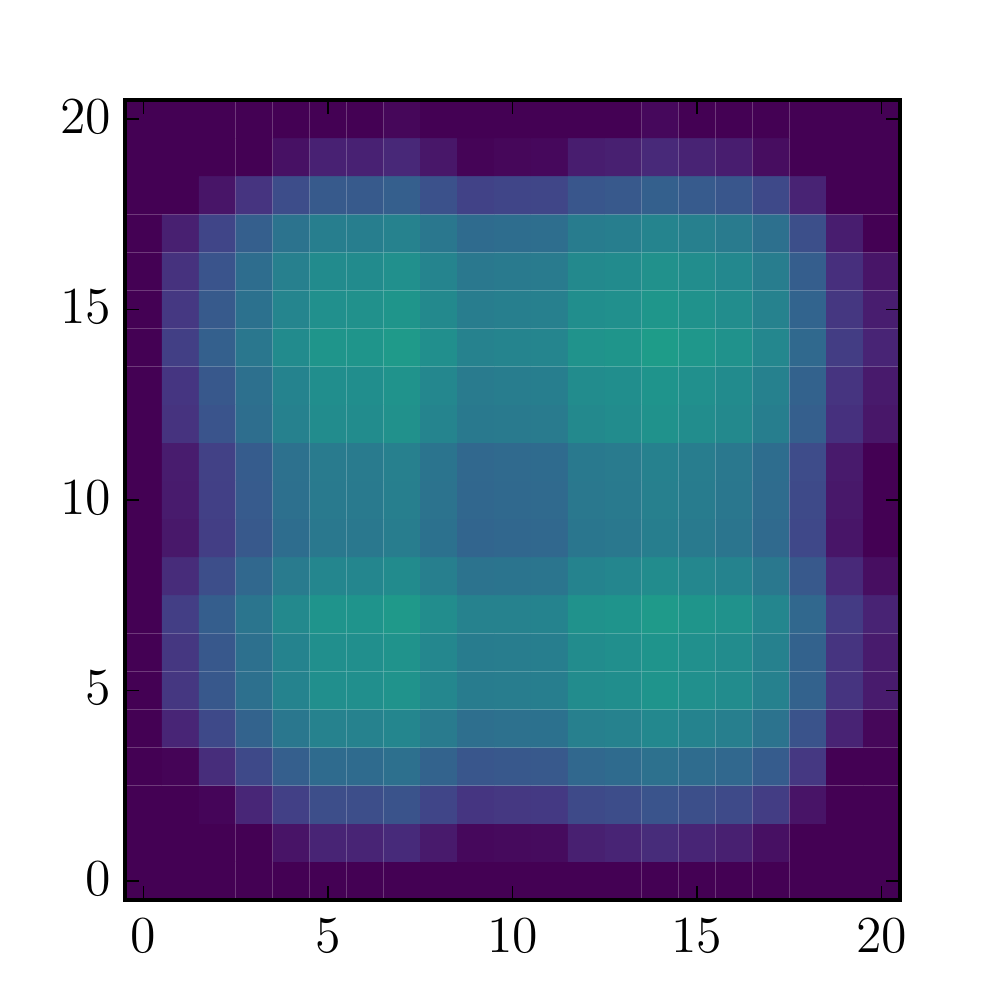}
  \caption{\scriptsize 4000 step}
\end{subfigure}\hspace*{\fill}
\hspace*{\fill}
\begin{subfigure}[t]{.12\columnwidth}
  \includegraphics[width=\linewidth]{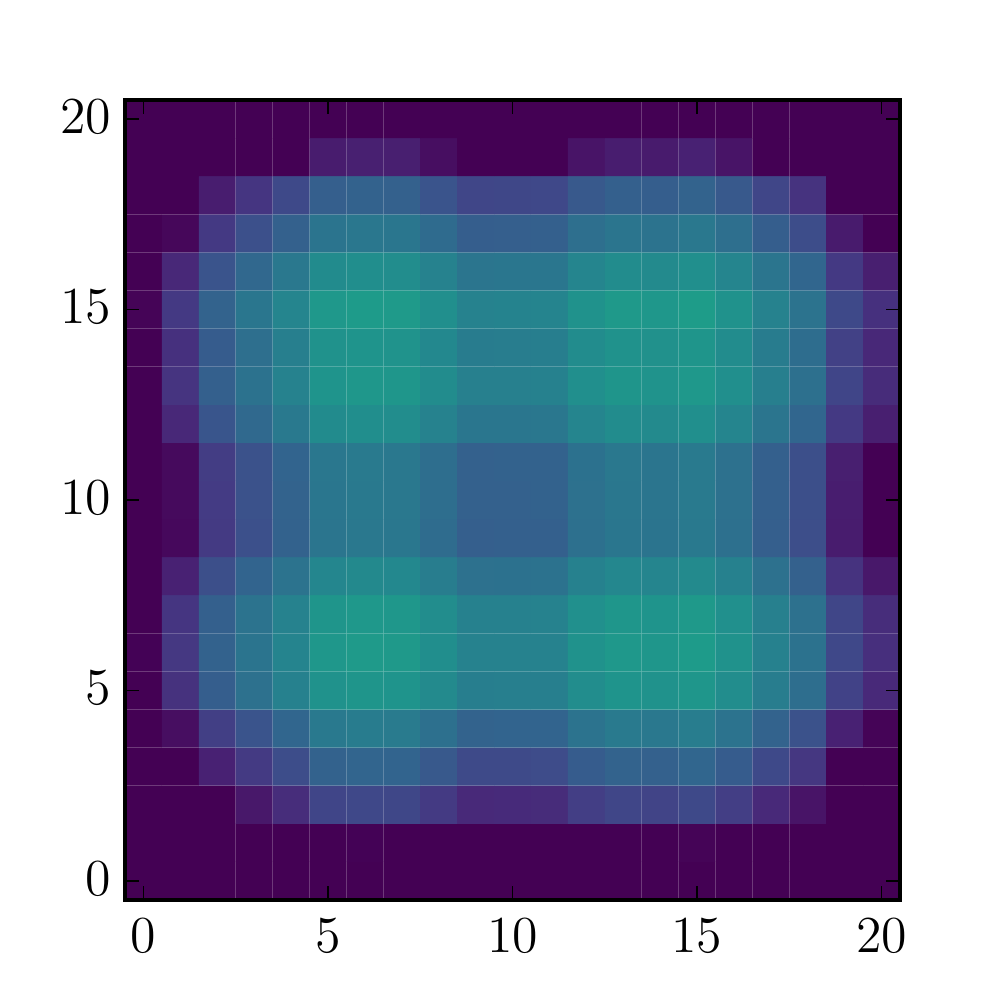}
  \caption{\scriptsize 5000 step}
\end{subfigure}\hspace*{\fill}
\begin{subfigure}[t]{.12\columnwidth}
  \includegraphics[width=\linewidth]{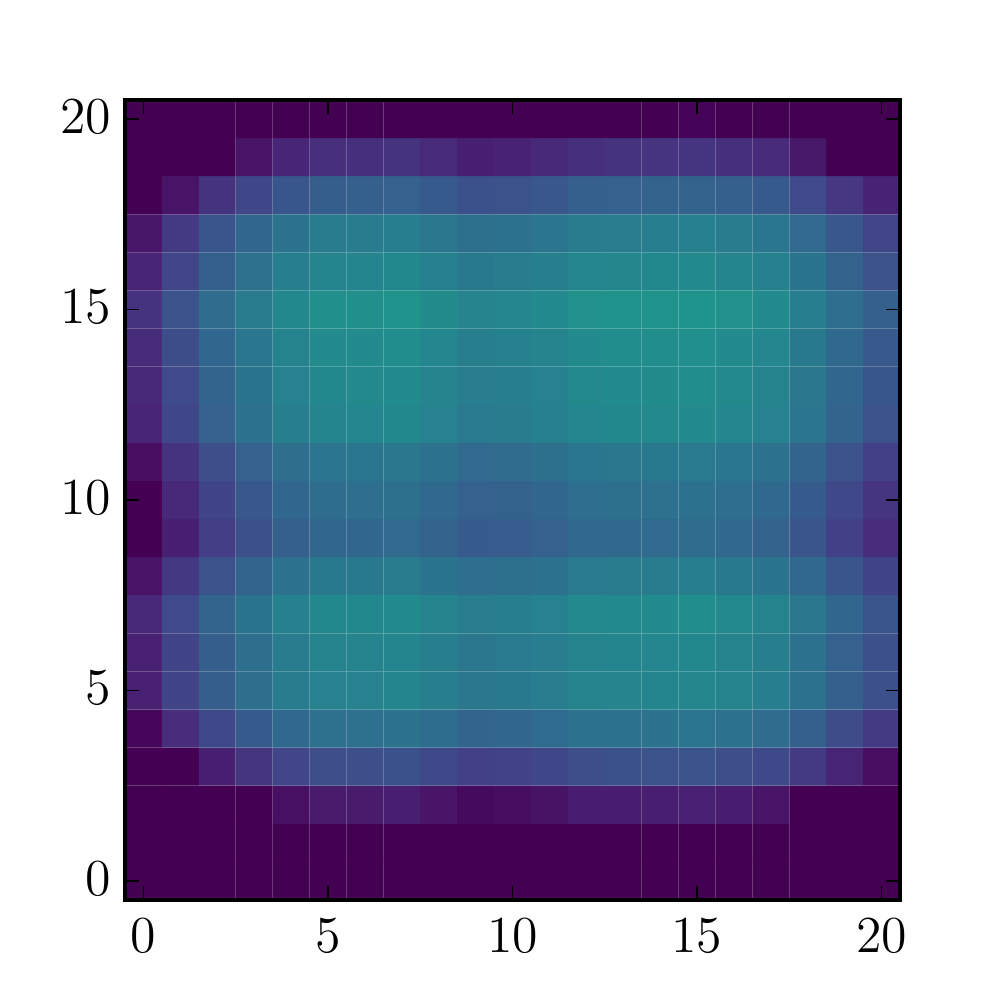}
  \caption{\scriptsize 6000 step}
\end{subfigure}\hspace*{\fill}
\begin{subfigure}[t]{.12\columnwidth}
  \includegraphics[width=\linewidth]{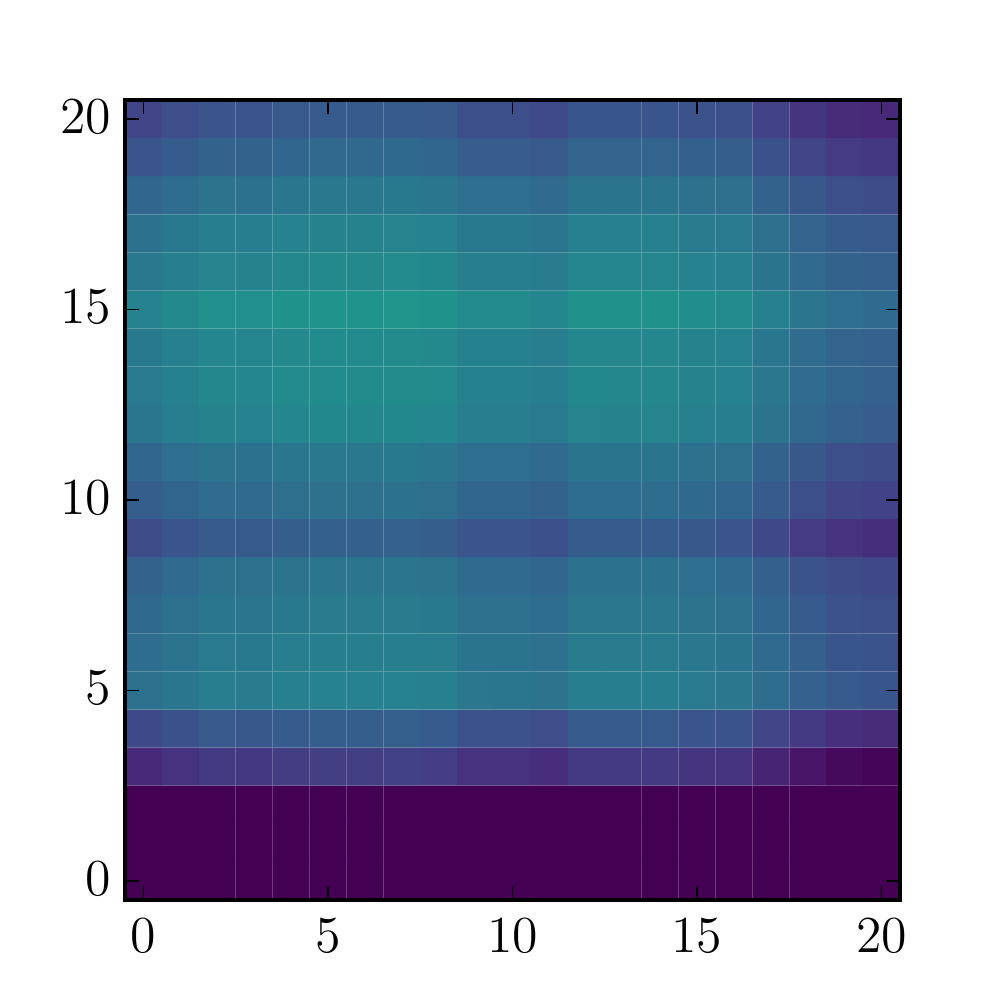}
  \caption{\scriptsize 7000 step}
\end{subfigure}\hspace*{\fill}
\begin{subfigure}[t]{.12\columnwidth}
  \includegraphics[width=\linewidth]{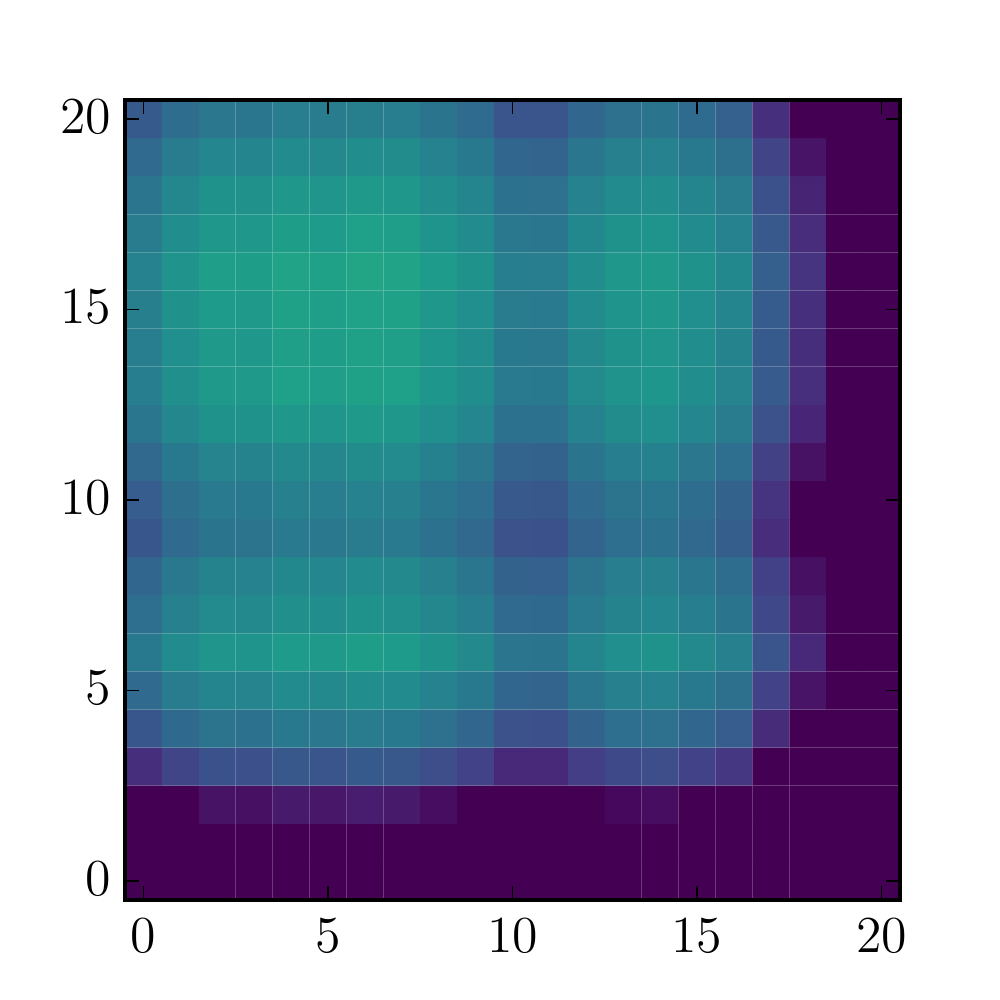}
  \caption{\scriptsize 8000 step}
\end{subfigure}\hspace*{\fill}
\vspace{-0.2cm}
\caption{$x$-axis and $y$-axis: agents 1 and 2's actions. Colored values represent the values of $\Qtot$ for VDN}
\label{fig:PCresult4}

\end{figure}

\vspace{-0.6cm}
\begin{figure}[!h]
\hspace*{\fill}
\begin{subfigure}[t]{.12\columnwidth}
  \includegraphics[width=\linewidth]{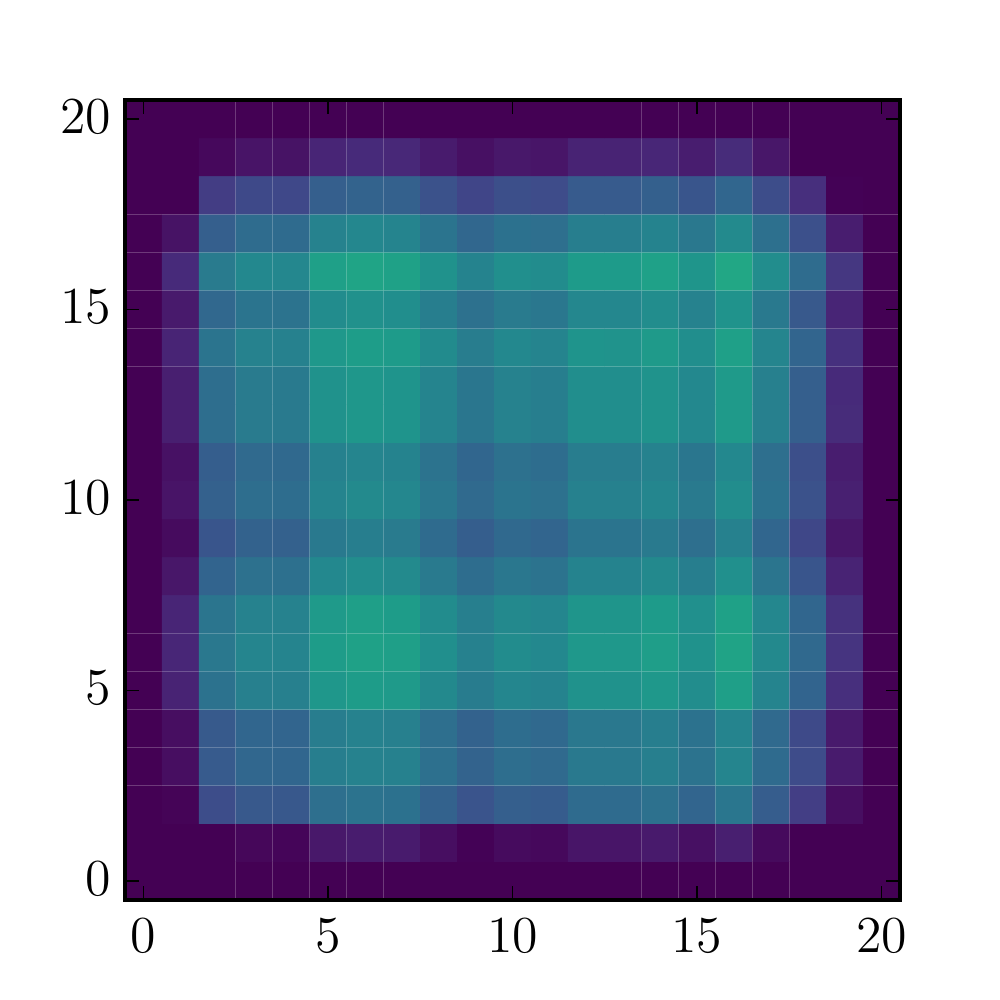}
  \caption{\scriptsize 1000 step}
\end{subfigure}\hspace*{\fill}
\begin{subfigure}[t]{.12\columnwidth}
  \includegraphics[width=\linewidth]{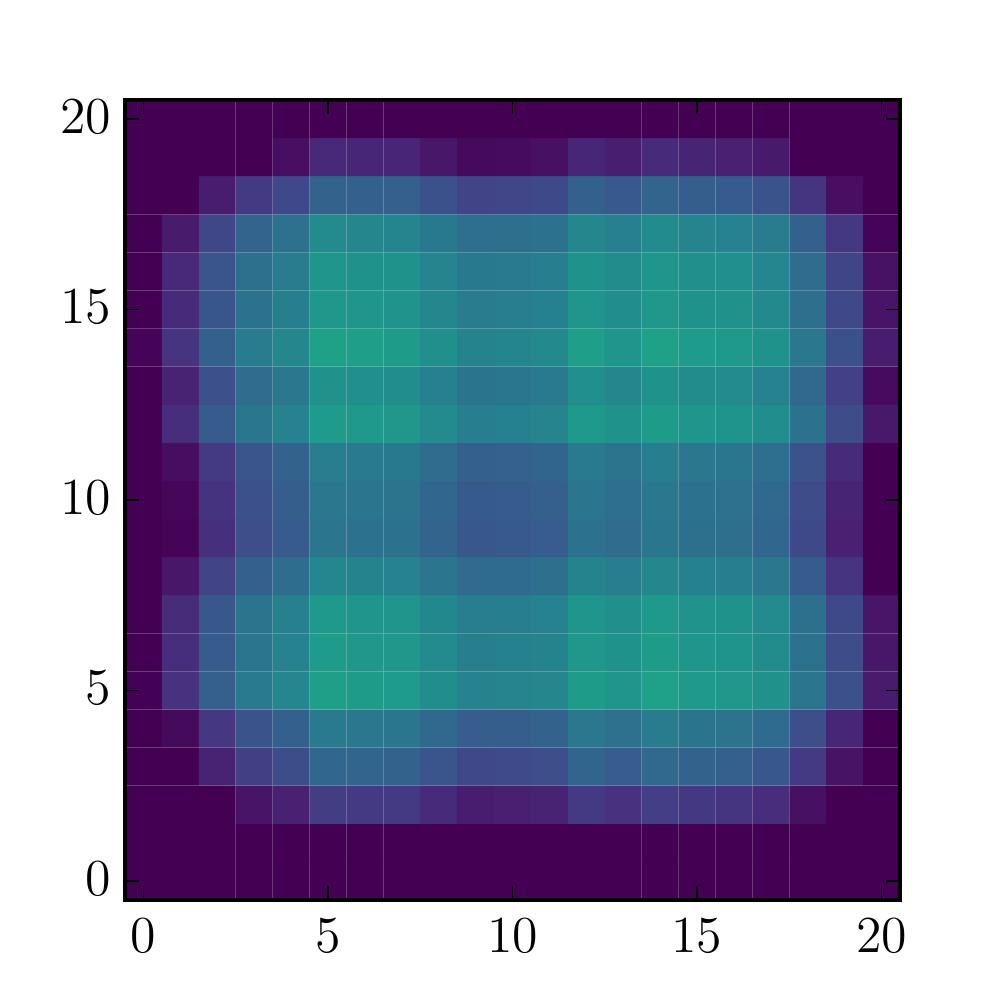}
  \caption{\scriptsize 2000 step}
\end{subfigure}\hspace*{\fill}
\begin{subfigure}[t]{.12\columnwidth}
  \includegraphics[width=\linewidth]{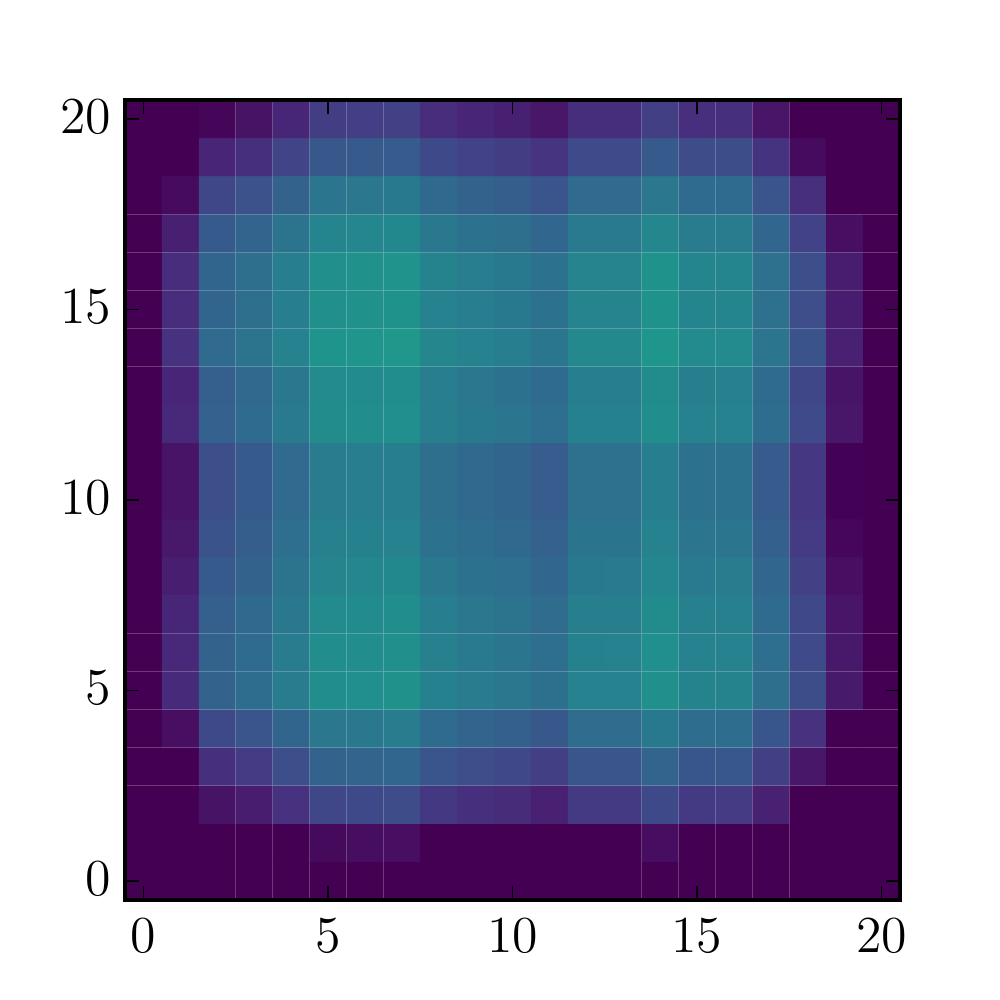}
  \caption{\scriptsize 3000 step}
\end{subfigure}\hspace*{\fill}
\begin{subfigure}[t]{.12\columnwidth}
  \includegraphics[width=\linewidth]{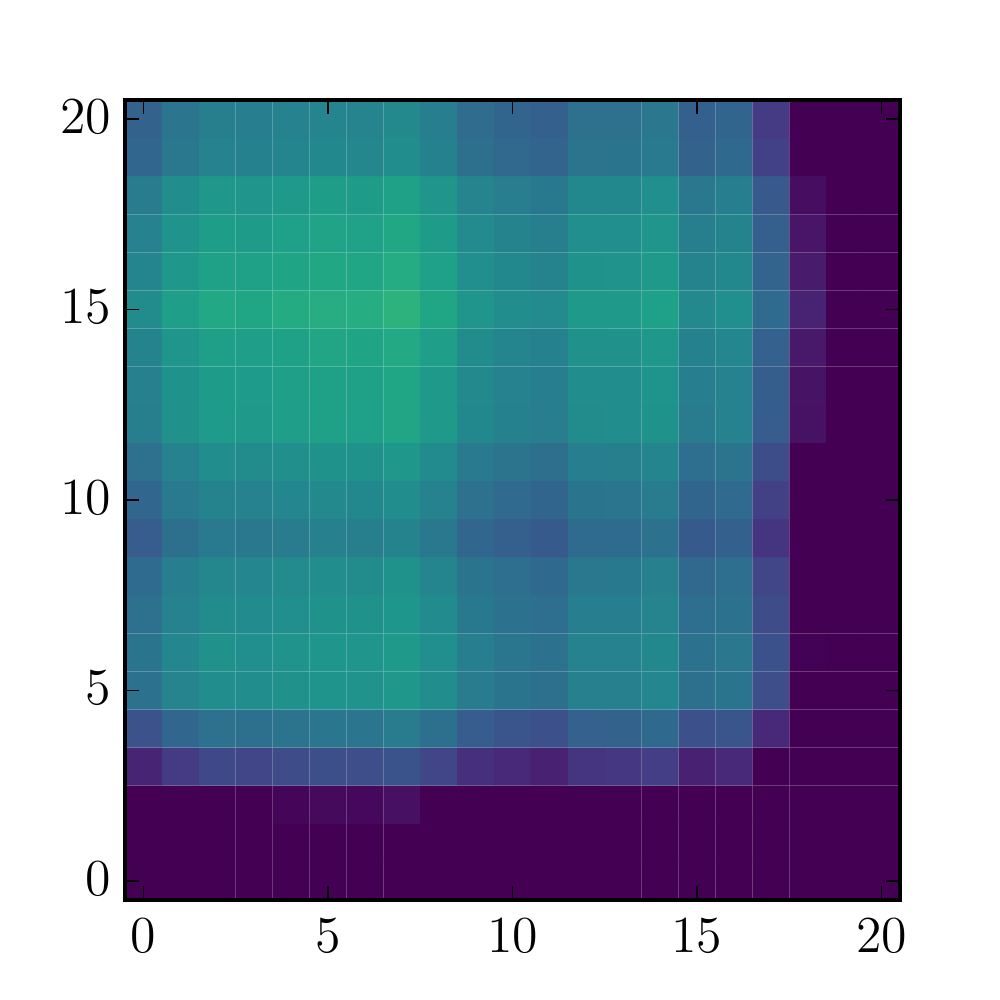}
  \caption{\scriptsize 4000 step}
\end{subfigure}\hspace*{\fill}
\hspace*{\fill}
\begin{subfigure}[t]{.12\columnwidth}
  \includegraphics[width=\linewidth]{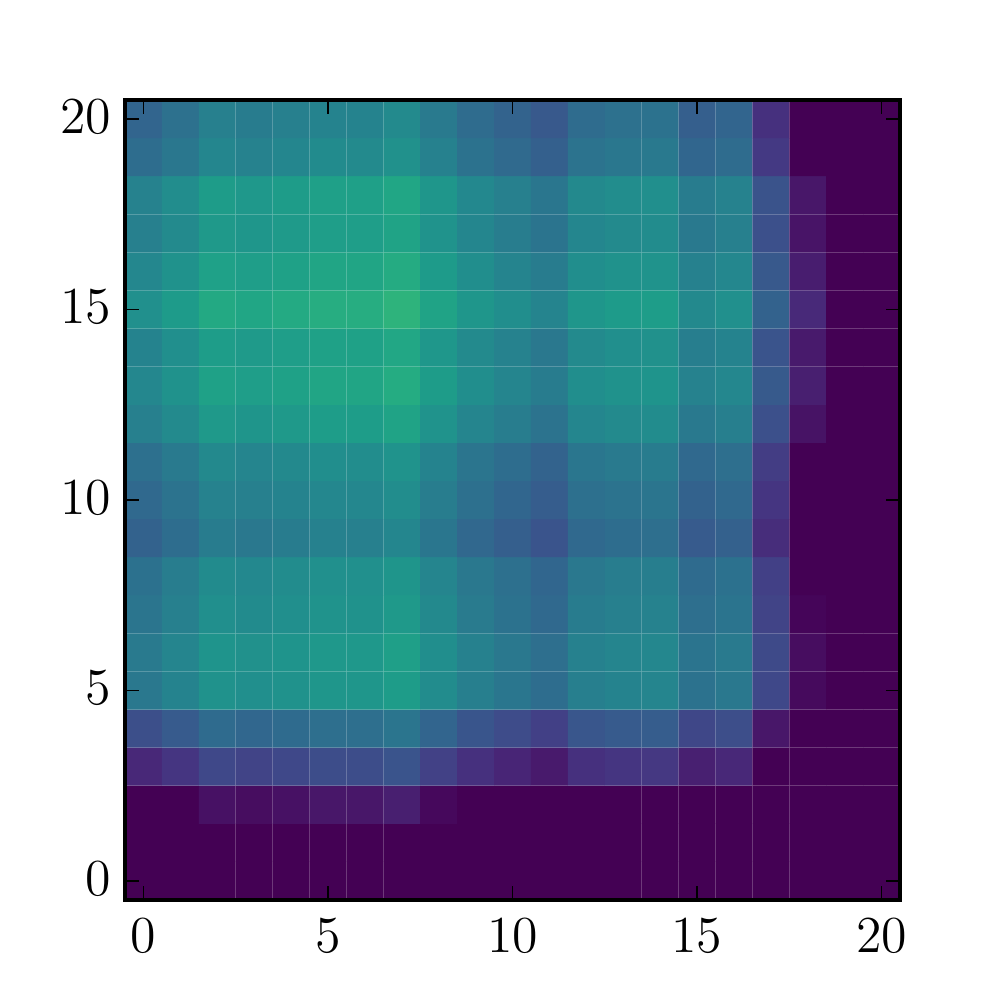}
  \caption{\scriptsize 5000 step}
\end{subfigure}\hspace*{\fill}
\begin{subfigure}[t]{.12\columnwidth}
  \includegraphics[width=\linewidth]{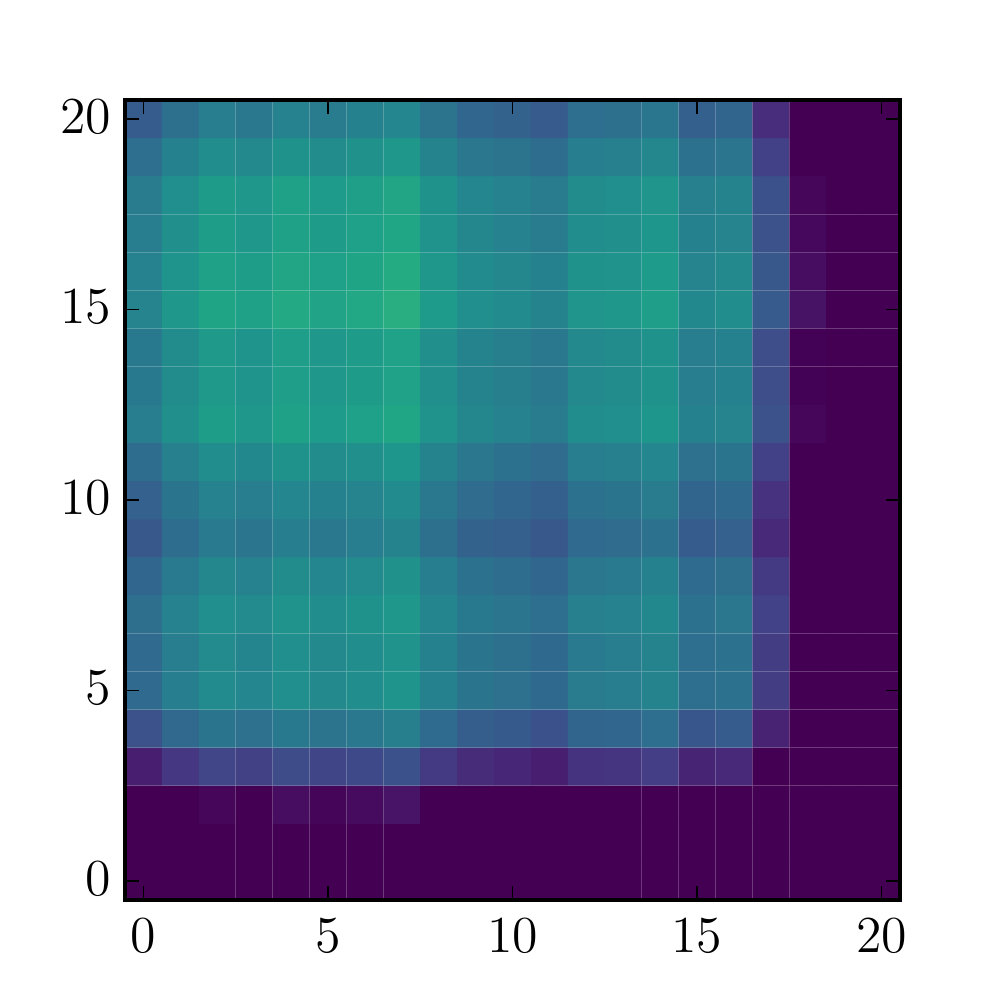}
  \caption{\scriptsize 6000 step}
\end{subfigure}\hspace*{\fill}
\begin{subfigure}[t]{.12\columnwidth}
  \includegraphics[width=\linewidth]{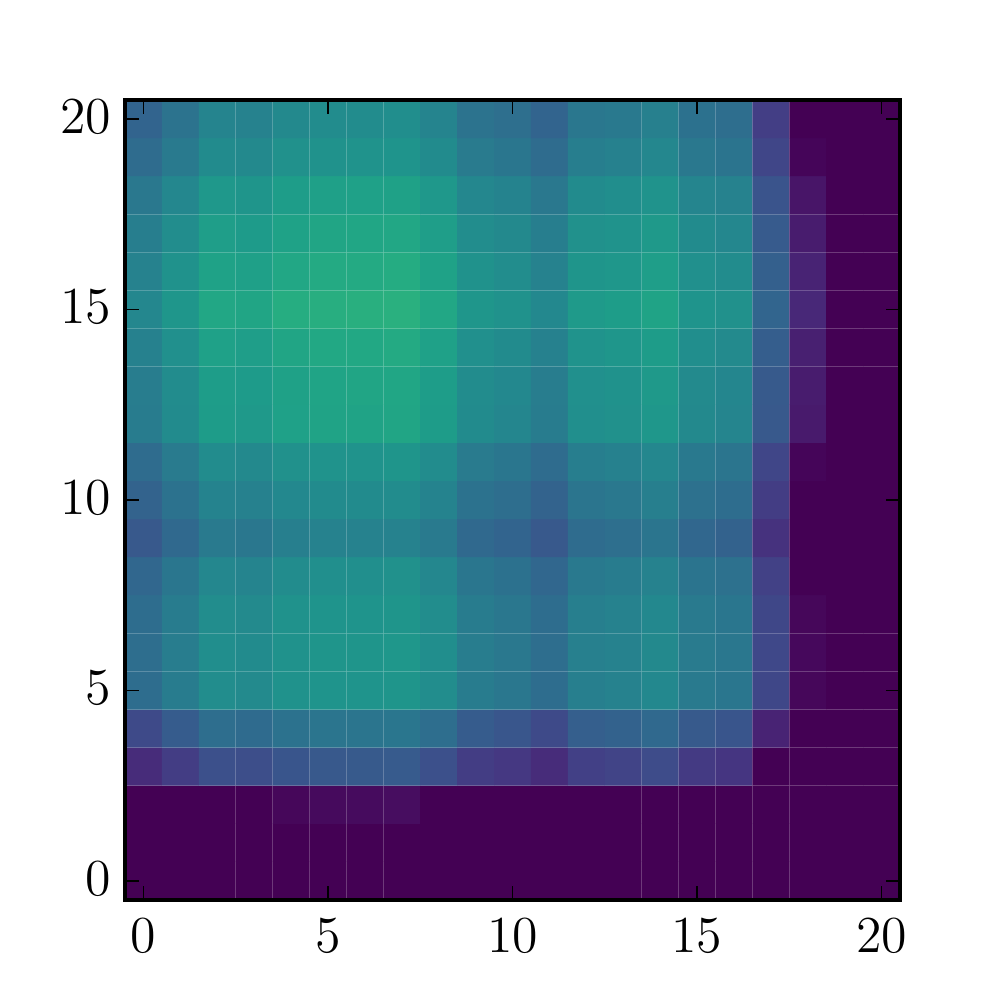}
  \caption{\scriptsize 7000 step}
\end{subfigure}\hspace*{\fill}
\begin{subfigure}[t]{.12\columnwidth}
  \includegraphics[width=\linewidth]{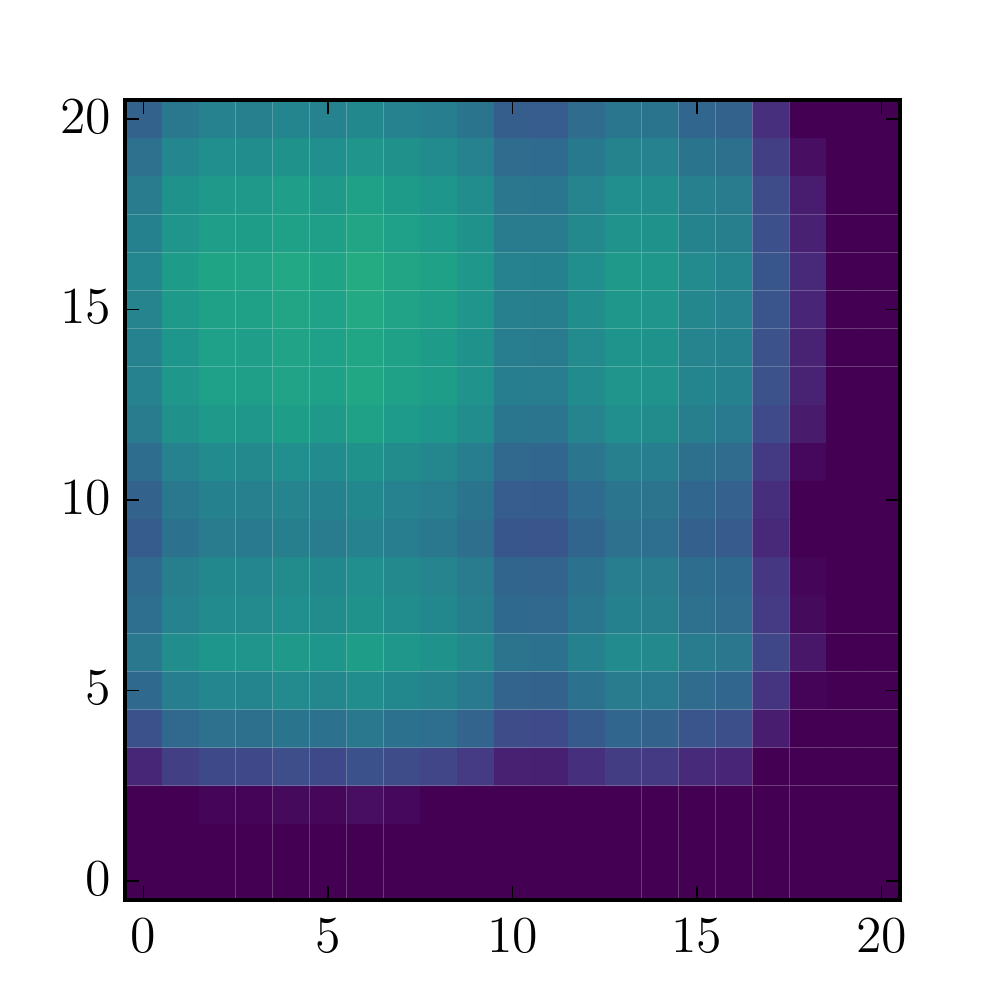}
  \caption{\scriptsize 8000 step}
\end{subfigure}\hspace*{\fill}
\vspace{-0.2cm}
\caption{$x$-axis and $y$-axis: agents 1 and 2's actions. Colored values represent the values of $\Qtot$ for QMIX}
\label{fig:PCresult5}
\end{figure}
\vspace{-0.6cm}
\begin{figure}[!h]
\hspace*{\fill}
\begin{subfigure}[t]{.12\columnwidth}
  \includegraphics[width=\linewidth]{figure/result-pcolor/QTRAN0.pdf}
  \caption{\scriptsize 1000 step}
\end{subfigure}\hspace*{\fill}
\begin{subfigure}[t]{.12\columnwidth}
  \includegraphics[width=\linewidth]{figure/result-pcolor/QTRAN1.pdf}
  \caption{\scriptsize 2000 step}
\end{subfigure}\hspace*{\fill}
\begin{subfigure}[t]{.12\columnwidth}
  \includegraphics[width=\linewidth]{figure/result-pcolor/QTRAN2.pdf}
  \caption{\scriptsize 3000 step}
\end{subfigure}\hspace*{\fill}
\begin{subfigure}[t]{.12\columnwidth}
  \includegraphics[width=\linewidth]{figure/result-pcolor/QTRAN3.pdf}
  \caption{\scriptsize 4000 step}
\end{subfigure}\hspace*{\fill}
\hspace*{\fill}
\begin{subfigure}[t]{.12\columnwidth}
  \includegraphics[width=\linewidth]{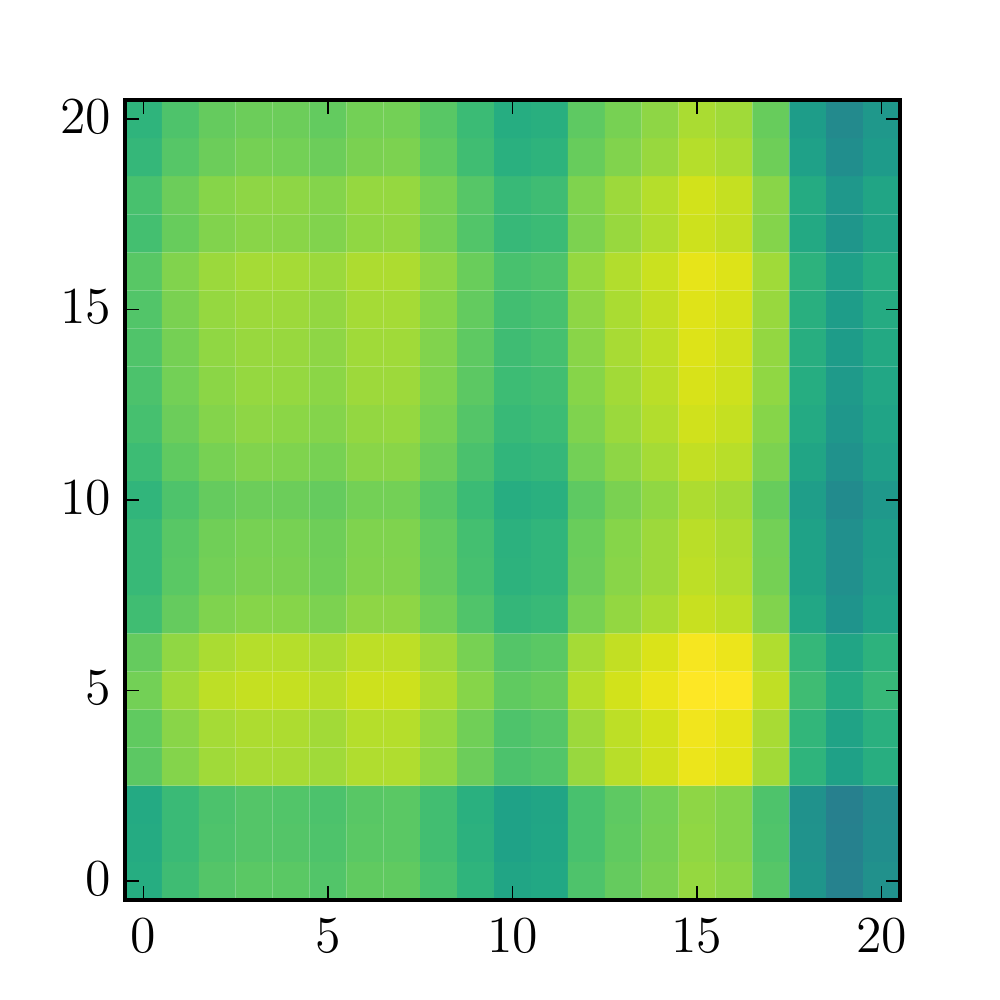}
  \caption{\scriptsize 5000 step}
\end{subfigure}\hspace*{\fill}
\begin{subfigure}[t]{.12\columnwidth}
  \includegraphics[width=\linewidth]{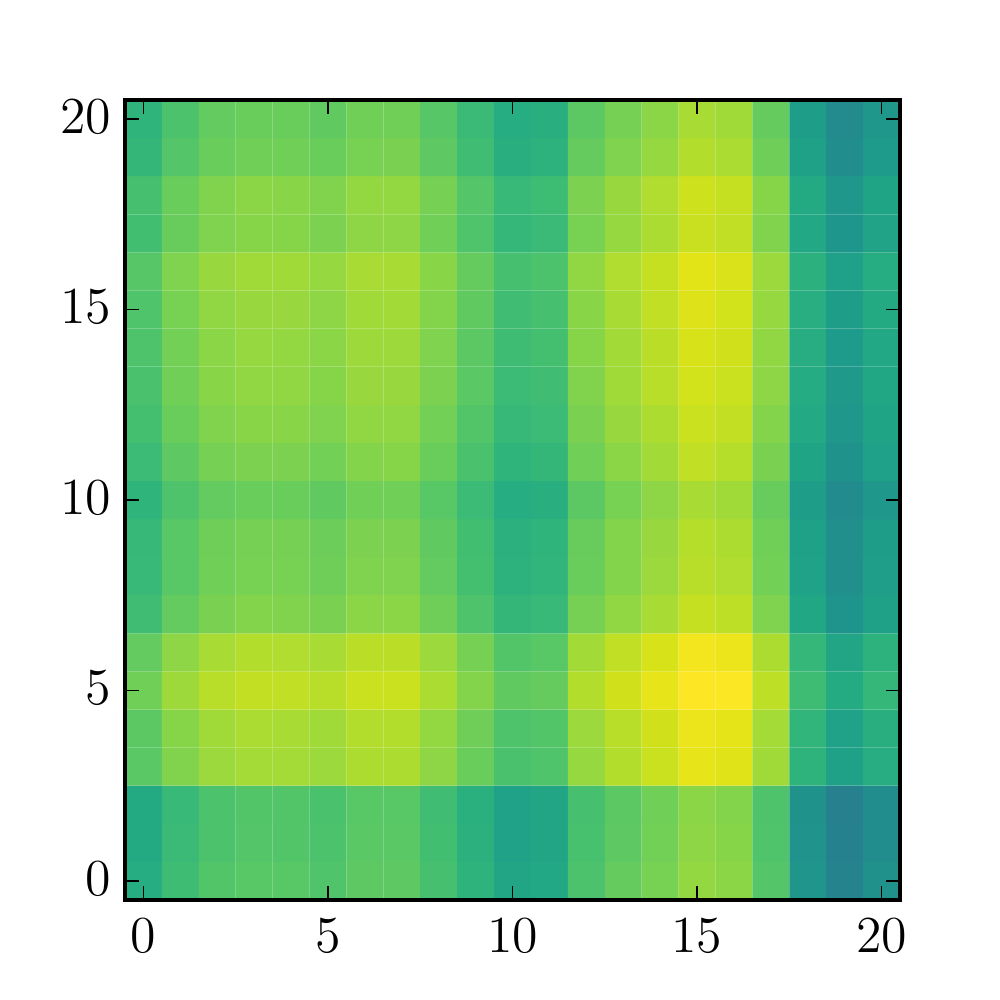}
  \caption{\scriptsize 6000 step}
\end{subfigure}\hspace*{\fill}
\begin{subfigure}[t]{.12\columnwidth}
  \includegraphics[width=\linewidth]{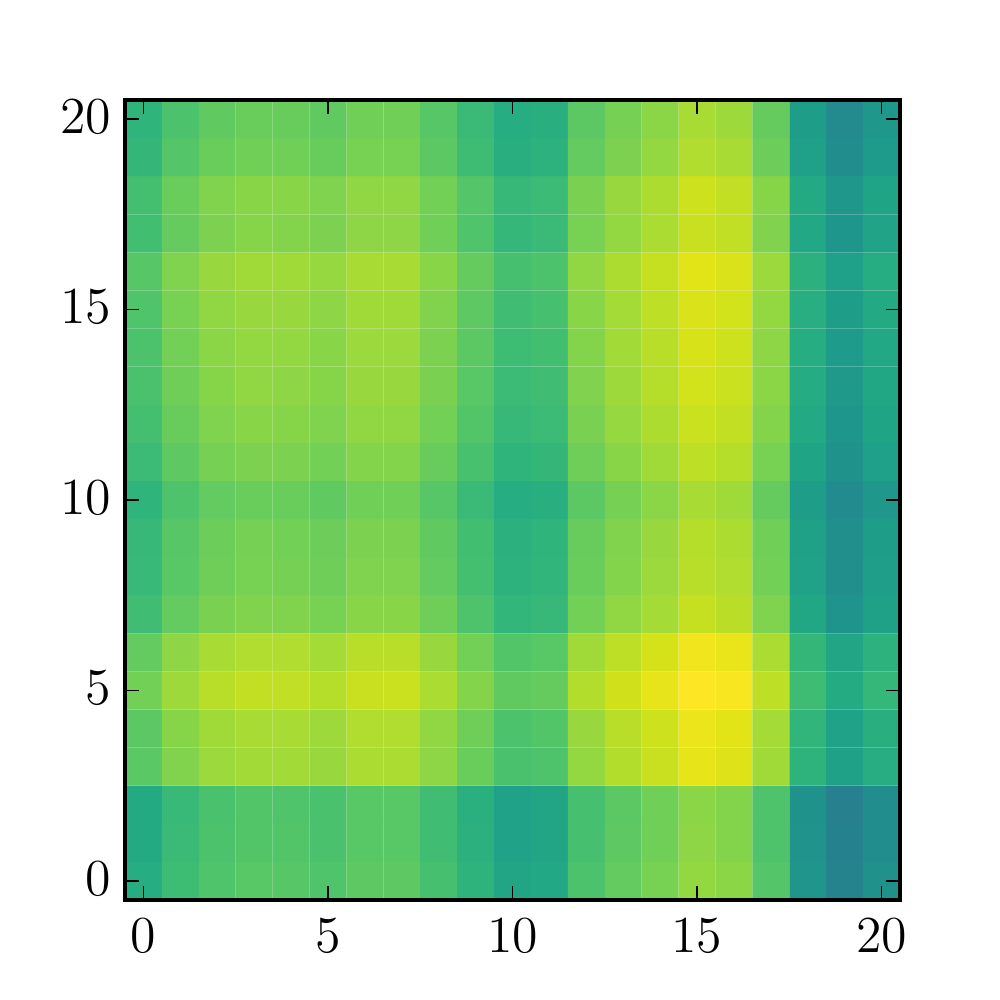}
  \caption{\scriptsize 7000 step}
\end{subfigure}\hspace*{\fill}
\begin{subfigure}[t]{.12\columnwidth}
  \includegraphics[width=\linewidth]{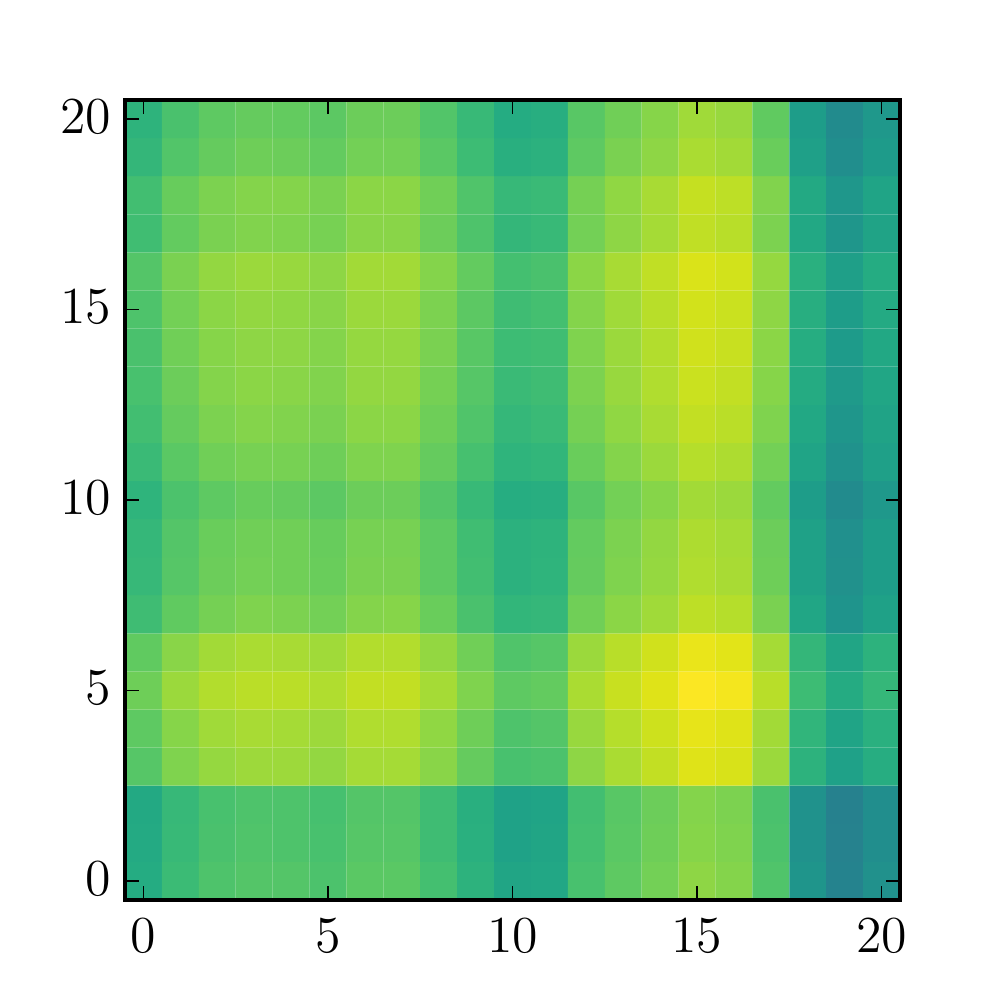}
  \caption{\scriptsize 8000 step}
\end{subfigure}\hspace*{\fill}
\vspace{-0.2cm}
\caption{$x$-axis and $y$-axis: agents 1 and 2's actions. Colored values represent the values of $\Qtot'$ for $\algonamebasic$}
\label{fig:PCresult2}
\end{figure}
\vspace{-0.6cm}
\begin{figure}[!h]
\hspace*{\fill}
\begin{subfigure}[t]{.12\columnwidth}
  \includegraphics[width=\linewidth]{figure/result-pcolor/QTRAN-adv0.pdf}
  \caption{\scriptsize 1000 step}
\end{subfigure}\hspace*{\fill}
\begin{subfigure}[t]{.12\columnwidth}
  \includegraphics[width=\linewidth]{figure/result-pcolor/QTRAN-adv1.pdf}
  \caption{\scriptsize 2000 step}
\end{subfigure}\hspace*{\fill}
\begin{subfigure}[t]{.12\columnwidth}
  \includegraphics[width=\linewidth]{figure/result-pcolor/QTRAN-adv2.pdf}
  \caption{\scriptsize 3000 step}
\end{subfigure}\hspace*{\fill}
\begin{subfigure}[t]{.12\columnwidth}
  \includegraphics[width=\linewidth]{figure/result-pcolor/QTRAN-adv3.pdf}
  \caption{\scriptsize 4000 step}
\end{subfigure}\hspace*{\fill}
\hspace*{\fill}
\begin{subfigure}[t]{.12\columnwidth}
  \includegraphics[width=\linewidth]{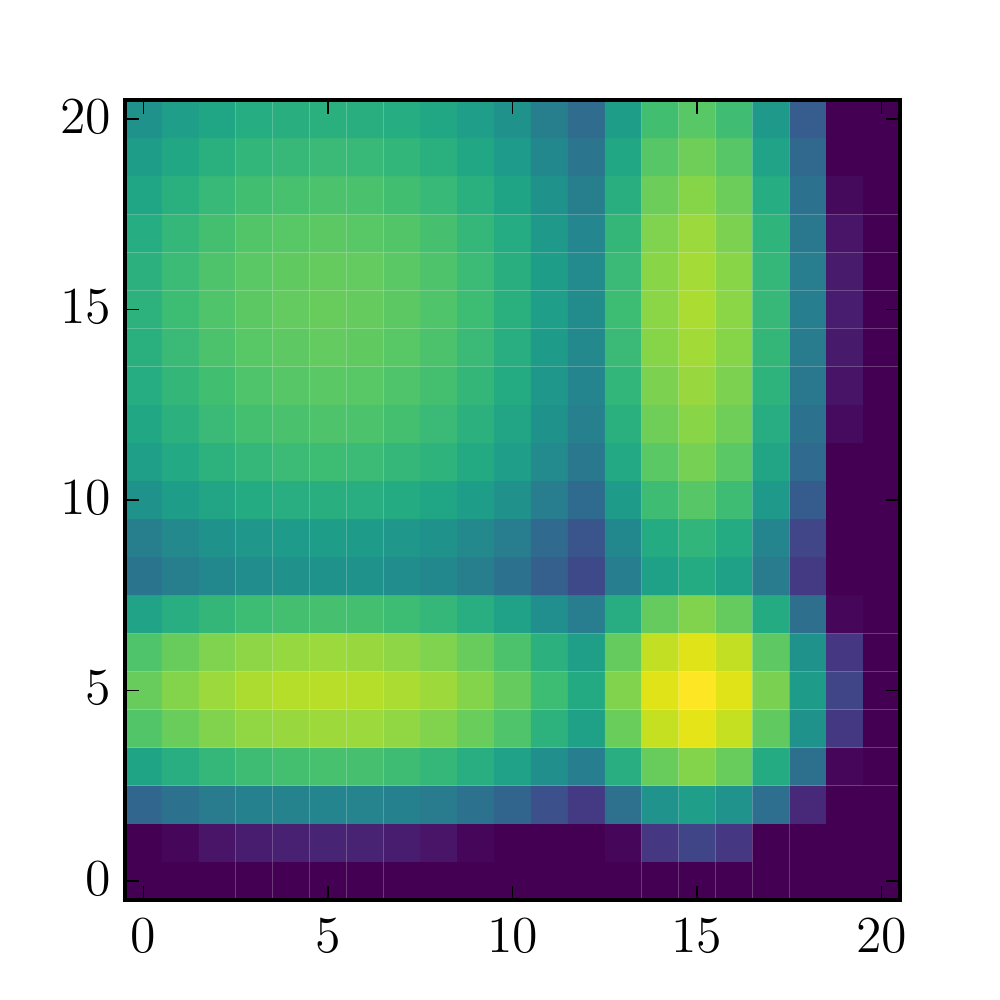}
  \caption{\scriptsize 5000 step}
\end{subfigure}\hspace*{\fill}
\begin{subfigure}[t]{.12\columnwidth}
  \includegraphics[width=\linewidth]{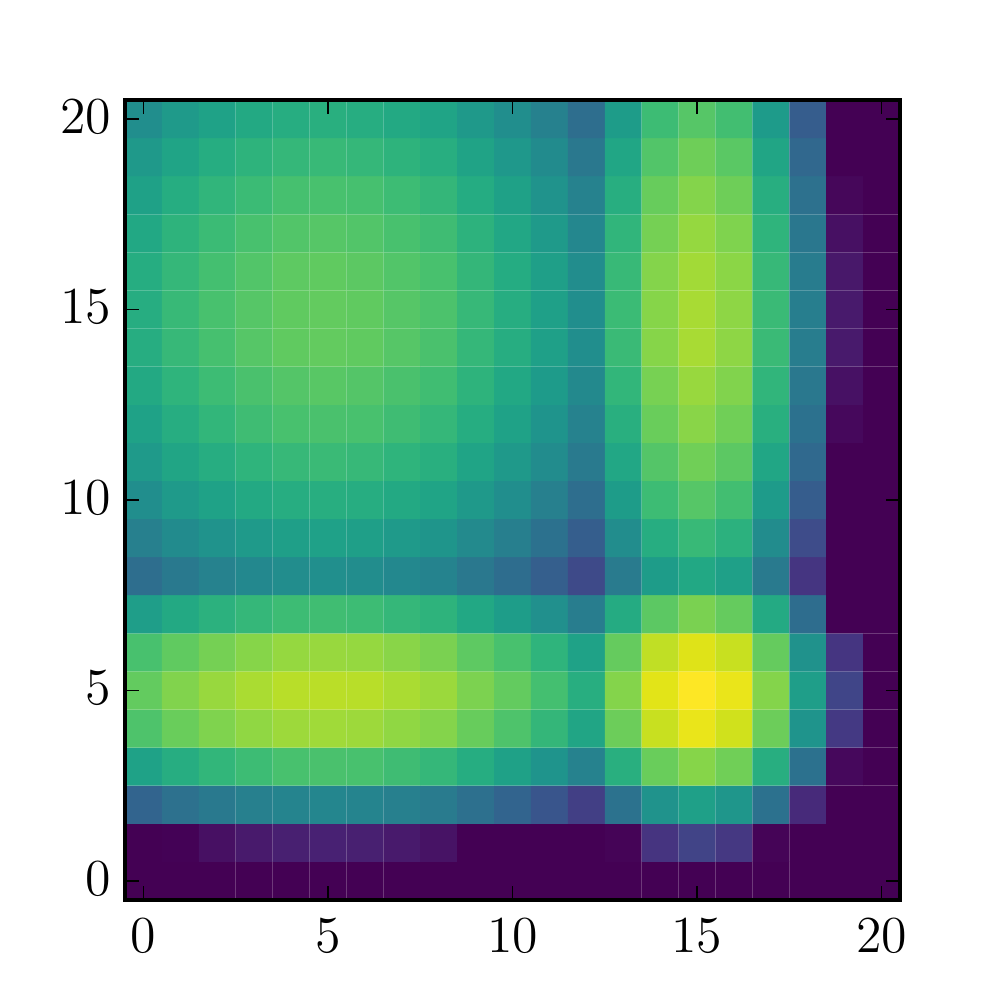}
  \caption{\scriptsize 6000 step}
\end{subfigure}\hspace*{\fill}
\begin{subfigure}[t]{.12\columnwidth}
  \includegraphics[width=\linewidth]{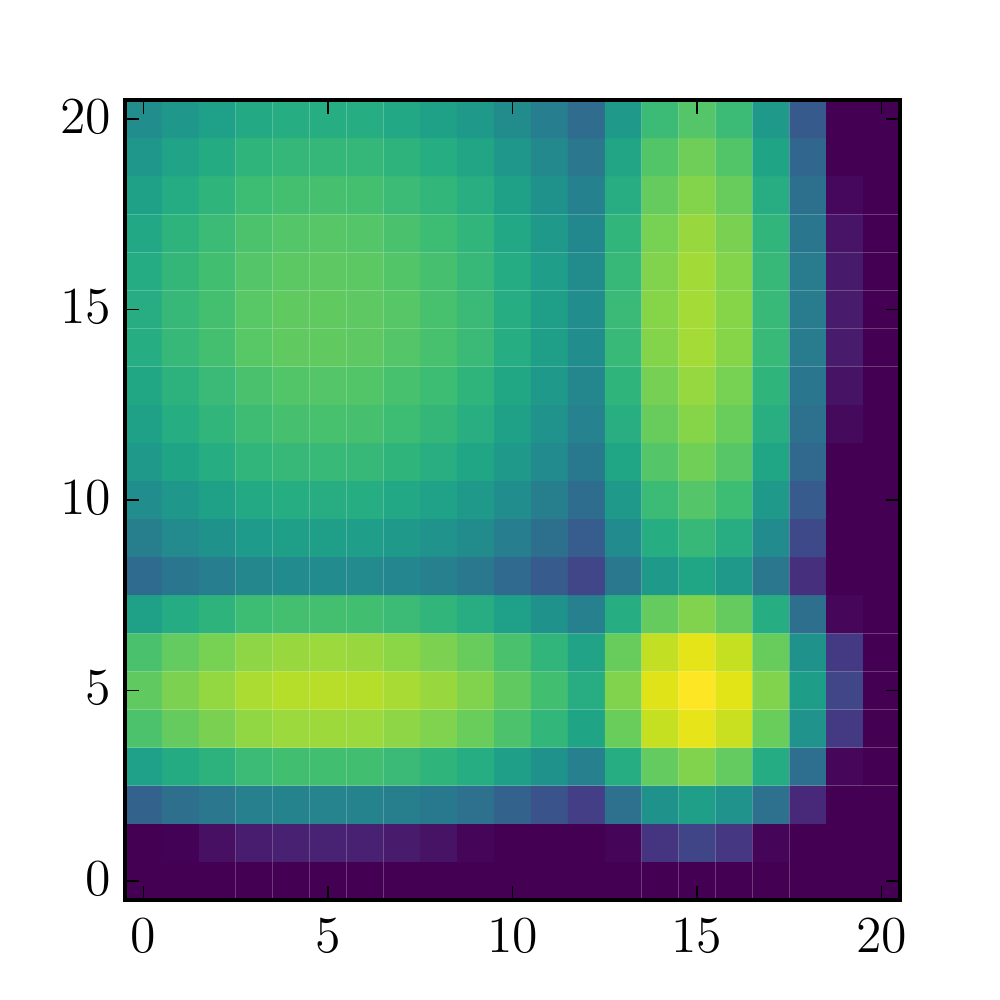}
  \caption{\scriptsize 7000 step}
\end{subfigure}\hspace*{\fill}
\begin{subfigure}[t]{.12\columnwidth}
  \includegraphics[width=\linewidth]{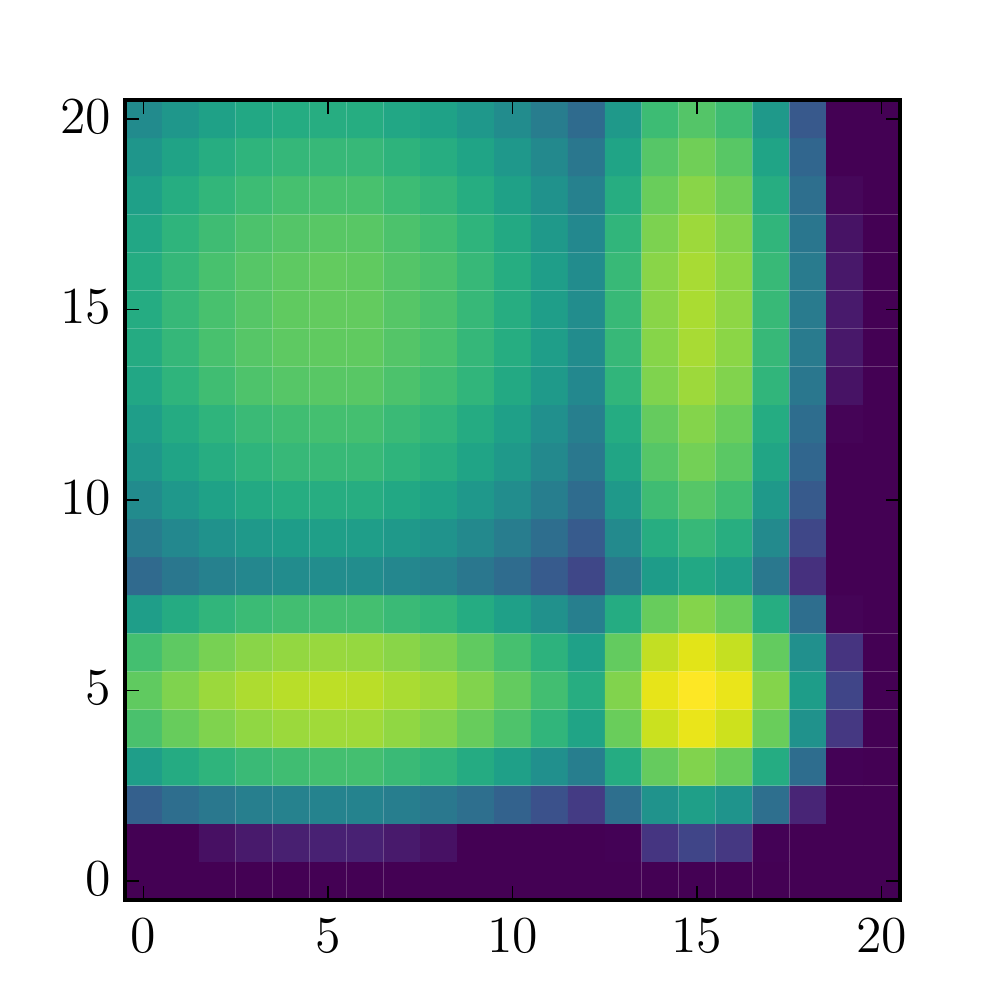}
  \caption{\scriptsize 8000 step}
\end{subfigure}\hspace*{\fill}
\vspace{-0.2cm}
\caption{$x$-axis and $y$-axis: agents 1 and 2's actions.  Colored values represent the values of $\Qtot'$ for $\algonameadv$}
\label{fig:PCresult3}
\end{figure}

\section{Comparison with other value-based methods for modified predator-prey}

Additionally, we have conducted experiments with Dec-HDRQN \cite{Omid:Hysteretic}. Dec-HDRQN can indeed solve problems similar to ours by changing the learning rate according to TD-error without factorization. However, Dec-HDRQN does not take advantage of centralized training. We implemented Dec-HDRQN ($\alpha = 0.001, \beta = 0.0002$) with modified predator-prey experiments and Figure~\ref{fig:PPresult1cr} shows the performance of algorithms for six settings with different $N$ and $P$ values.

First, when the complexity of the task is relatively low, Dec-HDRQN shows better performance than VDN and QMIX as shown in the Figure~\ref{fig:PPgraph2cr}. However, QTRAN performs better than Dec-HDRQN in the case. When the penalty and the number of agents are larger, Dec-HDRQN underperforms VDN and QMIX. Figure~\ref{fig:PPgraph3cr} shows Dec-HDRQN scores an average of nearly 0 in the case of $N=2, P=1.5$. There is a limit of Dec-HDRQN since the method is heuristic and does not perform centralized training. Finally, Dec-HDRQN showed slower convergence speed overall.

\begin{figure*}[t!]
\begin{subfigure}[t]{.33\textwidth}
  \includegraphics[width=\linewidth]{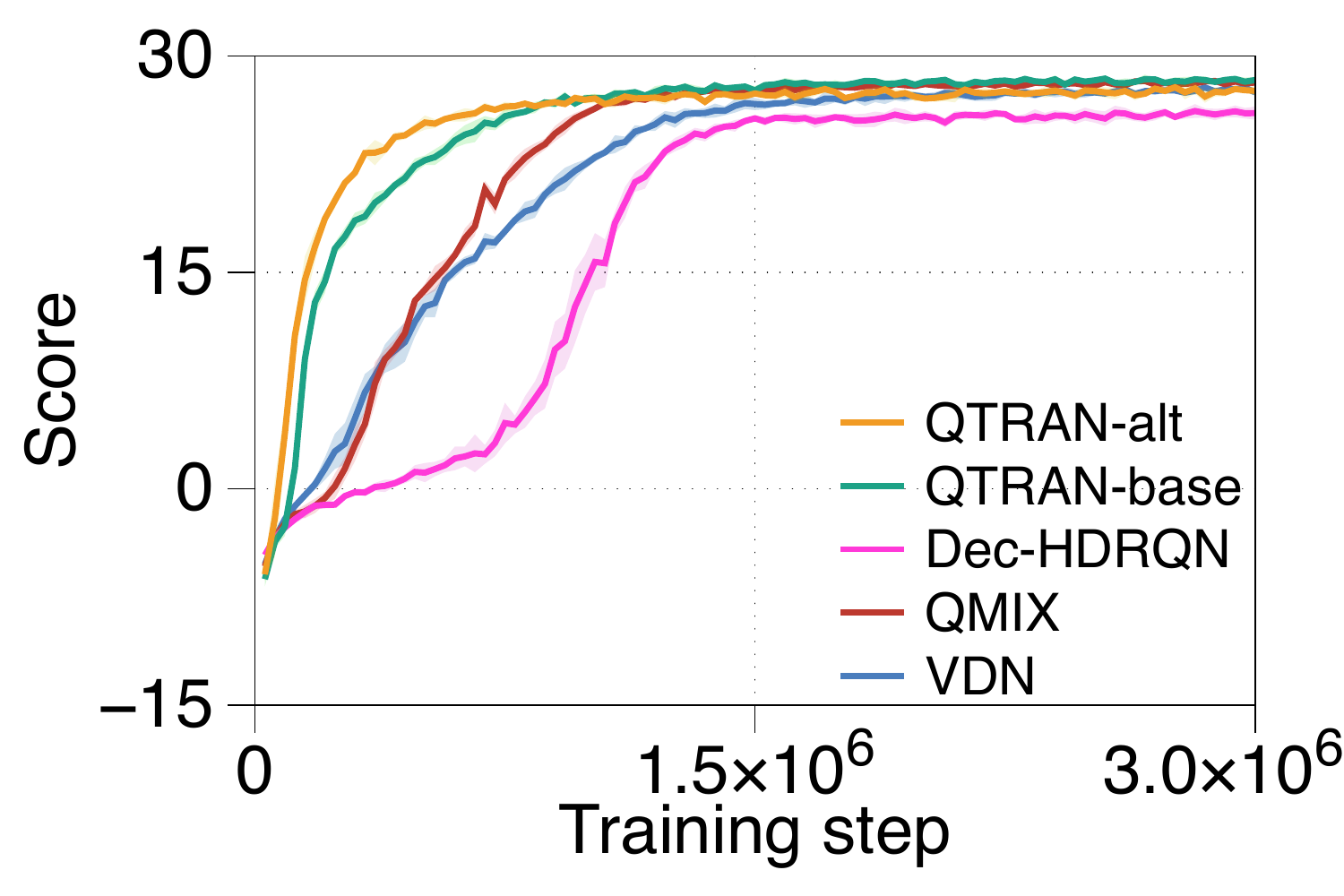}
  \caption{$N=2, P = 0.5$}
  \label{fig:PPgraph1cr}
\end{subfigure}\hspace*{\fill}
\begin{subfigure}[t]{.33\textwidth}
  \includegraphics[width=\linewidth]{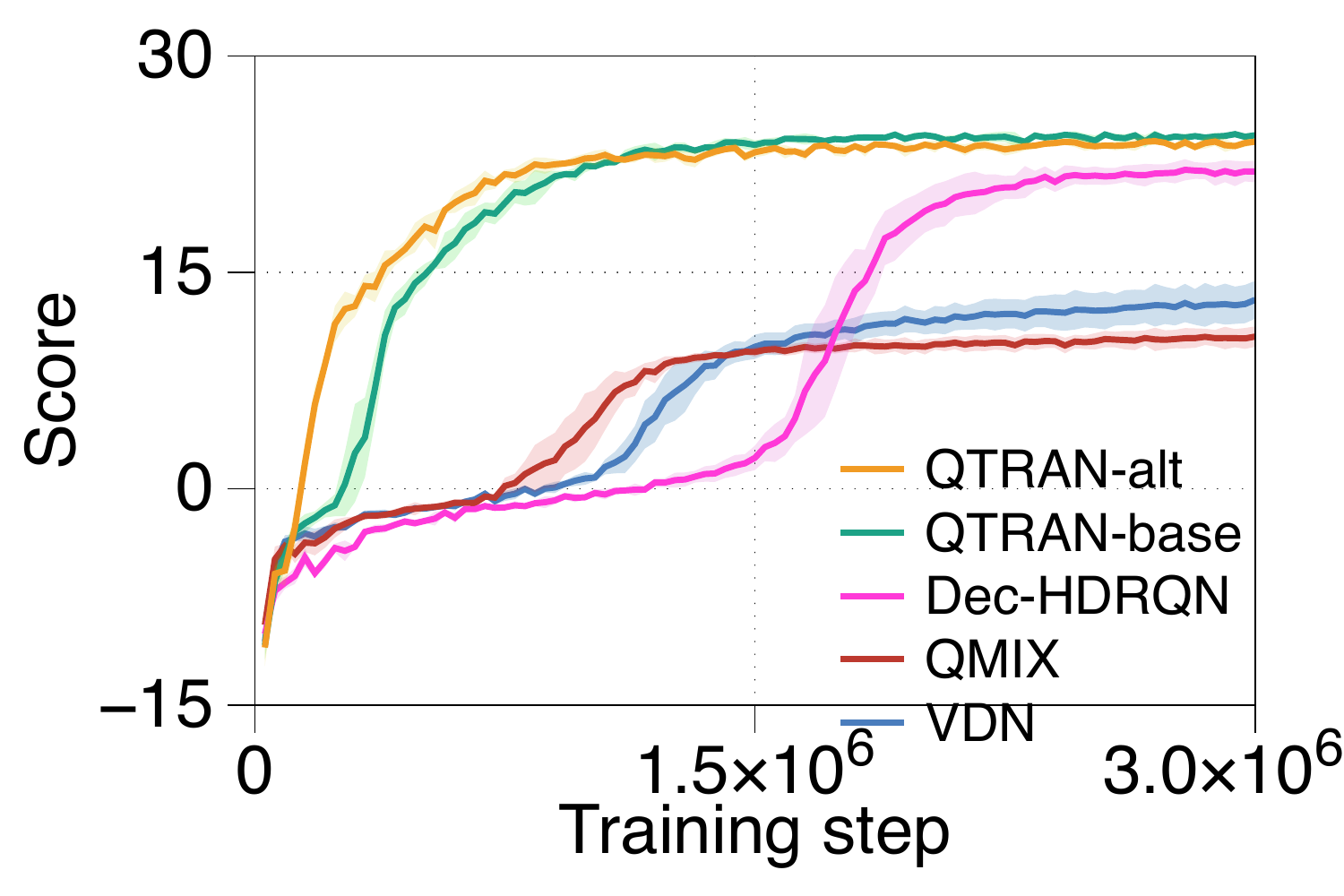}
  \caption{$N=2, P = 1.0$}
  \label{fig:PPgraph2cr}
\end{subfigure}\hspace*{\fill}
\begin{subfigure}[t]{.33\textwidth}
  \includegraphics[width=\linewidth]{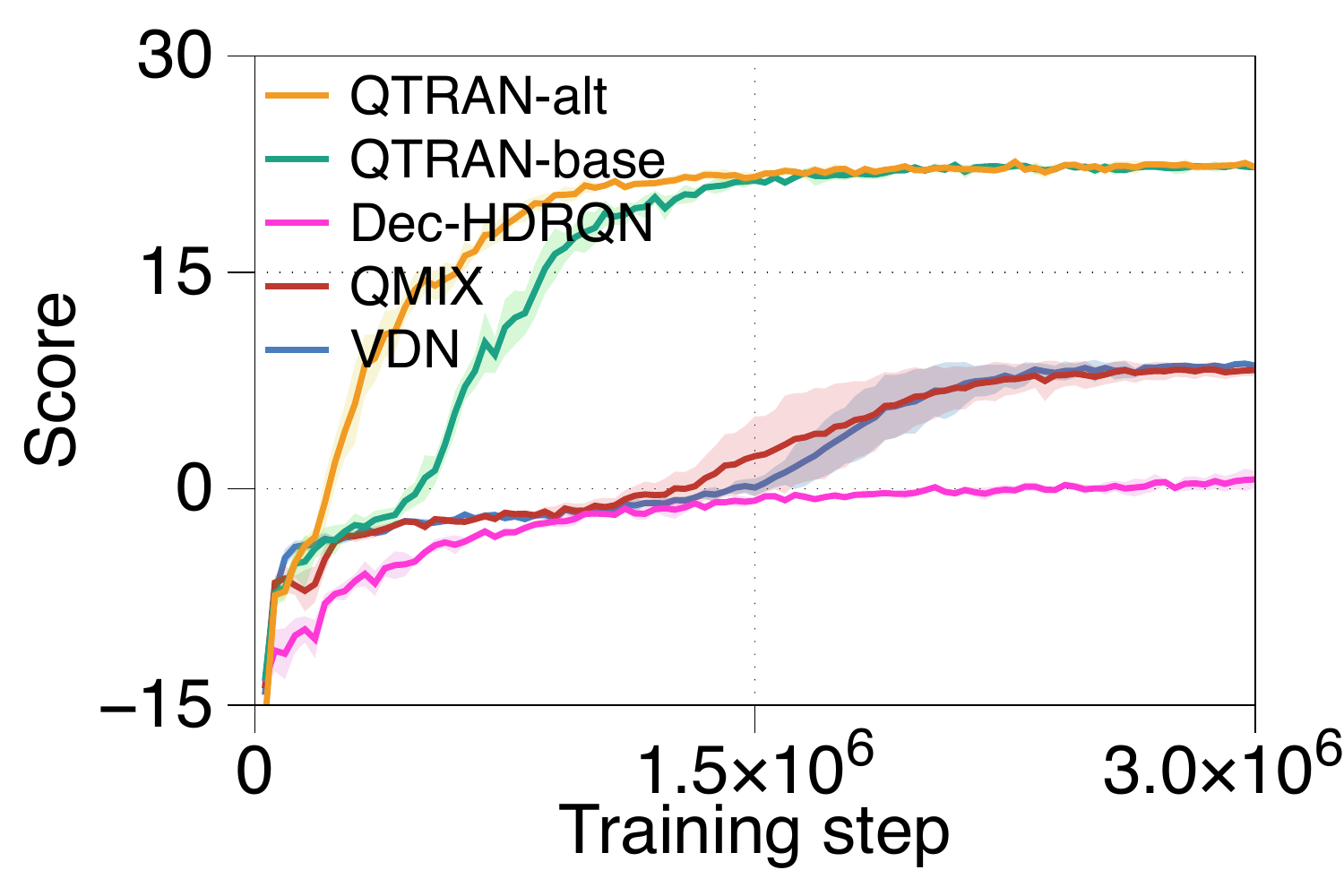}
  \caption{$N=2, P = 1.5$}
  \label{fig:PPgraph3cr}
\end{subfigure}\hspace*{\fill}\\
\begin{subfigure}[t]{.33\textwidth}
  \includegraphics[width=\linewidth]{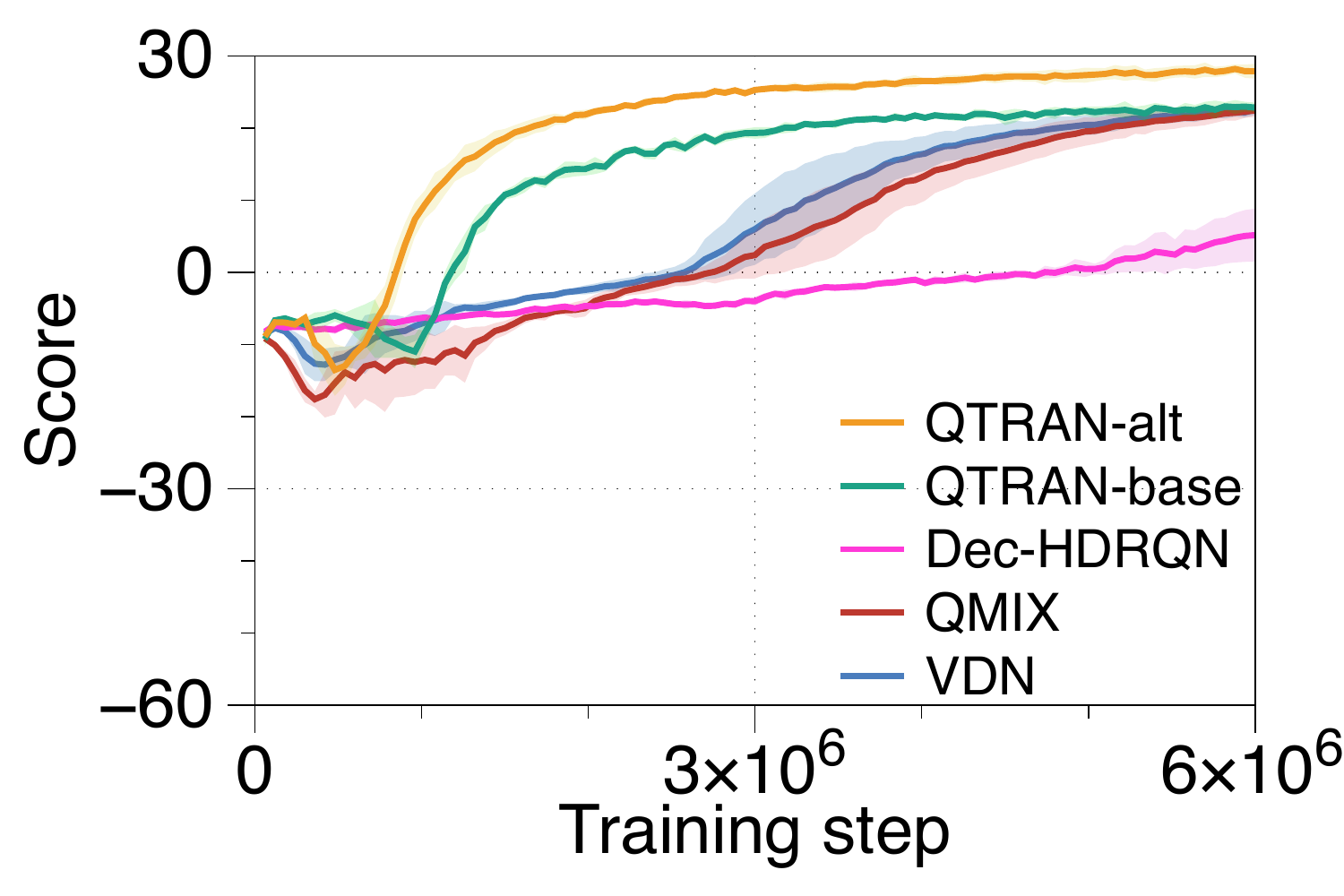}
  \caption{$N=4, P = 0.5$}
  \label{fig:PPgraph4cr}
\end{subfigure}\hspace*{\fill}
\begin{subfigure}[t]{.33\textwidth}
  \includegraphics[width=\linewidth]{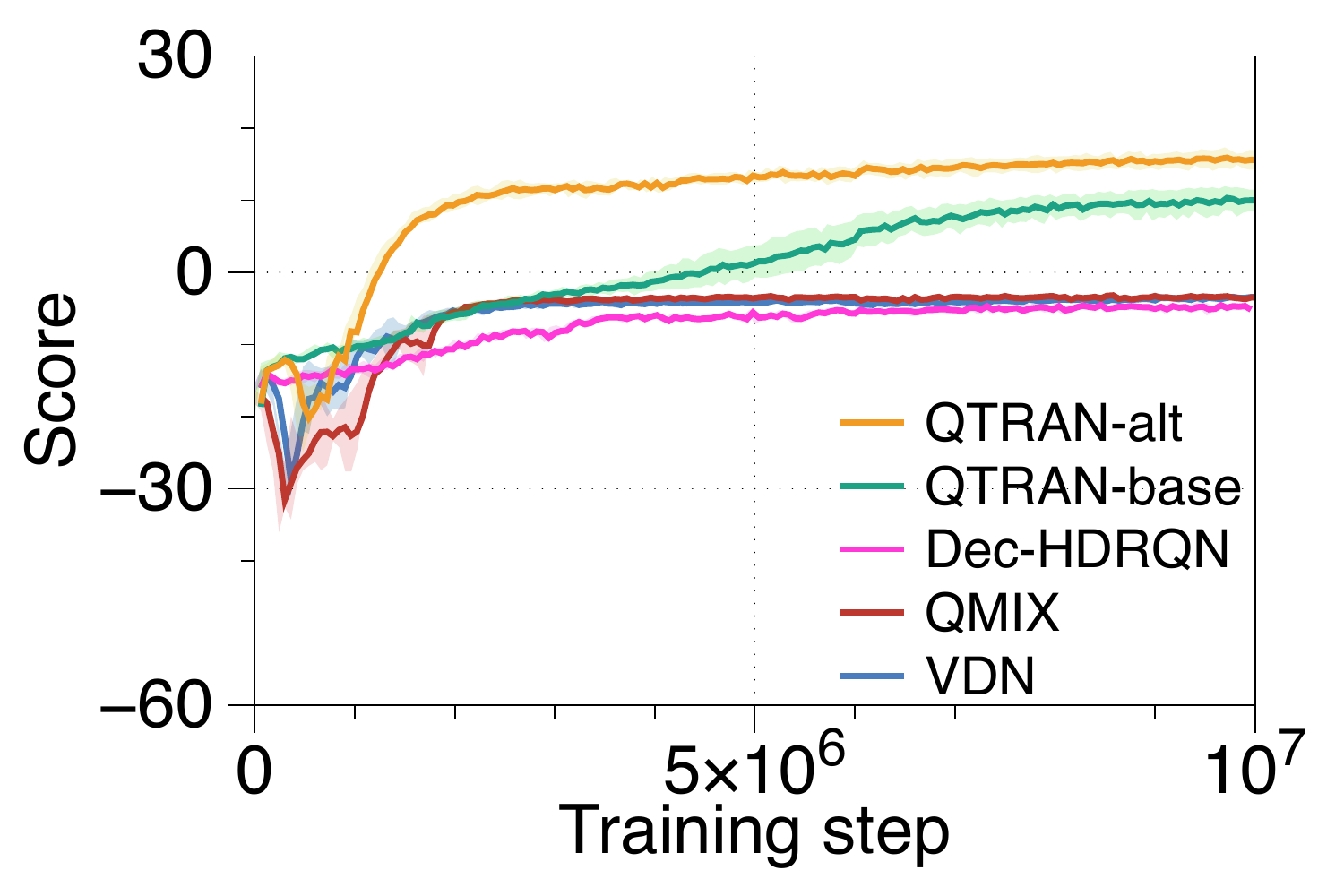}
  \caption{$N=4, P = 1.0$}
  \label{fig:PPgraph5cr}
\end{subfigure}\hspace*{\fill}
\begin{subfigure}[t]{.33\textwidth}
  \includegraphics[width=\linewidth]{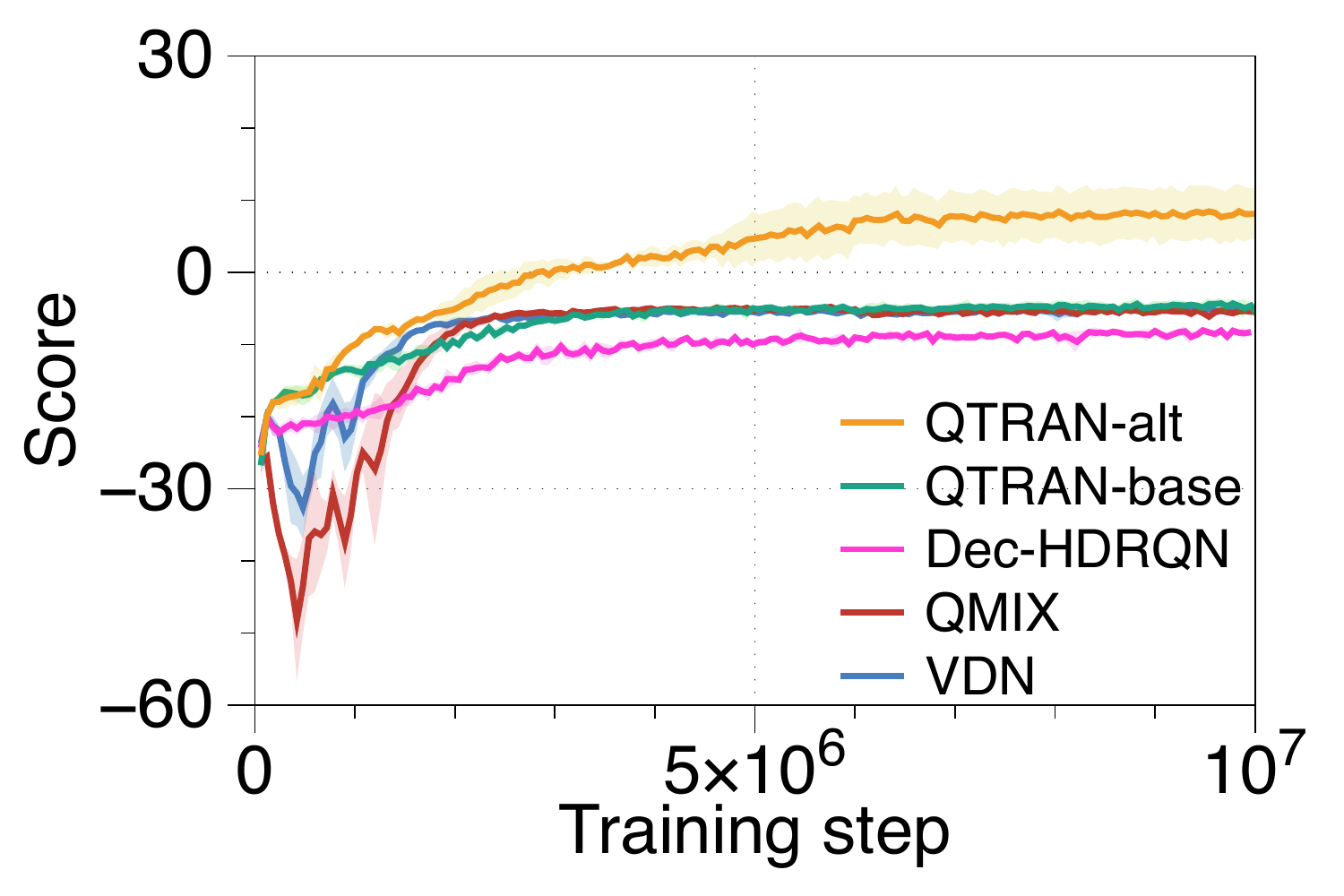}
  \caption{$N=4, P = 1.5$}
  \label{fig:PPgraph6cr}
\end{subfigure}\hspace*{\fill}
\caption{Average reward per episode on the MPP tasks with 95\% confidence intervals for VDN, QMIX, Dec-HDRQN and $\algoname$}
\label{fig:PPresult1cr}
\vspace{-0.3cm}
\end{figure*}





\bibliographystyle{icml2019}

\end{document}